%% file: main-arxiv.tex
\documentclass{article}

\usepackage{ez}

\usepackage[utf8]{inputenc} 
\usepackage[T1]{fontenc}    
\usepackage{url}            
\usepackage{amsfonts}       
\usepackage{nicefrac}       
\usepackage{microtype}      

\usepackage{lmodern}

\usepackage{floatrow}

\usepackage{amssymb}
\usepackage{amsmath}

\usepackage{thmtools, thm-restate}
\usepackage{mathrsfs} 
\usepackage[usenames,dvipsnames]{pstricks}
\usepackage{graphicx}
\usepackage{enumerate, subfigure, color}

\usepackage{newfloat}

\usepackage{booktabs}

\usepackage{algorithm}
\usepackage{algorithmicx}

\usepackage{bbm, dsfont}

\definecolor{dgreen}{rgb}{0.00,0.49,0.00}
\definecolor{dblue}{rgb}{0,0.08,0.75}
\RequirePackage[colorlinks,citecolor=dgreen,urlcolor=dblue,linkcolor=dblue]{hyperref}

\usepackage[nameinlink,capitalize]{cleveref}

\usepackage{etoolbox}

\input{commands}

\let\mathsf\relax    
\DeclareRobustCommand{\mathsf}[1]{\text{\normalfont\sffamily#1}}

\newcommand{\X}{{\mathcal X}}
\newcommand{\Y}{{\mathcal Y}}
\newcommand{\Z}{{\mathcal Z}}

\newcommand{\fstar}{{f^*}}
\newcommand{\fn}{{f_n}}

\newcommand{\ghat}{{\widehat g}}
\newcommand{\gstar}{{g^*}}

\newcommand{\gn}{{g_n}}

\newcommand{\E}{\mathcal{E}}

\newcommand{\EE}{\mathbb{E}}

\renewcommand{\H}{\mathcal{H}}

\newcommand{\hh}{{\H}}
\newcommand{\HX}{{\mathcal{F}}}
\newcommand{\HY}{{\H}}

\newcommand{\ones}{{\mathbbm{1}}}

\newcommand{\R}{\mathbb{R}}

\newcommand{\N}{\mathbb{N}}
\newcommand{\CC}{\mathbb{C}}

\newcommand{\F}{\mathcal{F}}

\newcommand{\rr}{{\mathcal{R}}}

\newcommand{\ff}{{\mathcal{F}}}
\renewcommand{\gg}{{\mathcal{G}}}

\newcommand{\closs}{{\msf{c}_\loss}}

\newcommand{\la}{\lambda}

\newcommand{\ot}{\otimes}

\newcommand{\loss}{\bigtriangleup}

\newcommand{\decoding}{\msf d}
\newcommand{\coding}{\msf c}

\newcommand{\rhox}{{\rho_\X}}

\newcommand{\Shat}{ {\widehat {S}} }

\newcommand{\Chat}{ {\widehat {C}} }

\newcommand{\Chatla}{ {\widehat {C}_\la} }
\newcommand{\Cla}{ {{C_\la}} }

\newcommand{\Cnl}{ \Chatla}
\newcommand{\Cl}{\Cla}

\newcommand{\Ltwo}{{L^2}}
\newcommand{\LX}{L^2(\X,\rhox)}
\newcommand{\LXR}{L^2(\X,\rhox,\R)}
\newcommand{\LXH}{L^2(\X,\rhox,\hh)}

\newcommand{\domrho}{D_{\rho|\X}}

\newcommand{\Q}{Q}

\newcommand{\eqals}[1]{\begin{align*}#1\end{align*}}

\newcommand{\eqal}[1]{\begin{align}#1\end{align}}

\renewcommand{\eqals}[1]{\eqal{#1}}

\newcommand{\msf}[1]{\mathsf{#1}}
\newcommand{\mbf}[1]{\mathbf{#1}}

\newcommand{\tr}{\ensuremath{\text{\rm Tr}}}
\newcommand{\Span}{\ensuremath{\text{\rm span}}}
\newcommand{\ran}{\ensuremath{\text{\rm Ran}}}

\newcommand{\argmin}{\operatornamewithlimits{argmin}}
\newcommand{\argmax}{\operatornamewithlimits{argmax}}

\newcommand{\HS}{{\rm HS}}

\newcommand{\sign}{\ensuremath{\text{\rm sign}}}
\newcommand{\deff}{\ensuremath{d_{\text{\rm{eff}}}}}

\newcommand{\filter}{\eta}

\newcommand{\xmap}{\phi}
\newcommand{\ymap}{\varphi}
\newcommand{\zmap}{\psi}

\renewcommand{\paragraph}[1]{\vspace{1em}\noindent{\bfseries #1}.}

\declaretheorem[name=Theorem,refname=Thm.]{theorem}
\declaretheorem[name=Lemma,sibling=theorem]{lemma}
\declaretheorem[name=Proposition,refname=Prop.,sibling=theorem]{proposition}
\declaretheorem[name=Remark]{remark}
\declaretheorem[name=Corollary,refname=Cor.,sibling=theorem]{corollary}
\declaretheorem[name=Definition,refname=Def.]{definition}

\declaretheorem[name=Assumption,refname=Asm.]{assumption}

\declaretheorem[name=Example]{example}

\AfterEndEnvironment{restatable}{\noindent\ignorespaces}

\crefname{assumption}{Assumption}{Assumptions}
\crefname{equation}{}{}
\crefname{figure}{Fig.}{Fig.}
\crefname{table}{Table}{Tables}
\crefname{section}{Sec.}{Sec.}
\crefname{theorem}{Thm.}{Thm.}
\crefname{lemma}{Lemma}{Lemmas}
\crefname{corollary}{Cor.}{Cor.}
\crefname{example}{Example}{Examples}
\crefname{remark}{Remark}{Remarks}
\crefname{algorithm}{Alg.}{Algorightms}

\crefname{appendix}{Appendix}{Appendices}
\crefname{subappendix}{Appendix}{Appendices}
\crefname{subsubappendix}{Appendix}{Appendices}

\newcommand{\titlestruct}{A General Framework for Consistent Structured Prediction with Implicit Loss Embeddings}

\newcommand{\structshort}{ILE}
\newcommand{\struct}{Implicit Loss Embedding}

\begin{document}

\title{\titlestruct{}}
\date{}

\author{%
    Carlo Ciliberto$^1$,~~ Lorenzo Rosasco$^{2,3}$,~~ Alessandro Rudi$^{4}$\\
    \hspace{-1.35em}{\footnotesize \texttt{c.ciliberto@imperia.ac.uk, lorenzo.rosasco@mit.edu, alessandro.rudi@inria.fr}}\\ 
    {\footnotesize $^1$  Department of Electrical and Electronic Engineering, Imperial College London, UK}\\
    {\footnotesize $^2$ University of Genova and Istituto Italiano di Tecnologia, Genova, Italy}\\
    {\footnotesize $^3$ Massachusetts Institute of Technology, Cambridge, MA, USA}\\
    {\footnotesize $^4$ INRIA, Paris, France, \'Ecole Normale Sup\'erieure, Paris, France, PSL Research, France}\\
}

\maketitle

\begin{abstract}
\noindent We propose and analyze a novel  theoretical and algorithmic framework for structured prediction. While so far the term has referred to discrete output spaces, here we consider more general settings, such as manifolds or spaces of probability measures. 
We define structured prediction as a problem where the output space  lacks a vectorial structure. We identify  and study  a large class of loss functions  that implicitly defines a suitable geometry on the problem. The latter is the key to  develop an  algorithmic framework amenable to a sharp statistical analysis and yielding efficient computations. When dealing 
with output spaces with infinite cardinality, a suitable implicit formulation of the estimator is  shown to be crucial.  
    
\end{abstract}


\section{Introduction}

Statistical learning theory offers a number of methods to deal with supervised problems when the output space is linear (e.g. scalar values or vectors). However, applications involving more general output spaces are becoming increasingly common. Examples include image segmentation \citep{alahari2008reduce} or captioning \citep{karpathy2015deep}, speech recognition \citep{bahl1986maximum,sutton2012introduction}, manifold regression \citep{steinke2010nonparametric}, trajectory planning \citep{ratliff2006maximum}, protein folding \citep{joachims2009predicting}, prediction of probability distributions \citep{frogner2015learning}, ordinal regression \citep{pedregosa2017consistency}, information retrieval or ranking \citep{duchi2010} to name a few \citep[see][for more examples]{bakir2007,nowozin2011structured}.  When considering discrete output spaces, these settings are often referred to as {\em structured prediction} problems, since they require dealing with output spaces that have a specific structure, such as strings, graphs or sequences. 

Standard machine learning  methods  like  empirical risk minimization are faced with both modeling and computational challenges in these settings.  Therefore, in practice, either one of the following two   main strategies are often considered. On the one hand, {\em surrogate methods} \citep{bartlett2006,mroueh2012} that design  {\em ad-hoc} algorithms and theory for different learning settings on a case-by-case basis. While this allows to prove strong theoretical guarantees, it makes it difficult to extend previous results to new settings. On the other hand, {\em likelihood estimation} methods \citep{lafferty2001conditional,taskar2004max,tsochantaridis2005,bakir2007,nowozin2011structured} have broad applicability but typically poor theoretical guarantees \citep{tewari2007consistency}.  

In this work, we propose a novel structured prediction framework combining the best of both worlds. Our approach  extends structured prediction beyond discrete outputs, to include problems such as manifold regression.  The lack of linear structure in the output space is the common feature of the different problems we consider. The key observation is that for a very wide range of problems, the associated loss function carries implicitly a natural corresponding geometry. More precisely it admits an {\em \struct{} (\structshort{})} into a linear (albeit possibly infinite dimensional) space. Exploiting such a  geometry allows us to  derive a consistent least squares algorithmic framework. The latter can be seen as related to {\itshape kernel dependency estimation (KDE)}  \citep{weston2002,cortes2005,kadri2013generalized}, where the loss was assumed to be induced by an inner product. Here we consider loss functions defined by non-symmetric forms, hence covering a much wider range of examples, including losses defined by distances and divergences. Like for KDE, an implicit formulation of the proposed estimator can be derived. This becomes crucial when considering structured prediction problems beyond the discrete case. 
This paper is the extended version of \cite{ciliberto2016}, which has initiated a recent stream of works where the method has been equivalently referred to as either the {\em Structured Encoding Loss Function (SELF) approach} or the {\em Quadratic Surrogate framework} \citep{osokin2017structured,ciliberto2017consistent,korba2018structured,rudi2018manifold,struminsky2018quantifying,luise2018differential,nowak2018sharp,djerrab2018output,luise2019leveraging,blondel2019structured,ciliberto2019localized}. In this paper we refine the analysis in \cite{ciliberto2016}, providing novel insights and estimators for structured prediction. More precisely, the novel contributions of the current work are: (a) we propose a number of novel estimators for structure prediction by leveraging the connection between our \structshort{} framework and vector-valued regression settings (\cref{sec:framework}). (b) We study the generalization properties of the proposed estimators and show that these methods are statistically equivalent to the original structured prediction algorithm, while requiring less computations (\cref{sec:theory}). (c) We show that the learning rates and computational costs of the proposed methods are adaptive with respect to standard regularity properties of the learning problem (\cref{sec:refined-rates}), further reducing the complexity of the learning process when the problem is regular. (d) We provide a number of sufficient conditions to determine whether a loss admits an \structshort{}, which are easy to verify in practice (\cref{sec:self-sufficient-conditions}). We use the latter results to show that most loss functions used in machine learning applications satisfy the \structshort{} definition and therefore that our framework is suited to a large number of settings and applications.

The paper is organized as follows: \cref{sec:setting} introduces structured prediction within the framework of statistical learning theory for supervised problems. In \cref{sec:framework} we present the \structshort{} framework and the novel estimators. In \cref{sec:previous-literature} we draw extensive connections with previous work. \cref{sec:theory} is devoted to study the statistical and computational properties of the proposed estimators. \cref{sec:self-sufficient-conditions} studies sufficient conditions for a loss function to satisfy the \structshort{} definition. Finally, \cref{sec:conclusions} concludes the work discussing relevant future directions.

\section{Problem Setting and Background}\label{sec:setting}

We denote by $\X, \Y$ and $\Z$ respectively the {\em input space}, {\em label space} and {\em output space} of a learning problem. We let $\rho$ be a probability measure on $\X \times \Y$ and $\loss:\Z \times \Y \to \R$ be a loss function measuring prediction errors between a label $y\in\Y$ and an output $z\in\Z$. The distinction between $\Y$ and $\Z$ allows to consider applications where labels do not necessarily correspond to the desired outputs (e.g. ranking/information retrieval, see below).

\vspace{1em}\noindent{\bf Supervised Learning}.
 In supervised learning problems, the goal is to estimate a function $\fstar: \X \to \Z$ minimizing the {\em expected risk}
\eqal{\label{eq:expected-risk}
	\min_{f:\X\to\Z} \E(f), \qquad \textrm{with} \qquad \E(f) = \int \loss(f(x), y) ~d\rho(x,y),
}

\noindent over the set of measurable functions $f:\X\to\Z$. In practice, the distribution $\rho$ is given but unknown and only $n$ examples $(x_i, y_i)_{i=1}^n$ independently distributed according to $\rho$ are provided. Given a training set, a learning algorithm needs to find a good approximation $\fn:\X\to\Z$ to $\fstar$ such that the corresponding excess risk $\E(\fn)$ is close to $\E(\fstar)$ and tends to it  as the number $n$ of training points increases.

\vspace{1em}\noindent{\bf Empirical Risk Minimization.} A standard learning approach   is {\em Empirical Risk Minimization (ERM)} \citep[see e.g.][]{devroye2013probabilistic}. This method consists in obtaining the estimator $\fn:\X\to\Z$ as  
\eqal{\label{eq:abstract-erm}
	\fn ~=~ \argmin_{f\in\hh} ~ \frac{1}{n} \sum_{i=1}^n \loss(f(x_i),y_i),
}
by minimizing the empirical version of the expected risk over a suitable space $\hh$ of functions $f:\X\to\Z$. The idea underlying ERM is to use the empirical risk as an approximation of the expected risk, so that the estimation of $\fstar$ via $\fn$ should become increasingly  accurate as the number of training samples grows.

From a statistical perspective, it is sufficient for the loss $\loss$ to satisfy very general conditions (e.g. Lipschitz, bounded, etc.) in order for the ERM strategy to enjoy strong statistical guarantees. In particular, a number of results are available proving  universal consistency and learning rates (in terms of generalization or excess risk bounds) for the empirical risk estimator $\fn$, and a variety of hypotheses spaces $\hh$ \citep[see for instance ][and references therein]{shalev2014understanding}. 

From a computational perspective, a central question is whether the ERM problem can be solved efficiently. When the output space is linear, such as $\Z = \R$, and the loss $\loss$ is convex, ERM becomes an efficient strategy for a large family of hypotheses spaces. For instance, a standard approach is to consider linear parametrization of functions $f:\X\to\R$ in a Hilbert space $\hh$ of the form $f(x) = \scal{w}{\xmap(x)}_\hh$ with $w\in\hh$ and $\xmap:\X\to\hh$ a feature map. Following this approach, the resulting ERM problem in \Cref{eq:abstract-erm} can solved efficiently leveraging convex optimization techniques. This same strategy can be naturally extended to the general linear case $\Z = \R^M$.

\vspace{1em}\noindent{\bf Structured Prediction and Limitations of ERM.} When the space $\Z$ does not have a linear structure, applying ERM poses concrete challenges to both modeling and computations:


\begin{itemize}
       \item {\bf Modeling.} If the output space is non-linear, it is not  clear how to design a suitable hypotheses space $\hh$ of candidate estimators. In particular, linear parametrizations of the form $f(x) = \scal{w}{\xmap(x)}$ introduced above are not possible. For instance, given $f_1,f_2:\X\to\Z$, it is not guaranteed that $f_1+f_2:\X\to\Z$ takes values in $\Z$ as well. 

       \item {\bf Computations.} If the hypotheses space is non-linear or the loss is non-convex (e.g. integer-valued), solving  ERM  can be extremely challenging. Often, approaches based on the regularity of the loss function and the optimization domain, such as gradient methods,  cannot be adopted. Hence, it is not clear how to obtain $\fn$ in practice. 
\end{itemize}
Next we describe a number of problems falling in the above setting.

\subsection{Examples of Structured Prediction Problems}

Below we provide some examples of structured prediction problems according to our definition, that is problems where the output space lacks a linear structure. We refer to \citep{bakir2007,nowozin2011structured} for more examples.

\begin{itemize}
\item {\bf Classification, Multi-class, Multi-labeling.} In these settings $\Z = \Y = \{1,\dots,T\}$ is a collection of classes that can be associated to inputs from $\X$. 

\item {\bf Ranking, Ordinal Regression, Information Retrieval.} The goal is to predict an {\em ordered} list of documents \citep{bakir2007,pedregosa2017consistency,duchi2010}. For instance $x\in\X$ can be a query in a search engine and $\Z$ is the space of all permutations (ordering) over the documents in a database. The label space $\Y\neq\Z$ typically contains a set of scalar scores representing the individual relevance of each document to the input query.

\item {\bf Sequence Prediction.} The goal is to predict sequences such as time series for financial applications or planning trajectories \citep{ratliff2006maximum}. Loss functions such as the Dynamic Time Warping \citep{cuturi2017soft} can be used to measure the similarity between two sequences.

\item {\bf Predicting Probability Distributions / Histograms.} In these settings, the output $\Z$ corresponds to a space of probability distributions \citep{frogner2015learning,luise2018differential,mensch2019geometric}. The loss $\loss$ is a metric comparing probabilities, such as the Kullback-Libler divergence or the Hellinger, $\chi^2$ or Wasserstein distance.

\item {\bf Manifold Regression.} Problems where the outputs belong to a smooth Riemannian manifold $\Z$ \citep{steinke2010nonparametric,rudi2018manifold}. A natural choice for the loss $\loss$ is the squared geodesic distance of the manifold. This setting generalizes the standard regression problem with least-squares loss and $\Z = \R^M$, to the manifold scenario.

\end{itemize}
\section{A General Framework for Structured Prediction}\label{sec:framework}

In this section, we introduce and motivate our structured prediction framework. Our discussion starts from a useful characterization of $\fstar:\X\to\Z$ the {\em minimizer of the expected risk} in \Cref{eq:expected-risk}. Assume that $\rho$ can be factorized as $\rho(x,y) = \rho(y|x)\rhox(x)$ with $\rho(\cdot|x)$ the conditional distribution over $\Y$ given $x\in\X$ and $\rhox$ the marginal distribution of $\rho$ over $\X$. It can be shown that\eqal{\label{eq:fstar}
	\fstar(x) ~=~ \argmin_{z\in\Z}~ \int_{\Y} \loss(z,y)~d\rho(y|x),
}
that is,   the value $\fstar$ at  any given $x\in\X$ corresponds to the minimizer, over the output set $\Z$, of the conditional expectation $\EE_{y|x} \loss(z,y)$. 
Indeed, it is possible to show that if $\loss$ is measurable, then  such a point-wise estimate defines a measurable function. This latter result requires some care and follows from Berge maximum theorem \citep{aliprantis2006} (see also Aumann's principle \citep{steinwart2008}). We refer the reader to \Cref{sec:app-framework} for a detailed discussion.

In the following, we leverage  this characterization of $\fstar$  to develop our approach to  structured prediction. First,   {we consider the case  where both output and label spaces are discrete and finite}. As noted,  this  is relevant, since most previous work focused on this setting \citep{bakir2007}. Additionally, for our presentation, it allows a self-contained introduction of key ideas, deferring the technical details to the general discussion in   \Cref{sec:general-algorithm}.

\subsection[alt-text]{Motivating Analysis: Finite Output Spaces}\label{sec:finite-Y}
We begin by discussing how  loss functions define a geometry on finite output spaces, and how it can be used to define an estimator. Then, we can consider linear estimators and show how for this class of estimators a useful implicit formulation can be derived.

\paragraph{Geometry of Structured Loss Functions}
Let $\X = \R^d$ and assume $\Z = \Y = \{1,\dots,T\}$ for $T \in\N$. In this setting, any loss function $\loss:\Z\times\Y\to\R$ can be characterized in terms of a matrix $V\in\R^{T \times T}$ such that
\eqal{\label{eq:self-for-finite}
	\loss(z,y) ~=~ e_z^\top ~V~e_y, \qquad\qquad \forall ~z\in\Z,~y\in\Y,
}
where $e_y\in\R^T$ denotes the $y$-th element of the canonical basis of $\R^T$, namely the vector with $y$-th entry equal to $1$ and the rest equal to $0$. Combining this observation with the characterization of $\fstar$ in \Cref{eq:fstar} and using  linearity of the  integral, we have 
\eqal{\label{eq:fstar-in-terms-of-gstar}
	\fstar(x) ~ = ~ \argmin_{z\in\Z} ~ e_z^\top V ~ \gstar(x), \qquad\qquad \gstar(x) = \int_{\Y} e_y ~d\rho(y|x),
}
for any $x\in\X$. In particular, note that the function $\gstar:\X\to\R^T$ is given by 
\eqals{
	\gstar(x) = \big(~\rho\big(1|x\big)~,\dots,~\rho\big(T|x\big)~\big)^\top \in\R^T,
}

\noindent the vector whose $y$-th entry is equal to the probability of observing $y$, given $x$. This observation is crucial,  since it allows to identify $\gstar$ with the {\itshape regression function}, that is the minimizer of the expected least squares error (see \cref{lem:fstar-in-terms-of-gstar} for a formal statement)
$$
\gstar = \argmin_{g:\X\to\R^T}~\int  \nor{e_y- g(x)}_{\R^T}^2~d\rho(x,y).
$$
The above discussion suggests the following approach.
Given given $n$ training points $(x_i,y_i)_{i=1}^n$,  we could approximate $\gstar$ by a least squares estimate $\gn$
minimizing 
$$
\gn = \argmin_{g\in\mathcal{G}}~\frac{1}{n} \sum_{i=1}^n\nor{e_{y_i}- g(x_i)}_{\R^T}^2
$$
over some function space $\mathcal{G}$. Then we obtain the estimator $\fn:\X\to\Z$ 
 for any $x\in\X$ as
\eqal{\label{eq:fn-in-terms-of-gn}
	\fn(x) ~=~ \argmin_{z\in\Z}~ e_z^\top ~ V ~\gn(x).
}
The advantage of this strategy is that approximating $\gstar$ corresponds to a standard vector-valued regression problem, since the output space is now $\R^T$ and not the ``structured'' $\Z$.  As we discuss next, for linearly parameterized estimators, we can develop a useful implicit formulation.  We first add two remarks pointing out connections to related ideas.


%
%

\begin{remark}[Conditional mean embedding] By construction, $\gstar(x)$ defined in~\eqref{eq:fstar-in-terms-of-gstar} corresponds to the definition of {\em conditional mean embedding} of $\rho(\cdot|x)$ in $\R^T$ \citep{song2009hilbert,lever2012conditional,singh2019kernel}. In \Cref{sec:mean-embeddings}, this connection will provide relevant insights on the structured prediction estimator we propose and its statistical properties.
\end{remark}

\begin{example}[Classification]\label{ex:multi-class-classification}
For classification,  the estimator $\fn$ in \Cref{eq:fn-in-terms-of-gn}, recovers the least-squares classifier in \citep{yao2007early,mroueh2012}. To see this, let $\Z = \Y = \{1,\dots,T\}$ be a finite set of class labels and let $\loss$ be the $0$-$1$ (or mis-classification) loss, namely $\loss(z,y) = 1$ if $z\neq y$ and $0$ otherwise. It is easy to see that $\loss$ is of the form of \Cref{eq:self-for-finite} with matrix $V = \ones\ones^\top - I$, where $I$ is the $T \times T$ identity matrix and $\ones$ is the $T$-dimensional vector with all entries equal to $1$. 
 For any $y\in\Y$, the $y$-th entry of $\gn(x)\in\R^T$ is interpreted as the likelihood of observing a the class $y$ given the input $x\in\X$. Therefore, the classifier $c_n:\X\to\Y$ acts by predicting the index of $\gn(x)$ with higher likelihood. Direct comparison with our approach leads to 
\eqals{
	c_n(x) ~=~ \argmax_{t=1,\dots,T}~ \big(\gn(x)\big)_t~, \qquad\qquad \fn(x) ~=~ \argmin_{t=1,\dots,T}~ \big(V\gn(x)\big)_t~.
}
Since $V =\ones\ones^\top - I$, it is straightforward to see that the two methods coincide, namely $c_n(x) = \fn(x)$ for all $x\in\X$. 
\end{example}
\noindent We next describe a useful representation for linear estimators. 

\vspace{1em}

\noindent{\bf Implicit formulation for linear estimators.}
A possible approach to learn $\gn$ is by linear ridge regression, namely
\eqal{\label{eq:gn}
	\gn(x) ~=~ W_n ~ x, \qquad\qquad W_n ~=~ \argmin_{W\in\R^{T \times d}} ~ \frac{1}{n} \sum_{i=1}^n \|e_{y_i} - W x_i \|^2 + \la \|W\|_F^2~,
}
where $\la>0$ is a hyperparmeter and $\|\cdot\|_F^2$ denotes the squared Frobenius norm of a matrix (sum of the squared entries). The solution of \Cref{eq:gn} can be obtained in closed form as
\eqal{\label{eq:solution-linear-regression}
       W_n = Y^\top X ~ (X^\top X + n\la I)^{-1}
}
with $I\in\R^{d \times d}$ the identity matrix and $X\in\R^{n \times d}$ and $Y\in\R^{n \times T}$ the matrices with $i$-th row corresponding to $x_i$ and $e_{y_i}$ respectively. Plugging this solution in \Cref{eq:fn-in-terms-of-gn} leads to an explicit approach to obtain the estimator $\fn$. 
 We next discuss  a useful  observation that will be key to extend our discussion to $\Y$ and $\Z$ that are neither finite nor discrete. Specifically, we will show that it is possible to obtain a characterization for $\fn$ that is equivalent to \Cref{eq:fn-in-terms-of-gn} but in which $\gn$ {\em does not appear explicitly}. To see this, first notice that for any $x\in\X$ we can leverage the closed-form solution for the estimator $\gn$ to have
\eqal{\label{eq:gn-as-weighted-sum-finite}
	\gn(x) ~=~ W_n x ~=~ Y^\top \alpha(x) ~=~ \sum_{i=1}^n \alpha_i(x)~e_{y_i}~,
}
where the weights $\alpha:\X\to\R^n$ are such that
\eqals{\label{eq:alpha-linear}
	\alpha(x) ~=~ (\alpha_1(x),\dots,\alpha_n(x))^\top ~=~ [X (X^\top X + n\la I)^{-1}]~ x ~\in\R^n.
}
We now plug this characterization of $\gn$ in the definition of $\fn$ in \cref{eq:fn-in-terms-of-gn}. Thanks to the linearity of the sum and the matrix-vector product, we have 
\eqal{\label{eq:fn-alternative}
	\fn(x) ~=~  \argmin_{z\in\Z} ~ \sum_{i=1}^n \alpha_i(x)~ e_z^\top~ V ~ e_{y_i} ~=~  \argmin_{z\in\Z} ~ \sum_{i=1}^n \alpha_i(x)~ \loss(z,y_i),
}
where the last equality follows from the fact that the loss $\loss$ is identified by the matrix $V$ according to \Cref{eq:self-for-finite}. Intuitively, for any $i=1,\dots,n$, we can interpret each $\alpha_i(x)$ as a relevance score encouraging the candidate prediction $z\in\Z$ to be ``similar'' to the observed training label $y_i$, {\em according to $\loss(z,y_i)$}.

The key observation in \Cref{eq:fn-alternative} is that the estimator $\fn$ can be characterized exclusively in terms of the weights $\alpha$ and the observed labels $y_i$. This implies that the cardinality of the output and label spaces does not directly affect the proposed approach. In \Cref{sec:general-algorithm} we will leverage this observation to extend the same learning strategy to the case where $\Y$ and $\Z$ are not finite or discrete.

\vspace{1em}\noindent{\bf Extension to generic $\X$}. We conclude this section by observing that the construction of $\fn$ can be naturally extended to the case where $\X$ is a generic set. In particular, consider $k:\X\times\X\to\R$ a positive definite kernel \citep{aronszajn1950theory}. Then, according to standard practice from the kernel methods literature \citep[see e.g.][]{shawe2004kernel} we can derive a ``dual'' formulation for the relevance scores $\alpha$. In particular, given the input $(x_i)_{i=1}^n$ in training, \Cref{eq:alpha-dual} can be equivalently rewritten as
\eqal{\label{eq:alpha-dual}
	\alpha(x) ~=~ (\alpha_1(x),\dots,\alpha_n(x))^\top ~=~ (K + n\la I)^{-1} ~\msf{v}(x) ~ \in\R^n,
}
for any $x\in\X$, where $K\in\R^{n \times n}$ is the empirical kernel matrix with entries $K_{ij} = k(x_i,x_j)$ and $\msf{v}(x)\in\R^n$ is the evaluation vector, with entries $\msf{v}(x)_i = k(x,x_i)$, for any $i,j=1,\dots,n$. In the following, we will denote $\kappa^2 = \sup_{x\in\X}k(x,x)$. We will always assume $\kappa<+\infty$ (e.g. by using a normalized kernel or by requiring $\X$ to be a compact set). It is easy to see that this strategy corresponds indeed to learn the estimator $\gn$ by solving the empirical risk minimization problem in \Cref{eq:gn} over the reproducing kernel Hilbert space (RKHS) associated to $k$ \citep{aronszajn1950theory}. We discuss this in more detail in the following.

\subsection{General Structured Prediction: Beyond Finite Output Spaces}\label{sec:general-algorithm}

In this section,  we generalize the discussion of \Cref{sec:finite-Y}
to structured prediction problems where $\Y$ or $\Z$ are not necessarily finite (or discrete).
Also in this case, we  show how  a relevant geometry can be  defined by a corresponding
loss function.  Further we  extend the analysis of linearly parameterized estimators, and show how in this general setting the implicit formulation becomes essential.

%
%
%
%
%

\vspace{1em}\noindent{\bf \struct{}s.} The extension to non finite output spaces hinges 
on a key assumption on the loss that generalizes the observation of \Cref{eq:self-for-finite}. We refer to functions satisfying this condition as admitting an {\em \struct{}}.
\begin{definition}[\structshort{}]\label{def:self}
A continuous map $\loss:\Z\times\Y\to\R$ is said to admit an {\em \struct{} (\structshort{})} if there exists a separable Hilbert space $\hh$ and two measurable bounded maps $\zmap:\Z\to\hh$ and $\ymap:\Y\to\hh$, such that for any $z\in\Z$ and $y\in\Y$ we have
\eqals{\label{eq:self}
	\loss(z,y) ~=~ \big\langle~ \zmap(z)~,~\ymap(y)~\big\rangle_\hh,
}
and $\|\ymap(y)\|_\hh\leq1$. Additionally, we define $\closs =\sup_{z\in\Z} ~ \|\zmap(z)\|_\hh$.
\end{definition}
The definition of \structshort{} is similar to the characterization of positive definite kernels in terms of feature maps (and indeed it recovers such definition when $\Z = \Y$ and $\zmap=\varphi$), but is significantly more general in that it allows also to consider functions that are not positive definite (for example, distances) or even not symmetric (such as divergences). It is clear that any loss function on finite sets $\Z$ and $\Y$ admits an \structshort{}. For instance, in the setting of \Cref{sec:finite-Y} it is sufficient to choose $\hh=\R^T$, with maps $\ymap(y) = e_y$ and $\zmap(z) = V^\top e_z$, to recover the \structshort{} definition. Note that the requirement $\sup_{y\in\Y}\|\ymap(y)\|\leq1$ is introduced to simplify the notation but does not limit the generality of the assumption (see \cref{lem:self-equivalence-not-normalized} in the Appendix for more details). We note that in \cite{ciliberto2016}, a variant of the \structshort{} property was introduced (see Asm $1$ in such paper). In this work we opted for \cref{def:self} since it allows for a more clear notation in the following. However, in \cref{prop:equivalence-self-old} in the Appendix we provide more details on this point and show that the two notions are actually equivalent.

While the definition of \structshort{} is  abstract, it is satisfied by many loss functions often used in structured prediction applications and more generally in machine learning. In \Cref{sec:self-sufficient-conditions} we present a wide range of sufficient conditions to guarantee a function to admit an \structshort{}, which are easier to interpret and verify in practice.

Under the assumption that $\loss$ admits an \structshort{}, we can easily retrace the reasoning in \Cref{sec:finite-Y} to derive the structured prediction estimator. In particular, we have the following result to which we refer to as {\em Fisher consistency}, a term borrowed from the literature of surrogate methods (see discussion in \Cref{sec:surrogate-frameworks}).
\begin{restatable}[Fisher Consistency]{lemma}{LFstarCharacterization}\label{lem:fstar-in-terms-of-gstar}
Let $\Z$ be compact, $\loss:\Z\times\Y\to\R$ admit an \structshort{} and let $\fstar:\X\to\Z$ be the solution of \Cref{eq:expected-risk}. Then,
\eqal{\label{eq:fstar-in-terms-of-gstar-full}
	\fstar(x) ~=~ \argmin_{z\in\Z}~ \scal{~\zmap(z)~}{~\gstar(x)~}_\hh, \qquad\qquad 
	\gstar(x) = \int_{\Y} \ymap(y)~d\rho(y|x)
}
almost surely with respect to $\rhox$. Moreover, $\gstar:\X\to\hh$ is the minimizer of 
\eqal{\label{eq:surrogate-risk}
	\rr(g) ~=~ \int_{\Y\times\X} \|\ymap(y) - g(x)\|_\hh^2~d\rho(x,y).
}
\end{restatable}
The result above provides a characterization of $\fstar$ in terms of the conditional expectation of $\ymap(y)$ with respect to $x\in\X$. It generalizes \Cref{eq:fstar-in-terms-of-gstar} to the case where $\Y$ and $\Z$ are not finite. 
Analogously to \Cref{eq:fstar}, the proof of \Cref{lem:fstar-in-terms-of-gstar} is reported in \Cref{sec:app-framework} and leverages Berge's Maximum theorem. In particular, the compactness of $\Z$ is a technical requirement to guarantee $\fstar(x)$ to be well-defined. 
Analogously to the derivation in \Cref{sec:finite-Y}, the result of \Cref{lem:fstar-in-terms-of-gstar} motivates us to design a structured prediction estimator by first learning a $\gn$ 
to approximate $\gstar$
via least squares over a space $\mathcal{G}$ of functions $g:\X\to\hh$
$$
\gn ~=~ \argmin_{g\in\mathcal{G}}~ \frac{1}{n}\sum_{i=1}^n \|\ymap(y_i) - g(x_i)\|_\hh^2,
$$
and then plug $\gn$ in \Cref{eq:fstar-in-terms-of-gstar-full} to obtain an approximation $\fn$ of $\fstar$ characterized by
$$
\fn(x) ~=~ \argmin_{z\in\Z}~ \scal{~\zmap(z)~}{~\gn(x)~}_\hh,
$$
for all $x\in \X$. Learning $\gn$ corresponds to solving a vector-valued regression problem on a (possibly) infinite dimensional output space $\hh$ \citep{caponnetto2007}.
We next discuss in detail the case of linearly parameterized estimators.

\vspace{1em}\noindent{\bf Linearly parameterized estimators and implicit formulation.} For simplicity, instead of directly minimizing the emprical square loss over $\hh$, we consider again the ridge regression estimator $\gn:\X\to\hh$ defined as the minimizer of the regularized empirical risk
\eqal{\label{eq:gn-estimator-full}
	\gn ~=~\argmin_{g\in\gg} ~ \frac{1}{n}\sum_{i=1}^n \|\ymap(y_i) - g(x_i)\|_\hh^2 + \lambda \|g\|_\gg^2,
}
over a normed space $\gg$ of vector-valued functions from $\X$ to $\hh$. A viable choice for $\gg$ is, given a kernel $k:\X\times\X\to\R$ with associated {\em reproducing kernel Hilbert space (RKHS)} $\ff$, to consider $\gg = \hh\otimes\ff$, which corresponds  to the RKHS of vector-valued functions (see \citep{micchelli2004,alvarez2012kernels}) with operator-valued kernel $\Gamma(x,x') = k(x,x') I_\hh $ and $I_\hh:\hh\to\hh$ the identity operator on $\hh$. This approach is a direct generalization of the strategy introduced in the finite setting. Indeed, we have already observed that, when $\Y=\Z$ is a finite set, we can choose $\hh=\R^T$ to satisfy the \structshort{} definition. Moreover, for the linear kernel $k(x,x') = \scal{x}{x'}$ on $\X = \R^d$, the associated RKHS $\ff$ is isometric to $\R^d$. Therefore, we have $\hh\otimes\ff \cong \R^T \otimes \R^d \cong \R^{T \times d}$. Consequently, any $\gn\in\hh\otimes\ff$ can be parametrized by a matrix $W_n\in\R^{T \times d}$ and the ERM problem in \Cref{eq:gn-estimator-full} becomes equivalent to the one in \Cref{eq:gn}.

The solution of the ridge regression problem can be obtained in closed form. In particular, it is easy to prove that analogously to \Cref{eq:gn-as-weighted-sum-finite} in the finite setting, for any $x\in\X$ we have
\eqal{\label{eq:gn-as-weighted-sum}
	\gn(x) ~=~ \sum_{i=1}^n \alpha_i(x) ~\ymap(y_i),
}
with the weights $\alpha(x) = (K + \lambda n I)^{-1}~ \msf{v}(x)$ as in \Cref{eq:alpha-dual}. By replacing $\gn$ to $\gstar$ in \Cref{lem:fstar-in-terms-of-gstar} we recover the estimator $\fn$ of \Cref{eq:estimator}, as desired. As we will discuss in more detail in \Cref{sec:kde}, this strategy is related to the {\em Kernel Dependency Estimator (KDE)} originally proposed in \citep{weston2002} for kernel-based loss functions. 

The above reasoning can be applied to {\em any} function $\gn$ expressed as a linear combination of (embedded) output points $\ymap(y_i)$. The following result summarizes this property, which allows to consider a family of novel estimators $\fn$ paramterized by the weighting function $\alpha$. 

\begin{restatable}{lemma}{PLossTrick}\label{prop:derivation}
Let $\loss:\Z\times\Y\to\R$ admit an \structshort{}, $(y_i)_{i=1}^n$ a set of points in $\Y$ and $\alpha:\X\to\R^n$ a weighting function. Let $\gn:\X\to\hh$ be such that $\gn(x) = \sum_{i=1}^n \alpha_i(x)\ymap(y_i)$ for any $x\in\X$. Then, the function $\fn:\X\to\Z$ such that $\forall x\in\X$
\begin{equation}\label{eq:estimator}
	\fn(x) ~=~ \argmin_{z\in\Z} ~ \scal{~\zmap(z)~}{~\gn(x)~}_\hh = \argmin_{z\in\Z}~ \sum_{i=1}^n \alpha_i(x)~\loss(z,y_i).
\end{equation}
\end{restatable}
From the observation above, we see that the reasoning in \Cref{sec:finite-Y} can indeed be generalized to the setting where $\Z$ and $\Y$ are not finite and $\loss$ admits an \structshort{}. For any $x\in\X$, the weights $\alpha(x)$ are learned from training data according to \Cref{eq:alpha-dual} and only the loss function appears in the form of the estimator. We expand on this in the following remark.
\begin{remark}[``Loss Trick'']
In practice, learning and evaluating $\fn$ does not require explicit knowledge of the space $\hh$ or the embeddings $\zmap$ and $\ymap$ (see \Cref{eq:estimator}), which are implicitly encoded within the definition of \structshort{} and are only required for theoretical purposes (derivation and characterization of generalization properties of $\fn$ as discussed in \cref{sec:theory}). This effect was originally referred to as ``loss trick'' \citep{ciliberto2016} in analogy to the well-known ``kernel trick''~\citep{scholkopf2002}. 
\end{remark}

\subsection{Additional \structshort{}-induced Algorithms and Estimators}\label{sec:alternative-alpha}

In this section we discuss alternative approaches to learn the weighting functions $\alpha:\X\to\R^n$ defining the estimator $\fn$. These strategies are derived by replacing kernel ridge-regression with a different approximation of $\gstar$ that still satisfies the hypotheses of \Cref{prop:derivation}, namely can be written as the linear combination of training outputs. As already observed in the literature of standard regression settings \citep[see e.g. ][and references therein]{rosasco2005spectral}, these alternative approaches can offer significant computational advantages over kernel ridge regression while guaranteeing the same generalization performance from the statistical standpoint.


\vspace{1em}\noindent{\bf ``Exact'' Kernel methods.} A wide family of algorithms that can be used to estimate $\gstar$ are based on {\em spectral filtering} regularization strategies \citep{rosasco2005spectral,bauer2007regularization}. In particular we consider:
\begin{itemize}

\item {\bf $L^2$-Boosting (L2B).} By considering $\gn$ the $t$-th iterate of gradient descent of the (non-regularized) empirical risk minimization problem in \Cref{eq:gn-estimator-full}, we have 
\eqal{\label{eq:kboosting}
	\alpha(x) = C_t ~\msf{v}(x) \qquad\qquad \textrm{with} \qquad\qquad C_{t} = (I - \nu/n ~ K) ~ C_{t-1} +\nu/n~ I,
}
with $C_t\in\R^{n \times n}$ defined recursively starting from any $C_0\in\R^{n \times n}$ and $\nu/n$ the gradient descent step size with $\nu>0$. We recall that  $\msf{v}(x)\in\R^n$ denotes the evaluation vector in $x$, with entries $\msf{v}(x)_i = k(x,x_i)$ for any $i=1,\dots,n$. The number of steps $t\in\N$ acts as regularization parameter. Accelerated and stochastic versions can be considered.


\item {\bf Principal Component Regression (PCR).} Take $\gn$ to be the estimator obtained by filtering out the eigenvalues of the kernel matrix $K$ below a threshold $\la>0$ and inverting the eigenvalues that are above. We have
\eqal{\label{eq:kpcr}
       \alpha(x) = U \Sigma_{\la}^\dagger U^\top \msf{v}(x),
}
where $K = U \Sigma U^\top$ is the singular value decomposition of $K$ and $\Sigma_{\la}^\dagger$ is the pseudoinverse of the matrix corresponding to $\Sigma$ with all eigenvalues smaller than $\la$ set to $0$.

\end{itemize}

\noindent{\bf Random Projections.} Methods leveraging random projections achieve optimal generalization performance while being significantly more efficient computationally.

\begin{itemize}
       \item {\bf Random Features (RF).} Let $(\Omega, \pi)$ be a probability space and $\zeta: \X \times \Omega \to \R$ be a map such that $k(x,x') = \int_\Omega \zeta(x,\omega) \zeta(x',\omega) d\pi(\omega)$ \citep{rahimi2008random} \footnote{E.g.  for the Gaussian kernel and $\X = \R^d$, $d \in \N$, we have $\Omega = \R^d \times [0,1]$, and for $(w,b) =: \omega \in \Omega$, $\pi((w,b)) = {\cal N}(w)U(b)$, with ${\cal N}(\cdot)$ standard normal distribution $U(\cdot)$ uniform distribution, and $\zeta(x,(w,b)) = \cos(w^\top x + b)$ \citep[see][for more details]{rahimi2008random}.}.
       Let $M \in \N$ and $\omega_1,\dots,\omega_M$ be  independently sampled from $\pi$. Denote by $\hat{\msf{v}}_M: \X \to \R^M$, the map
       \eqals{
       \hat{\msf{v}}_M(x) = \frac{1}{\sqrt{M}}(\zeta(x,\omega_1),\dots,\zeta(x,\omega_M)).
       }
       By definition $\hat{\msf{v}}_M(x)^\top \hat{\msf{v}}_M(x')$ is a discretization of the integral defining $k(x,x')$. The scores $\alpha$ are learned according to this new feature map 
       \eqal{\label{eq:randf}
              \alpha(x) ~=~ W ~ \hat{\msf{v}}_M(x), \qquad\qquad W = Q_M (Q_M^\top Q_M + n\la I)^{-1},
       }
       with $Q_M \in \R^{n \times M}$, $Q_M = (\hat{\msf{v}}_M(x_1), \dots, \hat{\msf{v}}_M(x_n))^\top$.
       This approach is significantly faster than ridge-regression when $M \ll n$. 


       \item {\bf Nystrom Approximation (NY).} Sample $M\leq n$ points $\widetilde{x}_1,\dots, \widetilde{x}_M$ from the input dataset. Denote $K_{MM}\in\R^{M \times M}$ be the matrix with $(K_{MM})_{i,j} = k(\widetilde{x}_i, \widetilde{x}_j)$ and $K_{nM}\in\R^{n \times M}$ the matrix with elements $(K_{nM})_{ij} = k(x_i, \widetilde{x}_j)$  \citep[see][]{smola2000sparse}. The scores $\alpha$ are defined as
       \eqals{\label{eq:nystrom}
              \alpha(x) ~=~ W ~ \widetilde{\msf v}_M(x), \qquad\qquad W = K_{nM}~(K_{nM}^\top K_{nM} + n \la K_{MM})^{\dagger},
       }
   with $\widetilde{\msf v}_M(x) = (k(\widetilde{x}_1,x),\dots, k(\widetilde{x}_M, x)) \in \R^{M}$, for any $x \in \X$. These operations are significantly faster than solving ridge-regression when $M \ll n$.

\end{itemize}

\noindent{\bf Nadaraya-Watson (NW).} $\gn$ can be obtained via the Nadardaya-Watson (NW) strategy \citep{nadaraya1964estimating}. In this case we have 
\eqals{
       \alpha(x) ~=~ \frac{1}{\ones^\top ~\msf{v}(x)}~\msf{v}(x),
}
resulting in the estimator 
\eqal{
	\fn(x) ~=~ \argmin_{z\in\Z}~ \sum_{i=1}^n ~\frac{k(x,x_i)}{\sum_{j=1}^n k(x_j,x_i)}\loss(z,y_i)
}

\vspace{1em}\noindent Computation of the NW estimator does not involve the kernel matrix $K$. This can be beneficial in large-scale scenarios where the kernel matrix can be large. However, the NW estimator is often less adaptive than ridge-regression to the smoothness properties of the learning problem \citep{gyorfi2006distribution}. As a consequence, the learning rates of NW are usually slower than kernel ridge regression in non-worst-case settings.



\vspace{1em}\noindent{\bf Nearest Neighbors (NN).} Given a measure of similarity on the input set (e.g. a kernel), for any test point $x\in\X$, the Nearest Neighbor (NN) estimator \citep{friedman2001elements} corresponds to the average, on the space $\hh$ of output training points $\ymap(y_i)$ whose corresponding inputs $x_i$ are among the first $\msf q$ most similar (or closest) to $x$. This corresponds to return the {\em binary} scores $\alpha(x)\in\{0,1\}^n$
\eqals{
       \alpha(x)_i = \left\{\begin{array}{cl} 1 & \textrm{if}~ x_i~\textrm{is a}~ \msf q\textrm{-nearest-neighbor} \\ 0 & \textrm{otherwise} \end{array}\right..
}

\vspace{1em}\noindent NN does not rely on a training phase (except for a possible pre-processing such as kd-trees to allow for a faster search of neighbors at test time). Interestingly, at test time, the cost of the optimization over $\Z$ in \Cref{eq:estimator} depends also on the number $\msf q$ of neighbors selected, which is a hyperparamter of the estimator.

\vspace{1em}\noindent So far we have introduced a novel family of estimators for structured prediction and discussed how they can be derived from the notion of \struct{}. We are left with two critical questions that will be addressed in the following: on one hand we need to characterize how the approximation of $\gn$ can lead to good estimations of $\fstar$. A second, more concrete question is whether the \structshort{} definition is sufficiently flexible to encompass relevant structured prediction problems or, in other words, which functions admit an \structshort{}. We will address the first question in \Cref{sec:theory} and the second one in \Cref{sec:self-sufficient-conditions}. Before doing so, in \cref{sec:previous-literature} we draw some connections with previous work. These observations will prove useful to better situate our framework within the literature on structured prediction. We conclude this section with a comment on evaluating the \structshort{} in practice.

\subsection{Evaluating the \structshort{} Estimator}

According to \cref{eq:estimator}, evaluating $\fn$ on a test point $x\in\X$ consists in solving an optimization problem over the output space $\Z$. This design of the test phase is standard in structured prediction settings and in particular for likelihood estimation methods (see \Cref{sec:likelihood-estimation}), where a corresponding optimization protocol is derived on a case-by-case basis depending on the loss and the space $\Z$,  \citep[see][and references therein]{nowozin2011structured}. The objective functional in our setting allows also to suggest a general stochastic meta-algorithm. In particular, \cref{eq:estimator} can interpreted as the problem of minimizing an expectation
\eqal{\label{eq:evaluating-as-sgd}
\fn(x) = \argmin_{z\in\Z} ~ \mathbb{E}_{i\sim\alpha(x)}~ h_i(z,x) \qquad \textrm{with} \qquad h_i(z,x) = \frac{\textrm{sign}(\alpha_i(x))}{\msf a(x)}~ \loss(z,y_i)
}
where $i\in\{1,\dots,n\}$ is a random variable sampled according to the relevance weights $\alpha_i(x)$ and $\msf a(x) = \sum_{i=1}^n \alpha_i(x)$. Thus, when $\loss$ is (sub)differentiable in the first variable, problems of the form of \cref{eq:evaluating-as-sgd} can be directly addressed addressed by stochastic gradient methods (SGM).

\begin{remark}[On the Complexity of Inference]
Given the scores $\alpha_i(x)$, solving the inference problem to obtain $\fn(x)$ requires solving a possibly hard optimization problem. However, in most settings, this approach can be more favorable than ERM. Indeed, consider for simplicity the case where both $\X$ and $\Z$ are finite spaces with cardinality $|\X|$ and $|\Z|$ respectively. ERM would require solving an optimization problem on the space of all functions from $\X$ to $\Z$, which has cardinality $|\Z|^{|\X|}$. On the other hand, the \structshort{} approach acts by first learning the scores $\alpha_i(x)$, which is done efficiently by solving a linear system and then finding the best output $f(x)$ over the space $\Z$. This amounts to solving an optimization over a space of cardinality $|\Z|$, which is significantly smaller than $|\Z|^{|\X|}$.
\end{remark}

\section{Connections with Previous Work}\label{sec:previous-literature}

In this section we highlight some relevant connections between our framework and previous literature. As mentioned in \Cref{sec:setting} we show that, although starting from a different perspective, our method can be interpreted as a synthesis of the two main structured prediction strategies considered in the literature, namely {\em surrogate frameworks} and {\em likelihood estimation methods}. In this sense, our approach represents the best of both worlds, since it is theoretically sound (as we will show in \Cref{sec:theory}) but it is also applicable to a large family of learning problems. We also draw a connection with the {\em conditional mean embeddings} literature, which will offer relevant insights on the theoretical analysis of \Cref{sec:theory}.

\subsection{Surrogate Frameworks}\label{sec:surrogate-frameworks}
Surrogate approaches are designed to address specific structured prediction problems such as classification \citep{bartlett2006,mroueh2012}, multi-labeling \citep{gao2013}, ranking \citep{duchi2010} or quantile regression \citep{steinwart2011} to name a few. The core idea underlying these methods is to deal with the structure of the problem by: $1)$ finding an embedding (or {\em encoding}) of the output variables into a linear space where, $2)$ a {\em surrogate} learning problem can be solved efficiently and finally, $3)$ map back the surrogate solution by means of a suitable {\em decoding}. 

More formally, given a training dataset $(x_i,y_i)_{i=1}^n$ , a surrogate approach consists in the following three steps:
\begin{enumerate}
\item {\bf Encoding.} Choose a coding $\coding:\Y\to\hh$ into a surrogate space $\hh$. \\
\phantom{{\bf Encoding.}}~Map $(x_i,y_i)_{i=1}^n$ to the surrogate dataset $(x_i,\coding(y_i))_{i=1}^n$.
\item {\bf Learning.} Define a surrogate loss $\mathcal{L}:\H\times\H\to\R$.\\
\phantom{{\bf Learning.}}~Learn $g_n:\X\to\hh$ minimizing $\mathcal{L}$ on $(x_i,\coding(y_i))_{i=1}^n$.
\item {\bf Decoding.} Choose a decoding $\decoding:\hh\to\Z$ and define $f_n = d \circ g_n:\X\to\Y$.
\end{enumerate}
A prototypical example of this strategy is represented by binary classification. 
\begin{example}[Binary Classification]\label{ex:surrogate-binary-classification}
In binary classification the goal is to learn a binary-valued function $f:\X\to\Z = \{-1,1\}$. The prototypical approach to address this problem is to consider $\coding:\Y = \{-1,1\}\to\hh=\R$ the identity map and then learn $g_n:\X\to\R$ by minimizing a suitable loss $\mathcal{L}:\R\times\R\to\R$ (e.g. least-squares, hinge, logistic, etc.). The final classifier is then obtained as $f_n(x) = \sign(g_n(x))$, with decoding $\decoding = \sign:\R\to\{-1,1\}$.
\end{example}
Surrogate frameworks critically hinge on identifying a suitable candidate for the loss function $\mathcal{L}:\hh\times\hh\to\R$. Indeed, while on one hand $\mathcal{L}$ should allow to compute the estimator $\gn:\X\to\hh$  efficiently, on the other hand the surrogate learning process induced by $\mathcal{L}$ needs to be related to the original structured prediction problem. The requirement for efficiency is typically satisfied by choosing $\mathcal{L}$ to be a convex loss  (e.g. least-squares, hinge or logistic in binary classification, see \Cref{ex:surrogate-binary-classification}). The connection with structured prediction is investigated by studying the ideal learning problem induced by the {\em surrogate risk}
\eqal{
       \rr(g) = \int \mathcal{L}(g(x),\coding(y))~d\rho(x,y),
}
with $g:\X\to\hh$. Intuitively, for a ``good'' surrogate framework, the global minimizer $\gstar:\X\to\hh$ of the risk $\rr$ should allow to recover the original solution $\fstar:\X\to\Z$ by means of the decoding, for instance $\fstar = \decoding\circ\gstar$. Moreover, ideally, as the number $n$ of training points increases and the estimator $\gn$ provides a better approximation to $\gstar$, we would like the predictor $\fn$ to converge to $\fstar$ as well.

Formalizing the observations above, the following two conditions are typically required by Surrogate Frameworks:
\begin{itemize}
       \item {\bf Fisher Consistency.} $\E(\fstar) = \E(\decoding \circ \gstar)$,
       \item {\bf Comparison Inequality.} $\E(\decoding \circ g) - \E(\fstar) \leq \gamma(\rr(g) - \rr(\gstar))$ for any $g:\X\to\hh$,\\ with $\gamma:\R_+\to\R_+$ continuous, non-decreasing and such that $\gamma(0)=0$.
\end{itemize}

\vspace{1em}\noindent The Fisher consistency establishes the validity of the surrogate framework in terms of the original problem. It guarantees that we can always recover the ideal $\fstar$ from the surrogate solution $\gstar$. The comparison inequality allows to automatically extend any result characterizing the learning rates of the surrogate estimator to obtain explicit excess risk bounds for the structured prediction one (possibly accelerated or slowed down by a factor depending on the function $\gamma$). We refer to \cite{mroueh2012} and references therein for concrete examples of this strategy in a variety of structured prediction settings.

\vspace{1em}\noindent{\bf Surrogate Frameworks and \structshort{}.} While surrogate methods are typically designed on a case-by-case basis for each learning problem, the structured prediction framework proposed in this work can be interpreted as a general form of surrogate approach. In particular, it is natural to choose the encoding map as the \structshort{} feature map on the label space $\Y$, namely $\coding = \ymap$. Moreover, we have observed how any \structshort{} loss function is directly associated to a suitable surrogate output space $\hh$ via the \structshort{} definition itself. In particular, in \Cref{lem:fstar-in-terms-of-gstar} we have shown how the corresponding structured prediction problem is naturally associated to the surrogate risk $\rr$ with surrogate loss $\mathcal{L}(\eta_1,\eta_2) = \|\eta_1 - \eta_2\|_\hh^2$ the square loss on $\hh$. It follows that we can choose as decoding the map $\decoding:\hh\to\Z$ such that
\eqals{\label{eq:decoding}
       \decoding(\eta) ~ = ~ \argmin_{z\in\Z}~ \scal{~\zmap(z)~}{~\eta~}_\hh
}
for any $\eta\in\hh$. Indeed, with the notation of \Cref{sec:framework}, we have $\fn = \decoding \circ \gn$ according to \Cref{eq:estimator} that recovers our structured prediction estimators. Note in particular that \Cref{lem:fstar-in-terms-of-gstar} shows that our framework is Fisher consistent, by proving that indeed $\fstar = \decoding \circ \gstar$ as required. 

The connection with surrogate methods will be completed by our theoretical analysis of \Cref{sec:theory}. Indeed, analogously to surrogate approaches, our proof strategy hinges on proving a comparison inequality, which allows to study the generalization properties on the surrogate problem to control the excess risk of the structured prediction estimator. In particular, in \Cref{prop:comparison-inequality} we provide a comparison inequality for our framework with $\gamma$ corresponding to the square root function.



\subsection{Kernel Dependency Estimation}\label{sec:kde}

Kernel Dependency Estimation (KDE) was originally proposed in \citep{weston2002,cortes2005} to address general structured prediction problems. The core strategy of KDE methods is to substitute the original loss function with an alternative one, which might offer significant computational advantages. In KDE settings, the new loss corresponds to the canonical distance induced by a reproducing kernel on the output space. This choice is motivated by the intuition that the solution of the two problems should be close to each other, {\em provided that the original loss behaves ``coherently'' with the metric induced by the kernel between output predictions}. 

More formally, let assume $\Y = \Z$ for simplicity and consider $h:\Y\times\Y\to\R$ a reproducing kernel on the output space. Let $\hh$ be the RKHS associated to $h$, with feature map $\ymap:\Y\to\hh$, namely $h(y,y') = \scal{\ymap(y)}{\ymap(y')}_\hh$ for any $y,y'\in\Y$. KDE methods address the problem of learning a function $f:\X\to\Y$ by minimizing the least-squares loss on $\hh$, 
\eqals{
       \loss(f(x),y) = \|\ymap(f(x)) - \ymap(y)\|_\hh^2,
}
for $x\in\X$ and $y'\in\Y$. Given a dataset $(x_i,y_i)_{i=1}^n$, learning $\fn$ directly might be challenging because of the structure of $\Y$. Alternatively, it might be possible to leverage the linear structure of $\hh$ to learn a function $\gn:\X\to\hh$. Whenever a test point $x\in\X$ is provided, the KDE prediction $\fn(x)$ is then obtained by finding the $y\in\Y$ for which $\ymap(y)$ is closest to $\gn(x)$ according to the canonical distance on $\hh$. This phase, akin to the decoding of surrogate methods, is referred to as the {\em preimage problem} in KDE settings. 

Interestingly, when $\gn(x) = \sum_{i=1} \alpha_i(x)\ymap(y_i)$, given the score functions $\alpha$, this problem can be cast as the optimization
\eqal{\label{eq:kde-decoding}
       \fn(x) & ~=~ \argmin_{y\in\Y}~\|\ymap(y) - \gn(x)\|_\hh^2 ~=~ \argmin_{y\in\Y} ~h(y,y) - 2\sum_{i=1}^n\alpha_i(x)h(y,y_i).
}
Indeed, thanks to the reproducing property of the kernel $h$ we have 
\eqals{
       \|\ymap(y) - \ymap(y')\|_\hh^2 ~=~ h(y,y) - 2h(y,y') + h(y',y'),
}

\vspace{1em}\noindent for every $y,y'\in\Y$. Several approaches to KDE have been proposed, considering different strategies based on estimating the scores $\alpha$ via variants of a Kernel Ridge Regression strategy akin to the one described in \Cref{sec:framework} \citep{weston2002,cortes2005}. In particular, extension of KDE were then considered in \citep{kadri2013generalized,brouard2016input} leveraging ideas from the literature on vector-valued reproducing kernel Hilbert spaces.

\vspace{1em}\noindent{\bf Kernel Dependency Estimation and \structshort{}.} There is a clear connection between KDE approaches and the \structshort{} framework considered in this work. Indeed, as we will show in \Cref{sec:self-sufficient-conditions} (\Cref{thm:rkhs-and-self}), the loss function considered by KDE algorithms satisfies the \structshort{} definition. Moreover, if the output kernel $h$ is such that $h(y,y) \equiv \eta$ for any $y\in\Y$, with $\eta>0$ a constant, then the \structshort{} estimator in \Cref{eq:estimator} corresponds exactly to KDE (assuming same scores $\alpha$). The latter observation implies that the theoretical analysis reported in \Cref{sec:theory} applies also to KDE. Therefore, a further byproduct of our work is the theoretical analysis of KDE strategies, which to our knowledge is a novel contribution on its own.

We conclude this section by highlighting two critical differences between our framework and KDE methods:
\begin{itemize}
       \item KDE was designed to address structured prediction problems by substituting the original structured loss with the least-squares induced by a kernel on the output. There is no guarantee in general that the KDE estimator is in any way solving the structured prediction. In this sense KDE could be interpreted as a form of surrogate method for which the surrogate problem does not satisfy neither Fisher consistency nor Comparison Inequality. 

       \item If the condition $h(y,y)\equiv\eta$ does not hold, the KDE and \structshort{} estimators {\em do not coincide}. This means that there is no guarantee that the KDE approaches will enjoy the same generalization properties of \structshort{} methods studied in \Cref{sec:theory}.
\end{itemize}

\subsection{Likelihood Estimation Methods and Structured SVM}\label{sec:likelihood-estimation}

In contrast to surrogate approaches, {\em likelihood estimation} methods \citep{bakir2007} have been designed to address a wide range of structured prediction problems within a single, general framework. Given a training dataset $(x_i,y_i)_{i=1}^n$, these methods learn a {\em score} function $F_n:\Z\times\X\to\R$ that measures the likelihood of observing a given input-output pair $(z,x)$. In these settings, the structured prediction estimator $\fn:\Z\to\X$ is defined in terms of an optimization problem, namely as the function selecting the ``most likely'' output according to the score function $F_n$. This amounts to solving the maximization problem
\eqal{
       \fn(x) = \argmax_{z\in\Z}~ F_n(z,x),
}
for any input $x\in\X$ provided at test time.

Likelihood estimation methods are identified by the approach used to learn the score function $F_n$. A general strategy, adopted for instance by {\em Structured Support Vector Machines (SVMStruct)} \citep{tsochantaridis2005} is to consider a model of the form 
\eqal{\label{eq:likelihood-estimation-feature-parametrization}
       F_n(z,x) = \scal{~\Phi(z,x)~}{~\mbf w_n~}_{\gg}
}
for $x\in\X$ and $z\in\Z$, where $\gg$ is a suitable feature space and $\Phi:\Z\times\X\to\gg$ a joint feature map on the input-output set. The function $F_n$ is therefore parametrized by the vector $\mbf w_n \in\gg$, which is learned during the training phase. For instance, $\mbf w_n$ can be learned by minimizing an upper bound of the empirical risk by extending the strategy used in binary classification settings for the standard SVM approach \citep{tsochantaridis2005,cortes2016structured}.

Other approaches follow more adherently the interpretation of $F_n$ as measuring the likelihood of input-output pairs and thus consider models of the form 
\eqal{
       F_n(z,x) = p(z|x) = \frac{e^{-\scal{\Phi(z,x)}{\mbf w_n}} }{ \sum_{z'\in\Z} e^{-\scal{\Phi(z',x)}{\mbf w_n}}},
}

\noindent where $F_n(\cdot,x)$ is a  probability distribution over $\Z$. These methods consider a similar parametrization of the target function to SVMStruct approaches. However, they differ from the latter methods in that during training the aim to approximate the underlying input-output distribution. This model is often adopted by structured prediction approaches based on {\em Conditional Random Fields (CRF)} \citep{vishwanathan2006accelerated,morency2007latent}. For an in-depth introduction on likelihood estimation methods we refer the reader to \cite{nowozin2011structured} and references therein.

We care to point out that, in general, although the approaches above consider models that can be applied to arbitrary output spaces $\Z$, the likelihood estimation algorithms typically {\em require $\Z$ to be finite}. In this sense, a fundamental advantage of the \structshort{} framework considered in this work is to go beyond the finite setting.

\vspace{1em}\noindent{\bf Likelihood Estimation and \structshort{}.} The structured prediction framework discussed in this work has a natural interpretation as a likelihood estimation approach. To see this, consider 
\eqal{
       F_n(z,x) ~=~ - \sum_{i=1}^n~ \alpha_i(x) ~\loss(z,y_i)
}
for $x\in\X$ and $z\in\Z$. If $\loss$ admits an \structshort{}, our method corresponds to parametrizing $F_n$ as in \Cref{eq:likelihood-estimation-feature-parametrization} above, with $\gg = \hh\otimes\ff$ and $\Phi(z,x) = \zmap(z) \otimes \xmap(x)$, where $\ff$ is a RKHS on the input set $\X$ with associated kernel $\xmap:\X\times\X\to\R$ and the map $\xmap(x) = k(x,\cdot)\in\ff$ can be interpreted as a feature map from $\X$ to $\ff$. By leveraging the properties of the tensor product operation, we have 
\eqals{
       F_n(x,z) = \scal{~\Phi(x,z)~}{~\mbf w_n~}_\gg = \scal{~\zmap(z)~}{~W_n ~ \xmap(x)~}_\hh = \scal{~\zmap(z)~}{~\gn(x)~}_\hh,
}
where we have defined $\gn:\X\to\hh$ as the function such that $\gn(x) = W_n ~ \xmap(x)$ for any $x\in\X$ and $W_n:\ff\to\hh$ is the operator corresponding to $\mbf w_n$ given by the canonical isomorphism between $\hh\otimes\ff$ and the space $\HS(\ff,\hh)$ of Hilbert-Schmidt operators from $\ff$ to $\hh$. 

In \Cref{sec:framework} we have discussed a number of algorithms to learn $W_n$ (or, equivalently $\mbf w_n$), whose theoretical properties have then been studied in \Cref{sec:theory}. This connection opens two relevant questions for future investigation: $1)$ study approaches to learn the parameters $W_n$ that do not necessarily converge to the conditional mean (but for which it is still possible to prove the consistency of the resulting structured prediction); $2)$ While the \structshort{} assumption seem to require a ``separable'' representation of the form $\Phi(x,z) = \zmap(z) \otimes \xmap(x)$, it would be interesting to consider joint input-output feature maps, which could prove beneficial in settings where structural relations between input and outputs could be leveraged. A potential promising approach to address this question would be to borrow ideas from the literature on vector-valued and multi-task learning with RKHS for vector-valued functions \citep[see for instance][and references therein]{alvarez2012kernels}.

\subsection{Conditional Mean Embeddings}\label{sec:mean-embeddings}

In this section we highlight the relation between our structured prediction framework and {\em conditional mean embeddings} \citep{song2009hilbert}. This connection is particularly useful to understand the role played by the surrogate function $\gstar$ within the analysis of \Cref{sec:theory}.

Let $\hh$ be a RKHS of functions from $\Y$ to $\R$ with associated positive definite kernel $h:\Y\times\Y\to\R$. We recall that the conditional mean embedding in $\hh$ of $\rho(\cdot|x)$ given $x\in\X$ is defined as
\eqals{
       \mu_{y|x} = \int_\Y h(y,\cdot)~d\rho(y|x).
}
A key aspect of conditional mean embeddings is that, thanks to the reproducing property of the RKHS, for any $f\in\hh$ we have
\eqals{
       \scal{f}{\mu_{y|x}}_\hh = \EE_{y|x} ~ f(y).
}
This allows to evaluate the conditional expectation with respect to $\rho(\cdot|x)$ of any function $f\in\hh$ by directly performing an inner product with $\mu_{y|x}$.

It was observed in \cite{sriperumbudur2011universality} that for a wide family of RKHS, called {\em characteristic}, the kernel mean embedding operator is injective. In other words, two distributions have same embedding in $\hh$ if and only if they coincide. This implies that kernel mean embeddings, and in particular conditional mean embeddings, encode rich information about the associated distribution within a single vector in $\hh$.

\vspace{1em}\noindent{\bf Conditional Mean Embeddings and \structshort{}.} Let $\loss:\Z\times\Y\to\R$ admit an \structshort{}. Under the same notation of \Cref{def:self}, assume the corresponding surrogate space $\hh$ to be a RKHS and that $\ymap:\Y\to\hh$ is an associated feature map, namely such that $h(y,y') = \scal{\ymap(y)}{\ymap(y')}_\hh$ is the reproducing kernel associated to $\hh$. Then, according to the characterization of $\gstar$ in \cref{eq:fstar-in-terms-of-gstar-full}, for any $x\in\X$ we have that $\gstar(x) = \mu_{y|x}$ corresponds to the conditional mean embedding of $\rho(\cdot|x)$ \cite{song2009hilbert}. Moreover, by denoting $\zmap(z) = \loss(z,\cdot)$ and leveraging the reproducing property of the mean embedding, we have 
\eqals{
       \scal{~\zmap(z)~}{~\gstar(x)~}_\hh ~=~ \scal{~\loss(z,\cdot)~}{~\mu_{y|x}~}_\hh ~=~ \EE_{y|x}~\loss(z,y).
}
Interestingly, this observation recovers directly the Fisher consistency result of \Cref{lem:fstar-in-terms-of-gstar} {\em when $\hh$ is an RKHS}. Indeed, we have observed in \Cref{eq:fstar} that the solution $\fstar$ of the structured prediction expected risk corresponds to the minimizer of the conditional expectation $\EE_{y|x}~\loss(z,y)$. The equation above implies that this is equivalent to have $\fstar(\cdot) = \argmin_{z\in\Z}~\scal{\zmap(z)}{\gstar(\cdot)}$. 

In \Cref{sec:rates} we will see that in order to prove learning rates for the structured prediction estimator we will need to impose assumptions on $\gstar$. These could be interpreted as requiring the learning problem to satisfy regularity conditions. Indeed, the connection above between the \structshort{} definition and conditional mean embeddings shows that $\gstar$ is implicitly encoding key properties of the data generating distribution $\rho$ and, consequently, of the structured prediction problem itself. For more details on the topic, we refer the interested reader to \citep{muandet2017kernel} for an in-depth introduction on kernel mean embeddings and to \citep{song2009hilbert,lever2012conditional,singh2019kernel} for the special case of conditional mean embeddings.

\section{Theoretical Analysis}\label{sec:theory}

This section is devoted to characterize the statistical properties of the structured prediction estimators introduced in this work. In particular we will prove that under standard hypotheses from the statistical learning literature our approach is universally consistent and enjoys optimal learning rates. 

\subsection{Comparison Inequality}

The key result of our analysis, discussed in this section, is to show how the approximation of $\gstar$ via an estimator $\gn$ (such as those discussed in \Cref{sec:framework}) allows to characterize the behavior of the corresponding estimator $\fn = \decoding\circ\gn$ with respect to the ideal solution $\fstar$. The following result provides such characterization to any function $g:\X\to\hh$. 

\begin{restatable}[Comparison Inequality]{theorem}{PComparison}\label{prop:comparison-inequality}
Let $\Z$ be a compact set and $\loss:\Z\times\Y\to\R$ admit an \structshort{}. Let $\fstar$, $\gstar$ and the risk $\rr(\cdot)$ be defined as in \cref{lem:fstar-in-terms-of-gstar}. Let $g:\X\to\hh$ be measurable and let $f:\X\to\Z$ be such that
\eqal{\label{eq:decoding-ante-litteram}
       f(x) = \argmin_{z\in\Z}~\scal{~\zmap(z)~}{~g(x)~}_\hh,
}
for any $x\in\X$. Then,
\eqals{\label{eq:comparison-inequality}
	\E(f) - \E(\fstar) ~\leq~ 2~ \closs~ \sqrt{\rr(g)-\rr(\gstar)} 
}
\end{restatable}
The result in \Cref{prop:comparison-inequality} states that we can control the structured prediction excess risk in terms of the least-squares risk $\rr$ in approximating $\gstar$. The theorem holds for any function $g:\X\to\hh$ that is measurable, a technical requirement satisfied in particular by every regression estimator $\gn$ introduced in \Cref{sec:framework}. 

\Cref{prop:comparison-inequality} shifts the problem of studying the generalization properties of $\fn$ to that of characterizing the learning rates of the vector-valued estimator $\gn$, for which more well-established tools from statistical learning theory can be leveraged. Throughout this work we will refer to \Cref{eq:comparison-inequality} as the {\em comparison inequality} of our structured prediction framework. This notation is borrowed from the literature on surrogate methods, as discussed in more detail in \Cref{sec:surrogate-frameworks}. An result analogous to \cref{prop:comparison-inequality} was shown originally in \cite{ciliberto2016} for functions satisfying a similar property to \structshort{}. For completeness, in \cref{sec:app-framework} we prove it for \structshort{} functions.

\subsection{Universal Consistency}

The comparison inequality in \Cref{prop:comparison-inequality} is instrumental to study the generalization properties of the estimator considered in this work. In particular, the results reported in the rest of this section are obtained by characterizing the statistical properties of the estimator $\gn$ and then extending them to $\fn$ by means of the inequality in \Cref{eq:comparison-inequality}.

We start from the result proving the universal consistency of $\fn$. This is a fundamental requirement for a valid learning algorithm, stating that $\E(\fn)$ converges to the minimum possible risk $\E(\fstar)$ as the number $n$ of training points grows to infinity. A key assumption in this setting will be that the kernel $k:\X\times\X\to\R$ on the input space, introduced to learn the coefficients $\alpha_i$ characterizing the solution $\fn$ in \cref{eq:estimator}, is {\em universal}. This is a standard assumption in statistical learning theory \citep[see e.g.]{steinwart2008} and corresponds to requiring the RKHS $\ff$ associated to $k$ to be dense in the space of continuous function on $\X$. Typical examples of universal kernels on $\X \subseteq \R^d$ are the Gaussian $k(x,x') = e^{-\|x-x'\|^2/\sigma^2}$ or the Laplacian $k(x,x') = e^{-\|x-x'\|/\sigma}$ kernels.
\begin{restatable}[Universal Consistency]{theorem}{Tuniversal}\label{teo:universal_consistency}
Let $\Z$ be a compact set and $\loss:\Z\times\Y\to\R$ admit an \structshort{}. Let $k:\X\times\X\to\R$ be a bounded universal reproducing kernel. For any $n\in\N$ and any distribution $\rho$ on $\X\times\Y$ let $\fn:\X\to\Z$ be the estimator in \Cref{eq:estimator} trained on $(x_i,y_i)_{i=1}^n$ points independently sampled from $\rho$ and with weights $\alpha$ defined as:
\begin{enumerate}[(a)]
	\item{\em(Ridge Regression)} in \Cref{eq:alpha-dual} with $\lambda_n = n^{-1/2}$, or
	\item{\em(L2-Boosting)} in \Cref{eq:kboosting} with step-size $\nu<1/\kappa^2$ and $t_n = n^{1/2}$, or
	\item{\em(PCR)} in \Cref{eq:kpcr} with $\la_n = n^{-1/2}$, or
\end{enumerate}
 Then,
\eqal{\label{eq:universal_consistency}
\lim_{n \to +\infty} ~\E(\fn) ~=~ \E(\fstar)  \qquad \mbox{with probability} \quad 1
}
\end{restatable}
The proof of \Cref{teo:universal_consistency} is reported in \Cref{sec:app-theory}. The main technical step is to show that the estimator $\gn$ is universally consistent. Then universal consistency of $\fn$n follows by combining the latter result with the comparison inequality of \Cref{prop:comparison-inequality}. We point out that since $\gn$ is a vector-valued least-squares estimator, the corresponding result in the case where $\hh$ is a finite space has been extensively studied in previous work (see e.g. \citep{caponnetto2007}). However, to prove \cref{teo:universal_consistency} in the general setting, we extended the work in \citep{caponnetto2007} to the case where $\hh$ is infinite dimensional, which was considered an open question \citep{lever2012conditional}.

\subsection{Finite Sample Bounds}\label{sec:rates}

In order to prove finite sample bounds for structured prediction we need to impose regularity assumptions on the learning problem. This is a standard approach in learning theory (related to the {\em No-Free-Lunch Theorem} \citep{devroye2013probabilistic}). In particular, we will require the target function $\gstar$ to belong to $\hh\otimes\ff$. This is a standard assumption in learning theory in the context of ridge regression \citep{caponnetto2007,steinwart2008}. In \Cref{sec:mean-embeddings} we observed that $\gstar$ is strongly related to the concept of {\em conditional mean embedding} of the distribution $\rho(\cdot|x)$ \citep{song2009hilbert}. Therefore, by imposing it to belong to $\hh\otimes\ff$ or imposing additional regularity requirements, implicitly corresponds to controlling the regularity of the data generating distribution $\rho$. 

Below we report the learning rates of the algorithms considered in this work. 

\begin{restatable}[Learning Rates]{theorem}{TRates}\label{thm:rates} Let $\Z$ be a compact set and $\loss:\Z\times\Y\to\R$ admit an \structshort{} with associated Hilbert space $\hh$. Let $k:\X\times\X\to\R$ be a continuous reproducing kernel on $\X$ with associated RKHS $\ff$ such that $\kappa^2 := \sup_{x\in\X}k(x,x)<+\infty$. Let $\rho$ be a distribution on $\X\times\Y$ and let the corresponding $\gstar$ defined in \Cref{eq:fstar-in-terms-of-gstar-full} be such that $\gstar\in\hh\otimes\ff$. Let $\delta\in(0,1]$ and $n_0$ sufficiently large such that $n_0^{-1/2} \geq \frac{9\kappa^2}{n_0} \log \frac{n_0}{\delta}$. Then, for any $n\in\N$, the following estimators $\fn:\X\to\Z$ trained on $n$ points independently sampled from $\rho$ are such that, with probability at least $1-\delta$
\eqal{\label{eq:learning-rate}
 	\E(\fn) - \E(\fstar) ~\leq~ \closs ~\msf{m}~\msf{q}~ \log(4/\delta)~ n^{-1/4},
}
with
\eqals{
    \msf{m} = 16 \Big(\kappa(1 + \kappa\nor{\gstar}) + \kappa \sqrt{1 + \nor{\gstar}^2} + \nor{\gstar}\Big),
}
and $\msf{q}$ defined as follows. This holds for estimators $\fn$ of the form \cref{eq:estimator} with corresponding weights $\alpha$ defined as:
\begin{enumerate}[(a)]
	\item{\em(Ridge Regression)} in \Cref{eq:alpha-dual} with $\lambda_n = n^{-1/2}$. With constant $\msf{q} \leq 3$.
	\item{\em(L2-Boosting)} in \Cref{eq:kboosting} with $\nu<1/\kappa^2$ and $t_n = n^{1/2}$. With constant $\msf{q} \leq 2 + 2\gamma + e^{\gamma - 1}/\gamma$.
	\item{\em(Principal Component Regression)} in \Cref{eq:kpcr} with $\la_n =  n^{-1/2}$. With constant $\msf{q} \leq 5$.
\end{enumerate}
\end{restatable}
%
%
%
%
%
\noindent \Cref{thm:rates} is obtained as a specialization of \Cref{thm:rates-refined} below. This result represents a direct extension of the learning rates known for binary classification (see e.g. \citep{yao2007early}) to all structured prediction problems with \structshort{} $\loss$. This shows that up to constants, structured prediction problems are in general as challenging as classification, from the statistical perspective. See the next result for more details.

\begin{remark}[Adaptive \structshort{} constants]\label{rem:infimum-closs}
We comment on the constants $\closs$ and $\msf{m}$ in the bound above (the constant $\msf{q}$ depends exclusively on the chosen algorithm). Note that the \structshort{} characterization of a function $\loss$ is not unique in terms of the space $\hh$ and feature maps $\ymap,\zmap$. Moreover, as observed in \cref{sec:framework}, computing the estimator $\fn$ does not require explicit knowledge of such objects and therefore \cref{thm:rates} holds for any $(\hh,\ymap,\zmap)$ such that $\loss$ admits an \structshort{} and $\gstar\in\hh\otimes\ff$. As a consequence, the bound in \cref{eq:learning-rate} implicitly applies for the infimum value of $\closs\msf{m}$ over the set of such triplets. 

Explicitly estimating this constant is in general an open problem. When $\Z$ and $\Y$ have finite cardinality, \citep{nowak2018sharp} that for a large family of widely used loss functions, such constant is at most polylogarithmic in the cardinality of the sets. 
\end{remark}

\noindent The result in \cref{thm:rates} provides the suitable hyperparameters for different \structshort{} estimators to achieve same statistical performance. Interestingly, depending on the method, this leads to different computational costs, as reported in \cref{table:cost-for-rates}.

\begin{table}[t]
    \centering
    \resizebox{\linewidth}{!}{
    \begin{tabular}{r c c c c}
    \toprule
         {\bf Algorithm} & \textbf{Train time} & \textbf{Train memory} & \textbf{Eval. time} & \textbf{Eval. memory}\\
         \midrule
         \structshort{} + RR & $O(n^3 + n^2 c_X)$ & $O(n^2)$ & $O(n c_X)$ & $O(n)$\\
         \structshort{} + L2B  & $O(n^{2}\sqrt{n} + n^2 c_X)$ & $O(n^2)$ & $O(n c_X)$ & $O(n)$\\
         \structshort{} + PCR  & $O(n^{2}\sqrt{n} + n^2 c_X)$ & $O(n^2)$ & $O(n c_X)$ & $O(n)$\\
    \bottomrule
    \end{tabular}
    }
    \caption{Computational complexity for training and evaluation of $\fn$ in \cref{eq:estimator} with weights trained respectively according to \cref{eq:alpha-dual} (RR), \cref{eq:kboosting} (L2B), \cref{eq:kpcr} (PCR). Hyperparameters chosen according to \cref{thm:rates} to achieve the corresponding learning rate. The term $c_X$ denotes cost of evaluating the kernel function.}
    \label{table:cost-for-rates}
\end{table}

\subsection{Refined Sample Bounds}\label{sec:refined-rates}
Now we refine the analysis above considering additional regularity conditions for the learning problem. In particular we will introduce two standard assumptions in the context of non-parametric regression / conditional mean estimation \citep{caponnetto2007}.
Let $\ff$ be the reproducing kernel Hilbert space associated to the kernel $k$ on the input space $\X$, and $C: \F \to \F$ be the linear operator defined as
\eqals{
\scal{f}{Cg}_\ff = \int f(x) g(x) d\rhox(x), \qquad \forall~ f,g \in \ff.
}
Now we can introduce the first condition
\begin{restatable}[Source condition]{assumption}{ASource}\label{asm:source}
There exists $r \geq 0$ and $h \in \hh \otimes \ff$ for which
\eqals{
g^* = (C^r  \otimes I) ~h.
}
The norm of $\|h\|_{\hh \otimes \ff}$ will be denoted by $R := \|h\|_{\hh \otimes \ff}$.
\end{restatable}
The condition above measures the regularity of $g^*$ in terms of the eigenspectrum of $C$. Note that the assumption is always verified for $r = 0$ (in that case $h = g^*$ and $R = \|g^*\|_{\hh\otimes\ff}$). Moreover, since $\F$ is separable and $C$ is trace class \citep{caponnetto2007}, then $C$ can be characterized in terms of a non-increasing sequence $(\sigma_j)_{j \in \N}$ of eigenvalues with associated eigenvectors $(u_j)_{j\in\N}$. For simplicity, let $\hh = \R$. We have $g^* = \sum_{j} \beta_j u_j$, with $\beta_j = \scal{g^*}{u_j}_\ff$. Then, \cref{asm:source} is equivalent to require that $\sum_j \beta_j^2/\sigma_j^{2r} \leq R$. Hence, the source condition corresponds to require $g^*$ to have rapidly decaying coefficients, when expressed in terms of the basis of $C$. More generally, when $\hh$ is a separable Hilbert space, we have $g^* = \sum_{j} \beta_j \otimes u_j$, with $\beta_j \in \hh$ defined as $\beta_j = (u_j \otimes I) \gstar$. Then, \cref{asm:source} is equivalent to require that $\sum_j \nor{\beta_j}^2/\sigma_j^{2r} \leq R$.

The second assumption is expressed with respect to the so called {\em effective dimension}, defined as
\eqal{\label{eq:eff-dim}
\deff(\la) = \tr(C (C + \la I)^{-1}), \qquad \forall \la > 0,
}
and characterizes the interaction between the measure $\rho$ and the kernel $k$ on $\X$.
\begin{restatable}[Capacity condition]{assumption}{ACapacity}\label{asm:capacity}
There exists $\gamma\in[0,1]$ and $Q > 0$ for which
\eqals{
\deff(\la) \leq Q \la^{-\gamma}, \qquad \forall ~ \la > 0.
}
\end{restatable}
The condition above is always satisfied with $\gamma = 1$ when the kernel is bounded. Indeed let $\kappa^2 := \sup_{x} k(x,x)$, then $\deff(\la) \leq \kappa^2/\la$. Moreover when the eigenvalues of $C$ decay as $\sigma_j(C) \leq A j^{-\beta}$, for $A > 0$, $\beta > 1$ and $j\in\N$, then the assumption above is satisfied with $Q = A$ and $\gamma = 1/\beta$ \citep{caponnetto2007,rudi2015less}. In particular note that: (i) since $C$ is trace class, the sequence of eigenvalues is summable therefore $\beta > 1$; (ii) the eigenvalue decay is characterized by the choice of the kernel and the marginal probability distribution $\rhox$. For example, when $\X = [-B, B]^d$, $d\in\N$ for $B > 0$, $k:\X\times\X\to\R$ is a Sobolev-kernel of smoothness $s > d/2$ and $\rhox$ is a density bounded from above and away from zero (i.e. there exists $A \geq a > 0$ such that $a \leq \rhox(x) \leq A$ for $x \in X$), then there exists $Q$ depending on $B, s, d$ for which $\sigma_j(C) \leq Q j^{-2s/d}$ and so $\deff(\la) \leq Q \la^{-d/(2s)}$ \citep[see][]{wendland2004scattered}. We can now state the refined version of \Cref{thm:rates}.

\begin{restatable}[Refined Learning Rates]{theorem}{TRatesRefined}\label{thm:rates-refined} 
Under the same notation and assumptions of \Cref{thm:rates} and under the additional \cref{asm:source,asm:capacity}, let $\delta\in(0,1]$ and $n_0$ sufficiently large such that $n_0^{-1/(1+2r+\gamma)} \geq \frac{9\kappa^2}{n_0} \log \frac{n_0}{\delta}$. For any $n\geq n_0$, the following estimators $\fn:\X\to\Z$ trained on $n$ points independently sampled from $\rho$ are such that, with probability at least $1-\delta$
\eqal{\label{eq:refined-learning-rate}
 	\E(\fn) - \E(\fstar) ~\leq~ \closs ~\msf{m}~\msf{q}~ \log(4/\delta)~ n^{-\frac{r + 1/2}{2r+\gamma+1}},
}
with
\eqals{
    \msf{m} = 16 \Big(\kappa(1 + \kappa R) + \kappa \sqrt{Q + R^2} + R\Big),
}
and $\msf{q}$ defined as follows. This holds for estimators $\fn$ of the form \cref{eq:estimator} with corresponding weights $\alpha$ defined as:
\begin{enumerate}[(a)]
	\item{\em(Ridge Regression)} in \Cref{eq:alpha-dual} with $\lambda_n = n^{-\frac{1}{ 2r+\gamma+1}}$. With $\msf{q} \leq 3$.
	\item{\em(L2-Boosting)} in \Cref{eq:kboosting} with $\nu<1/\kappa^2$ and $t_n = n^{\frac{1}{2r+\gamma+1}}$. With $\msf{q} \leq 2 + 2\nu + e^{\nu - 1}/\nu$.
	\item{\em(Principal Component Regression)} in \Cref{eq:kpcr} with $\lambda_n = n^{-\frac{1}{ 2r+\gamma+1}}$. With $\msf{q} \leq 5$.
\end{enumerate}

\end{restatable}
The theorem above shows that the proposed estimator for structured prediction in \cref{eq:estimator} has learning rates that are adaptive to the source and capacity condition, when the coefficients are computed according to the algorithms considered in the theorem. Similarly to \cref{eq:universal_consistency,thm:rates}, the result is obtained by studying the generalization properties of $\gn$ and combining such analysis with the comparison inequality. As already pointed out in the commentary of \Cref{eq:universal_consistency}, our results in this section generalize those of \citep{caponnetto2007} to the case of infinite dimensional output spaces $\hh$. Interestingly, the result in \cref{thm:rates-refined} refine \cref{thm:rates}. In particular, in the worst case $r = 0, \gamma = 1$ we recover the learning rate $O(n^{-1/4})$ in \cref{thm:rates}, while for stronger regularity assumptions (namely $r>>1$ or $\gamma\approx0$) the proposed algorithms attain a significantly faster rate close to $O(n^{-1/2})$. 

When $\Y$ and $\Z$ have finite cardinality and the data distribution satisfies additional regularity hypotheses, such as the Tsybakov condition \citep[see][]{tsybakov2004,yao2007early}, it is possible to achieve rates of up to $O(n^{-1})$, as shown in \citep{nowak2018sharp}. A relevant question is whether an analogous notion of the Tsybakov condition could be identified in the case where $\Y$ and $\Z$ are not finite. Note that the $O(n^{-1})$ rate is optimal in the case of binary classification \citep{bartlett2006,tsybakov2004}. This implies that such rate is optimal also for the larger family of structured prediction problems satisfying the \structshort{} assumption. In this sense, a natural question is whether it may be possible to perform a more refined analysis by studying specific structured prediction problems individually. 

To conclude, note that the algorithms considered in this work are not only adaptive from the statistical viewpoint but also from a computational perspective. In particular, \Cref{table:cost-for-rates-refined} reports the computational costs of running the algorithms described in this work for the choice of hyperparameters reported by \Cref{thm:rates-refined} depending on the \cref{asm:source,asm:capacity}.

\begin{table}[t]
    \centering
    \resizebox{\linewidth}{!}{
    \begin{tabular}{r c c c c}
    \toprule
         {\bf Algorithm} & \textbf{Train time} & \textbf{Train memory} & \textbf{Eval. time} & \textbf{Eval. memory}\\
         \midrule
         \structshort{} + RR & $O(n^3 + n^2 c_X)$ & $O(n^2)$ & $O(n c_X)$ & $O(n)$\\
         \structshort{} + L2B  & $O(n^{2 + \frac{1}{2r+\gamma+1}} + n^2 c_X)$ & $O(n^2)$ & $O(n c_X)$ & $O(n)$\\
         \structshort{} + PCR  & $O(n^{2 + \frac{\gamma}{2r+\gamma+1}} + n^2 c_X)$ & $O(n^2)$ & $O(n c_X)$ & $O(n)$\\
    \bottomrule
    \end{tabular}
    }
    \caption{Computational complexity for training and evaluation of $\fn$ in \cref{eq:estimator} with weights trained respectively according to \cref{eq:alpha-dual} (RR), \cref{eq:kboosting} (L2B), \cref{eq:kpcr} (PCR). Hyperparameters chosen according to \cref{thm:rates-refined} to achieve the corresponding learning rate. The term $c_X$ denotes cost of evaluating the kernel function.}
    \label{table:cost-for-rates-refined}
\end{table}

\section{Sufficient Conditions for \structshort{}}\label{sec:self-sufficient-conditions}

In this section we focus our attention to the definition of {\em \struct{}} (\structshort{}) introduced in \Cref{def:self}. In particular, we provide a number of sufficient conditions that guarantee a loss function to admit an \structshort{}, which are more interpretable and easy to verify than the original definition. We will show that most loss functions used in machine learning and structured prediction settings indeed satisfy the \structshort{} property and therefore that the learning framework proposed in this work applies to a wide family of relevant problems.

\vspace{1em}\noindent{\bf Bounding $\closs$.} As a byproduct of our analysis, the results in the following provide also upper bounds for the constant $\closs$ for a number of loss functions. As observed in \cref{thm:rates,thm:rates-refined}, such constant plays a role in characterizing the learning rates of the \structshort{} estimators. Following the discussion of \cref{rem:infimum-closs}, it is important to note that the estimates for $\closs$ reported in this section have been derived for a single parametrization of the \structshort{} definition for $\loss$ (namely the space $\hh$ and the feature maps $\zmap$ and $\ymap$). Obtaining sharp bounds for such constants is outside the scope of this work. We refer to \citep{osokin2017structured,nowak2018sharp} for refined analysis in the case where $\Z$ and $\Y$ are finite.

\subsection{\structshort{} on finite Output or Label Spaces}
In \Cref{sec:finite-Y} we provided a preliminary analysis of structured prediction for the case where label and output spaces coincide and are finite, namely $\Y=\Z=\{1,\dots,T\}$. This discussion was key in that it motivated the definition of \structshort{}. Indeed, as already mentioned, the \structshort{} definition is satisfied by any loss function acting on finite output and label spaces. The following proposition shows that it is sufficient that only one of the two spaces $\Y$ or $\Z$ is finite to guarantee the loss function to admit an \structshort{}. 

\begin{restatable}[\structshort{} \& finite $\Y$ or $\Z$]{theorem}{TFiniteYorZ}\label{thm:finite-Y-or-Z}
The function $\loss:\Z\times\Y\to\R$ admits an \structshort{} if one of the following conditions hold:
\begin{enumerate}[(a)]	
	\item $\Z$ and $\Y$ are finite sets. In this case $\closs \leq \|\loss\|$ the operator norm \\of the matrix $\loss\in\R^{|\Z|\times|\Y|}$ with entires $\loss_{z,y}=\loss(z,y)$.
	\item $\Z$ is finite, $\Y$ is compact and $\loss(z,\cdot)$ is continuous on $\Y$ for any $z\in\Z$.\\ In this case $\closs \leq \sup_{y\in\Y} \sqrt{\sum_{z\in\Z} |\loss(z,y)|^2}$.
	\item $\Z$ is compact, $\Y$ is finite and $\loss(\cdot,y)$ is continuous on $\Z$ for any $y\in\Y$.\\ In this case $\closs \leq \sup_{z\in\Z} \sqrt{\sum_{y\in\Y} |\loss(z,y)|^2}$.
\end{enumerate}
\end{restatable}
The result above shows that most loss functions used in typical structured prediction applications admit an \structshort{}. Indeed, previous literature on the topic has been focused on problems where either the output or the label space (or both) are finite, albeit possibly very large \citep{bakir2007,nowozin2011structured}. In this setting, relevant examples of applications range from computer vision, such as segmentation \citep{alahari2008reduce}, localization \citep{blaschko2008learning,lampert2009efficient}, labeling \citep{karpathy2015deep}, pixel-wise classification \citep{szummer2008learning}), speech recognition \citep{bahl1986maximum,sutton2012introduction}, natural language processing \citep{tsochantaridis2005}, trajectory planing \citep{ratliff2006maximum} or hierarchical classification \citep{tuia2011structured}. 

The major implication of \Cref{thm:finite-Y-or-Z} is that it justifies the application of the estimator proposed and studied in this paper to address a variety of structured prediction problems previously considered in the literature. Indeed, our analysis in \Cref{sec:theory} automatically guarantees that the corresponding estimator has strong theoretical guarantees when applied to these settings. 

In the rest of this section we focus on the case where $\Y$ and $\Z$ are not necessarily finite, showing that the definition of \structshort{} encompasses a significantly wider family of settings compared to the classic structured prediction literature.

\subsection{\structshort{} and Reproducing Kernel Hilbert Spaces} 

We already highlighted the relation between the definition of \structshort{} and the notion of positive definite kernel. The following result provides a more refined characterization of this relation, showing in particular how it is possible to leverage kernels to ``build'' \structshort{} functions. 

\begin{restatable}[\structshort{} \& RKHS]{theorem}{TRkhsAndStruct}\label{thm:rkhs-and-self}
Let $\Z = \Y$ be a compact set and $h:\Y\times\Y\to\R$ a continuous bounded reproducing kernel on $\Y$ with associated RKHS $\hh$. Let $\eta^2 = \sup_{y\in\Y} h(y,y)$. Then, $\loss:\Y\times\Y\to\R$ admits an \structshort{} if one of the following holds:
\begin{enumerate}[(a)]
\item {\em (Kernels).} $\loss(z,y) = h(z,y)$ for any $y,z\in\Y$. In this case $\closs\leq\eta^2$.

\item {\em (Kernel Dependency Estimation (KDE)).} $\loss(z,y) = h(z,z) + h(y,y) - 2 h(z,y)$ for any $y,z\in\Y$. In this case $\closs = 2(2\eta^4 + 1)$.

\item For every $y\in\Y$ the functions $\loss(\cdot,y)\in\hh$ belong to a bounded set of $\hh$, namely $\sup_{y\in\Y}\|\loss(\cdot,y)\|_\hh = \msf{D}<+\infty$. In this case $\closs \leq \eta\msf{D}$. The same holds if the family of functions $\loss(z,\cdot)$ parametrized by $z\in\Z$ belong to a bounded set of $\hh$.

\item $\loss$ belongs to $\hh\otimes\hh$ the RKHS with associated kernel $\bar h:\Y^2\times\Y^2\to\R$ such that $\bar h((z,y),(z',y')) = h(z,z')h(y,y')$ for any $z,z',y,y'\in\Y$. In this case $\closs \leq \eta^2\|\loss\|_{\hh\otimes\hh}$
\end{enumerate}
\end{restatable}
\Cref{thm:rkhs-and-self} provides four interesting results. First, as already mentioned, the definition of \structshort{} function recovers and is more general than that of positive definite kernel. Second, we see that our framework recovers the {\em Kernel Dependency Estimation (KDE)} approach \citep{weston2002,cortes2005}, which corresponds to a structured prediction setting with loss function $\loss(z,y) = \|h(z,\cdot) - h(y,\cdot)\|_\hh^2 = h(z,z) + h(y,y) - 2 h(z,y)$.

Point $(c)$ shows that $\loss$ admits an \structshort{} if the family of functions $\{\loss(\cdot, y)\}_{y\in\Y}$ parametrized by $y\in\Y$, is uniformly contained in a ball in $\hh$. Finally, point $(d)$ of \Cref{thm:rkhs-and-self} reports a more general result, showing that {\em all functions that belong to the RKHS obtained as the tensor product of $\hh$ with itself admit an \structshort{}}. This recovers a large family of loss functions as discussed in the example below.

\begin{example}[Smooth Functions on $\Y\times\Y$ with $\Y = {[}-B,B{]}^d$ admit an \structshort{}]\label{ex:smoothness}
Let $\loss\in C^{\infty}(\Y\times\Y)$, where $C^{\infty}(\Y)$ denotes the space of smooth functions over $\Y$. Let $\hh = W^{d,2}(\Y)$ be the Sobolev space of functions over $\Y$ with up to order $d$ square integrable weak derivatives \citep{adams2003sobolev}. We have $C^{\infty}(\Y\times\Y) = C^{\infty}(\Y)\otimes C^{\infty}(\Y) \subset \hh\otimes\hh$. Then, \cref{thm:rkhs-and-self} (d) guaratntees that $\loss$ admits an \structshort{}. For more details see \citep{luise2018differential}.
\end{example}

\subsection{\bf \structshort{} and Regularity}

The connection between ILE and RKHSs suggest the definition of \structshort{} to be somewhat related to the concept of smoothness or regularity of a function. The following result goes beyond RKHSs and investigates this question in further detail.
\begin{restatable}[\structshort{} \& Regularity]{theorem}{PSmoothness}\label{prop:self-and-smoothness}
Let $\Z = \Y = [-B, B]^d$, $B > 0$. A function $\loss: \Y \times \Y \to \R$ admits an \structshort{} when at least one of the following conditions hold:
\begin{enumerate}[(a)]
\item $d=1$ and $\loss$ is $\alpha$-H\"older continuous with $\alpha > 1/2$ or it is of bounded variation and $\alpha$-H\"older continuous with $\alpha > 0$. 
\item $\loss(z,y) = v(z-y)$, where $v$ is a function such that $\closs = \int_{-\infty}^{+\infty} |\widehat{v}(\omega)| d \omega < \infty$ and $\widehat{v}$ is the Fourier transform of $v$.
\item The mixed partial derivative $\loss_{y_1,\dots,y_d}:\Y\to\R$ of $\loss$ exists almost everywhere and $\loss_{y_1,\dots,y_d}\in L^p(\Y)$ with $p>1$.
\end{enumerate}
\end{restatable}
\cref{prop:self-and-smoothness} shows that any function that is sufficiently regular admits an \structshort{}. This allows to recover most loss functions used in machine learning and robust estimation as special cases.

\begin{example}[Robust Estimation]
We have already observed that smooth functions such as the least-squares and logistic loss admit an \structshort{} according to the discussion in \Cref{ex:smoothness}. Here we observe that also the hinge loss, used in binary classification, and most loss functions used for scalar regression on $\Y = [0,1]$ admit an \structshort{}, albeit being not smooth. Indeed, most common loss functions used in these contexts are Lipschitz continuous and therefore satisfy \Cref{prop:self-and-smoothness} $(a)$. Notable examples are loss functions used for {\em robust estimation} such as the absolute value, Huber, Cauchy, German-McLure, ``Fair'' and $L_2-L_1$ \citep{huber2011}. All these functions are differentiable almost everywhere, with uniformly bounded derivatives and thus satisfy \Cref{prop:self-and-smoothness} $(c)$.
\end{example}

\subsection{Composition Rules for \structshort{}} 

A natural question is whether some operations over \structshort{} functions preserve the characterization introduced in \Cref{def:self}. Below we provide a set of rules that allow to ``build'' new \structshort{} functions from known ones. 

\begin{restatable}{theorem}{TCompositions}\label{thm:compositions}
Let $\Z$ and $\Y$ be compact sets. Then $\loss:\Z\times\Y\to\R$ admits an \structshort{} if one of the following holds:
\begin{enumerate}[(a)]
\item{\em(Restriction)} There exist two sets $\bar\Z\supseteq\Z$, $\bar\Y\supseteq\Y$ and $\bar\loss:\bar\Z\times\bar\Y\to\R$ such that $\bar\loss$ admits an \structshort{} and its restriction to $\Z\times\Y$ corresponds to $\loss$, namely
\eqals{
\loss = \bar\loss|_{\Z\times\Y}.
}
In this case $\closs \leq \msf{c}_{\bar\loss}$.

\item{\em(Right Composition)} There exits $\bar\Z,\bar\Y$ and a \structshort{} $\bar\loss:\bar\Z\times\bar\Y\to\R$, such that  
\eqal{\label{eq:composition-self}
	\loss(z,y) = \alpha(z)\bar\loss(A(z),B(y))\beta(y),
}
with $A:\Z\to\bar\Z$, $B:\Y\to\bar\Y$, $\alpha:\Z\to\R$ and $\beta: \Y \to \R$ continuous function, with $\sup_{z\in\Z} |\alpha(z)|\leq\bar\alpha$ and $\sup_{y\in\Y}|\beta(y)|\leq\bar\beta$ with $\bar\alpha,\bar\beta\in\R$. Then $\closs\leq\bar\alpha\bar\beta \msf{c}_{\bar\loss}$.

\item{\em(Left Composition)} There exist $P\in\N$, spaces $(\Z_p)_{p=1}^P$, $(\Y_p)_{p=1}^P$ and corresponding \structshort{} $\loss_p:\Z_p\times\Y_p\to\R$ such that $\Z= \Z_1\times\cdots\times\Z_P$,~ $\Y = \Y_1\times\cdots\times\Y_P$ and 
\eqals{
	\loss(z,y) = \Gamma \big(\loss_1(z_1,y_1),\dots,\loss_P(z_P,y_P)\big),
}
for any $z=(z_1,\dots,z_P)\in\Z$ and $y = (y_1,\dots,y_P)\in\Y$, where $\Gamma:\R^P \to \R$ is an analytic function (e.g. a polynomial).
\end{enumerate}
\end{restatable}
The result above provides us several tools to build new \structshort{} functions. In particular, \Cref{thm:compositions} $(a)$ shows that we can always restrict a \structshort{} function on a smaller pair of output-label sets and still enjoy the same properties of the original loss, also in terms of universal consistency and rates of the resulting structured prediction estimator. \Cref{thm:compositions} $(b)$ allows to extend a \structshort{} function $\bar\loss$ to other output-label pairs by means of the embeddings $A$ and $B$ and the weighting functions $\alpha$ and $\beta$ in \Cref{eq:composition-self}.

\begin{example}[Restriction of Smooth functions on compact sets admit an \structshort{}]
Let $\loss$ be a smooth function over $[-B,B]^d$ with $B>0$. Then, by \cref{thm:compositions} (a), $\loss$ admits an \structshort{} on every compact set $\Y\subseteq[-B,B]^d$.
\end{example}
Finally, \Cref{thm:compositions} $(c)$ shows that any combination (namely sum and multiplications) of \structshort{} functions is still \structshort{}. To highlight the importance of this result we clarify it in the following.
\begin{corollary}\label{cor:sum-and-products-of-self}
Let $\loss_1:\Z_1\times\Y_1\to\R$ and $\loss_2:\Z_2\times\Y_2\to\R$ admit an \structshort{}. Then $\loss:(\Z_1\times\Z_2)\times(\Y_1\times\Y_2)\to\R$ if, for any $z_i\in\Z_i, y_i\in\Y_i$ and $i=1,2$, one of the following conditions hold:
\begin{enumerate}[(a)]
\item $\loss((z_1,z_2),(y_1,y_2)) = \loss_1(z_1,y_2) + \loss_2(z_2,y_2)$,
\item $\loss((z_1,z_2),(y_1,y_2)) = \loss_1(z_1,y_2)\loss_2(z_2,y_2)$. 
\end{enumerate}
\end{corollary}
The result above allows to consider general combinations of loss functions within the framework considered in this work. In particular, the following remark shows how multitask learning problems (possibly with structure on the output) can be recovered in this setting.

\begin{example}[Multitask Learning]
In multitask learning (MTL) settings the goal is to solve multiple separate supervised problems simultaneously \citep{evgeniou2004regularized,alvarez2012kernels}. The loss functions used in MTL typically consist in the sum of ``single task'' loss functions over the separate tasks, such as least-squares for regression or logistic/hinge for classification. Since according to \Cref{cor:sum-and-products-of-self} the sum of \structshort{} functions is still \structshort{}, we see that multitask learning is naturally recovered by the framework considered in this work. This fact was observed in \citep{ciliberto2017consistent}, where the structured prediction perspective on the MTL problem allowed to address the question of how to impose non-linear relations among multiple tasks by introducing the constraint output set $\Z \subset \Y = \R^T$.
\end{example}

\section{Conclusions}\label{sec:conclusions}

In this work we have presented a general framework for structured prediction. Our work revolves around the key notion of \struct{} (\structshort{}), which allows us to study structured prediction applications where the output space is not finite, differently from most previous work on the topic. This work significantly expanded upon \cite{ciliberto2016}, providing novel insights on the \structshort{} property as well as new algorithms for structured prediction and their corresponding theoretical analysis. Among the main contributions of this work: (a) we showed that the proposed framework can be applied to a wide range of structured prediction problems, providing a systematic approach to derive estimators with strong theoretical guarantees. In particular, we showed that it is possible to leverage existing algorithms from the vector-valued regression literature to obtain novel structured prediction estimators that enjoy equivalent statistical properties of the original method, but with reduced computational requirements. (b) We performed a refined analysis of the excess risk bounds, showing that the statistical rates and computational cost of the considered algorithms are adaptive to standard regularity properties of the learning problem. (c) We provided a number of sufficient conditions to verify whether a given loss admits an \structshort{}. These conditions are significantly easier to verify in practice in comparison to the general definition. Leveraging these conditions we proved that most loss functions used in machine learning indeed admit an \structshort{} and are therefore suited to our framework.

Relevant directions for future work will involve: (a) considering alternative estimators within the \structshort{} framework not necessarily minimizing the square loss in the surrogate space; (b) learning the structure of the output space when it is not fully known a-priori (for instance in manifold regression settings where the output manifold is only accessible via examples). This could be addressed by parametrizing a family of candidate output spaces and finding the optimal parameters while simultaneously fitting the structured prediction model. Finally, (c) an interesting question is to leverage further additional knowledge on the problem structure to improve the overall learning rates of the estimator. This direction has been recently preliminarily investigated in \cite{cortes2016structured,ciliberto2019localized}, where an explicit factorization of the loss function was used to design problem-specific algorithms and perform a refined analysis of their generalization properties.

{
\bibliography{biblio}
}

\newpage


\crefname{assumption}{Assumption}{Assumptions}
\crefname{equation}{}{}
\Crefname{equation}{Eq.}{Eqs.}
\crefname{figure}{Figure}{Figures}
\crefname{table}{Table}{Tables}
\crefname{section}{Section}{Sections}
\crefname{theorem}{Theorem}{Theorems}
\crefname{proposition}{Proposition}{Propositions}
\crefname{lemma}{Lemma}{Lemmas}
\crefname{corollary}{Corollary}{Corollaries}
\crefname{example}{Example}{Examples}
\crefname{remark}{Remark}{Remarks}
\crefname{algorithm}{Algorithm}{Algorithms}
\crefname{enumi}{}{}

\crefname{appendix}{Appendix}{Appendices}

\numberwithin{equation}{section}
\numberwithin{lemma}{section}
\numberwithin{proposition}{section}
\numberwithin{theorem}{section}
\numberwithin{corollary}{section}
\numberwithin{definition}{section}
\numberwithin{algorithm}{section}
\numberwithin{remark}{section}

\appendix
\crefalias{section}{appendix}

\section*{\huge Appendix}

The appendix are organized in three main parts:
\begin{itemize}
    \item \Cref{sec:app-framework} focuses on the general \structshort{} framework, proving the results in \cref{sec:framework} and the Comparison Inequality of \cref{prop:comparison-inequality}. 
    \item \Cref{sec:app-theory} covers the details of the theoretical analysis reported in \cref{sec:theory}.
    \item \Cref{sec:app-self-sufficient-conditions} provides the proofs of the results in \Cref{sec:self-sufficient-conditions}, offering sufficient conditions to guarantee a loss function to admit an \structshort{}. 
\end{itemize}

\paragraph{Contributions and connection with previous work} We recall that this paper is the longer version of \citep{ciliberto2016}. Therefore, the results reported in \Cref{sec:app-framework} contain significant overlaps with the original work. We still prove each of the results in detail for the sake of completeness and since in the current work we have extended the framework in \citep{ciliberto2016} to the case where $\loss:\Z\times\Y\to\hh$ with output space $\Z$ not necessarily corresponding to the label space $\Y$. The results in \Cref{sec:app-theory}
 and in particular \Cref{sec:app-self-sufficient-conditions} are novel for the most part.
 
\paragraph{Setting and Notation} We assume input, label and output spaces $\X$, $\Y$ and $\Z$ to be Polish spaces, namely separable complete metrizable spaces, equipped with the associated Borel sigma-algebra. When referring to the data distribution $\rho$ on $\X \times \Y$ we will always assume it to be a Borel probability measure, with $\rhox$ the marginal distribution on $\X$ and $\rho(\cdot|x)$ the conditional measure on $\Y$ given $x \in \X$. We recall that $\rho(y|x)$ is a regular conditional distribution \citep{dudley2002real}. Its domain $\domrho$ is a measurable set contained in the support of $\rhox$ and corresponds to the support of $\rhox$ up to a set of measure zero.

For a Hilbert space $\hh$ we denote with $\scal{\cdot}{\cdot}_\hh$ and $\|\cdot\|_\hh$ the associated inner product and corresponding norm. Given two Hilbert spaces $\hh_1$ and $\hh_2$ we denote by $\hh_1\oplus\hh_2$ and $\hh_1\otimes\hh_2$ respectively their direct sum and tensor prodcut. In particular, for any $h_1,h_1'\in\hh_1$ and $h_2,h_2'\in\hh_2$, we have 
\eqals{
    \scal{h_1\oplus h_2}{h_1'\oplus h_2'}_{\hh_1\oplus\hh_2} & = \scal{h_1}{h_1'}_{\hh_1} + \scal{h_2}{h_2'}_{\hh_2}\\
    \scal{h_1\otimes h_2}{h_1'\otimes h_2'}_{\hh_1\otimes\hh_2} & = \scal{h_1}{h_1'}_{\hh_1} \cdot \scal{h_2}{h_2'}_{\hh_2}.
}
Given a linear operator $V:\hh_1\to\hh_2$, we denote by $\tr(V)$ the trace of $V$ and by $V^*:\hh_2\to\hh_1$ the adjoint of $V$, namely such that $\scal{Vh_1}{h_2}_{\hh_2} = \scal{h_1}{V^*h_2}_{\hh_1}$ for every $h_1\in\hh_1$, $h_2\in\hh_2$. Moreover, we denote by $\|V\| = \sup_{\|h\|_{\hh_1} \leq 1} \|Vh\|_{\hh_2}$ the operator norm and $\|V\|_\HS= \sqrt{\tr(V^*V)}$ the Hilbert-Schmidt norm of $V$. In particular, we recall that the tensor product $\hh_1\otimes\hh_2$ is isometric to the space of Hilbert-Schmidt operators. 

We denote with $L^2(\X,\rhox,\H)$ the Lebesgue space of square integrable functions on $\X$ with respect to a measure $\rhox$ and with values in a separable Hilbert space $\H$. For simplicity we denote with $\LX$ the space $\LXR$. We denote with $\scal{f}{g}_{\rhox}$ the inner product $\int \scal{f(x)}{g(x)}_\H d\rhox(x)$, for all $f, g \in L^2(\X,\rhox,\H)$.

\paragraph{On the Argmin} In the main paper we denoted the minimizer of \cref{eq:estimator} as
\eqals{
  \fn(x) = \argmin_{y\in\Y} \sum_{i=1}^n \alpha_i(x) \loss(y,y_i).
}
Clearly, the rigorous notation should be
\eqals{
  \fn(x) \in \argmin_{y\in\Y} \sum_{i=1}^n \alpha_i(x) \loss(y,y_i)
}
since it is not guaranteed in general to have one single minimizer for any given $x\in\X$. As we will discuss in the following, existence of a measurable function $\fn$ that satisfies such inclusions requirement for any $x\in\X$ can be guaranteed under mild assumptions.

\section{The \structshort{} Framework}\label{sec:app-framework}

This section is devoted to characterize the theoretical properties of the \structshort{} framework introduced in \Cref{sec:framework}. In particular we prove the results in \cref{lem:fstar-in-terms-of-gstar} (Fischer Consistency) and \cref{prop:comparison-inequality} (Comparison Inequality), which relate the ``surrogate'' risk to the original structured prediction one.

We begin by proving that both the structured risk $\E$ and $\rr$ admit measurable minimizers under very mild conditions. 

\begin{lemma}[Existence of a minimizer for $\E$]\label{lemma:solution-structured-risk}
Let $\loss:\Z\times\Y \to \R$ be a continuous function and $\Z$ a compact set. Then, the expected risk $\E$ in \cref{eq:expected-risk} admits a measurable minimizer $f^*: \X \to \Z$ such that
\eqal{\label{eq:solution_expected_risk}
  f^*(x) \in \argmin_{z\in\Z} ~ \int_{\Y} \loss(z,y) d\rho(y|x)
}
almost everywhere on $\domrho$. Moreover, the function $m: \X \to \R$ such that
\eqals{
m(x) = \inf_{z \in \Z} r(x,z), \qquad \mbox{with} \qquad r(x,z) = \left\{ \begin{array}{cc} 
                      \int_\Y \loss(z,y)d\rho(y|x)  & \mbox{if } x \in \domrho \\
                      0                               & \mbox{otherwise}
   \end{array} \right.
}
for any $x\in\X$, is measurable. 
\end{lemma}

\begin{proof}
Since $\loss$ is continuous and $\rho(y|x)$ is a regular conditional distribution, then $r$ is a Carath\'{e}odory function \citep[see Definition $4.50$ (pp. $153$) in][]{aliprantis2006}, namely continuous in $z$ for each $x\in\X$ and measurable in $x$ for each $z\in\Z$. Thus, by \citep[Theorem $18.19$ pp. $605$ in][]{aliprantis2006} (or Aumann's measurable selection principle \citep{steinwart2008,castaing2006}), we have that $m$ is measurable and that there exists a measurable $f^*:\X\to\Z$ such that $r(x, f^*(x)) = m(x)$ for all $x\in\X$. Moreover, by definition of $m$, given any measurable $f: \X \to \Z$, we have $m(x) \leq r(x, f(x))$. Therefore,
\eqals{
\E(f^*) = \int_\X r(x, f^*(x))d\rhox(x) = \int_\X m(x) d\rhox(x) \leq \int_\X r(x, f(x)) d\rhox(x) = \E(f).
}
We conclude $\E(f^*) \leq \inf_{f:\X \to \Z} \E(f)$ and, since $f^*$ is measurable, $\E(f^*) = \min_{f:\X\to\Y} \E(f)$ and $f^*$ is a global minimizer.
\end{proof}




\noindent In the following we will assume $\loss:\Z\times\Y\to\R$ to admit an \structshort{}, with associated Hilbert space $\hh$ and feature maps $\zmap:\Z\to\hh$ and $\ymap:\Y\to\hh$. We recall that the surrogate risk associated is defined as 
\eqals{
    \rr(g) = \int_{\X\times\Y} \|g(x) - \ymap(y)\|_\hh^2 ~d\rho(x,y)
}
for any $g:\X\to\hh$. Below we show that the global minimizer of $\rr$ corresponds to the conditional expectation of $\ymap(y)$. 

\begin{lemma}[Existence of a minimizer for $\rr$]\label{lemma:surrogate-problem-sol}
Let $\hh$ a separable Hilbert space and $\ymap:\Y\to\hh$ measurable and bounded with $\sup_{y\in\Y}\|\ymap(y)\|_\hh\leq\boldsymbol{\Phi}$. Then, the function $g^*:\X\to\hh$ such that
\eqal{\label{eq:g_average}
    g^*(x) = \int_\Y \ymap(y) d\rho(y|x) \quad \forall x\in \domrho
}
and $g^*(x) = 0$ otherwise, belongs to $\LXH$ and is a minimizer of the surrogate risk $\rr$. Moreover, for any $g\in\LXH$, 
\eqal{\label{eq:equation_excess_ls_risk}
    \mathcal{R}(g) - \mathcal{R}(g^*) = \int_\X \|g(x) - g^*(x)\|_\hh^2 ~d\rhox(x)
}
Hence, any minimizer of $\rr$ is equal to $g^*$ almost everywhere on the domain of $\rhox$.
\end{lemma}
\begin{proof}
By hypothesis, $\|\psi\|_{\hh}$ is measurable and bounded. Therefore, since $\rho(y|x)$ is a regular conditional probability, we have that $g^*$ is measurable on $\X$ (see for instance \cite{steinwart2008}). Moreover, the norm of $g^*$ is dominated by the constant function of value $\boldsymbol \Phi$, thus $g^*$ is integrable on $\X$ with respect to $\rhox$ and in particular it is in $L^2(\X,\rhox,\hh)$ since $\rhox$ is a finite regular measure. Recall that since $\rho(y|x)$ is a regular conditional distribution, for any measurable $g:\X\to\hh$ we have
\eqals{
{\cal R}(g) = \int_{\X\times\Y} \|g(x) - \psi(y)\|_\hh^2 d\rho(x,y) = \int_{\X} \int_\Y \|g(x) - \psi(y)\|_\hh^2 d\rho(y|x)d\rhox(x).
}
Notice that $g^*(x) = \argmin_{h \in \hh} \int_\Y \|h - \psi(y)\|_\hh^2 d\rho(y|x)$ almost everywhere on $\domrho$. Indeed,
\eqals{
  \int_\Y \|h - \psi(y)\|_\hh^2 d\rho(y|x) & = \|h\|_\hh^2 - 2 \scal{h}{\left(\int_\Y \psi(y)d\rho(y|x)\right)} + \int_\Y \|\psi(y)\|_\hh^2 d\rho(y|x) \\
  & = \|h\|_\hh^2 - 2 \scal{h}{g^*(x)}_\hh + const.
}
for all $x\in\domrho$, which is minimized by $h = g^*(x)$ for all $x\in\domrho$. Therefore, since $\domrho$ is equal to the support of $\rhox$ up to a set of measure zero, we conclude that $\mathcal{R}(g^*)\leq \inf_{g:\X\to\hh}\mathcal{R}(g)$ and, since $g^*$ is measurable, $\mathcal{R}(g^*) = \min_{g:\X\to\hh}\mathcal{R}(g)$ and $g^*$ is a global minimizer as required.

Finally, notice that for any $g:\X\to\hh$ we have
\eqals{
  \mathcal{R}(g) - \mathcal{R}(g^*) & = \int_{\X\times\Y} \|g(x) - \psi(y)\|_\hh^2 - \|g^*(x) - \psi(y)\|_\hh^2 ~d\rho(x,y) \\ 
  & = \int_{\X} \|g(x)\|_\hh^2 - 2 \scal{g(x)}{\left(\int_\Y \psi(y) d\rho(y|x)\right)}_\hh + \|g^*(x)\|_\hh^2 ~d\rhox(x) \\
  & = \int_\X \|g(x)\|_\hh^2 - 2 \scal{g(x)}{g^*(x)}_\hh + \|g^*(x)\|_\hh^2 ~d\rhox(x) \\
  & = \int_\X \|g(x) - g^*(x)\|_\hh^2 ~d\rhox(x),
}
which proves \cref{eq:equation_excess_ls_risk}. Therefore, for any measurable minimizer $g':\X\to\hh$ of the surrogate expected risk, we have $\mathcal{R}(g')-\mathcal{R}(g^*) = 0$ which, by the relation above, implies $g'(x) = g^*(x)$ a.e. on $\domrho$.
\end{proof}

\noindent Combining the characterizations of the global minimizers of the two risks $\E$ and $\rr$ we can now prove the following.

\LFstarCharacterization*

\begin{proof}
By \cref{lemma:surrogate-problem-sol} we know that $g^*(x) = \int_\Y \psi(y)d\rho(y|x)$ almost everywhere on $\domrho$ and is the minimizer of $\cal{R}$. Therefore, for every $z\in\Z$ we have
\eqals{
  \scal{\psi(z)}{g^*(x)}_\hh & = \scal{\psi(z)}{\int_\Y \ymap(y)d\rho(y|x)}_\hh \\
  & = \int_\Y \scal{\psi(z)}{\ymap(y)}_\hh d\rho(y|x) = \int_\Y \loss(z,y)d\rho(y|x)
}
almost everywhere on $\domrho$. Thus, for any measurable function $f: \X \to \Z$ we have
\eqals{
{\cal E}(f) &= \int_{\X\times\Y} \loss(f(x),y) d\rho(x,y) = \int_{X} \int_{\Y} \loss(f(x),y) d\rho(y|x)d\rhox(x) \\
&= \int_{X}\scal{\psi(f(x))}{g^*(x)}_\hh d\rhox(x).
}
We conclude that a minimizer $\fstar:\X\to\Z$ of $\E$ can be characterized as a function minimizing pointwise the integral above, namely
\eqals{
    \fstar(x) \in\argmin_{z\in\Z} ~\scal{\psi(f(x))}{g^*(x)}_\hh
}
almost everywhere on $\domrho$. 
\end{proof}

\noindent We now prove \cref{prop:comparison-inequality}, characterizing the relation between the excess risks associated to $\rr$ and $\E$.

\PComparison*

\begin{proof}
By applying \Cref{lem:fstar-in-terms-of-gstar}, we have
\eqals{
    \E(f) - \E(f^*) &= \int_{\X\times\Y} \loss(f(x),y) - \loss(\fstar(x),y) ~d\rho(x,y) \\ 
                    & = \int_{\X\times\Y} \scal{\psi(f(x)) - \psi(\fstar(x))}{\ymap(y)}_\hh ~d\rho(x,y) \\ 
                    & = \int_\X \scal{\psi(f(x)) - \psi(\fstar(x))}{\left(\int_\Y \ymap(y) ~ d\rho(y|x)\right)}_\hh d\rhox(x) \\
                    & = \int_\X \scal{\psi(f(x)) - \psi(\fstar(x))}{g^*(x)}_\hh d\rhox(x) \\ 
                    & = A + B.
}
where in the last equation we have removed and added a term $\int_\X \scal{\psi(f(x))}{g(x))}_\hh  ~d\rhox(x)$ leading to
\eqals{
    A & = \int_\X \scal{\psi(f(x))}{(g^*(x) - g(x))}_\hh  ~d\rhox(x) \\
    B &= \int_\X \scal{\psi(f(x))}{g(x)}_\hh - \scal{\psi(\fstar(x))}{g^*(x)}_\hh ~d\rhox(x) 
}
Now, the term A can be minimized by taking the supremum over $\Z$ so that
\eqals{
    A \leq \int_\X \sup_{z\in\Z} \Big|\scal{\psi(z)}{g^*(x) - g(x))}_\hh\Big| ~d\rho_\X(x).
}
For B, we observe that from the characterization of $f$ in the hypothesis and of $\fstar$ by \cref{lem:fstar-in-terms-of-gstar}, we have
\eqals{
\scal{\psi(\fstar(x))}{g^*(x)}_\hh  &= \inf_{z \in \Z} \scal{\psi(z)}{g^*(x)}_\hh, \\ \scal{\psi(f(x))}{g(x)}_\hh  &= \inf_{z \in \Z} \scal{\psi(z)}{g(x)}_\hh,
}
for all $x\in\X$. Therefore,
\begin{align}
    B   & = \int_\X ~ \inf_{z\in\Z} \scal{\psi(z)}{g(x)}_\hh - \inf_{z\in\Z} \scal{\psi(z)}{g^*(x)}_\hh ~ d\rhox(x) \\
        & \leq \int_\X \sup_{z\in\Z} \Big|\scal{\psi(z)}{(g(x)- g^*(x))}_\hh\Big| d\rhox(x) 
\end{align}
where we have used the fact that for any given two functions $\eta,\zeta:\Z\to\R$ we have 
\begin{equation}
\left|\inf_{z\in\Z} \eta(z) - \inf_{z\in\Z} \zeta(z)\right| \leq \sup_{z\in\Z} |\eta(z) - \zeta(y)|.
\end{equation}
Therefore, by combining the bounds on $A$ and $B$ we have
\eqals{
    \E(f) - \E(f^*) & \leq 2 \int_\X \sup_{z\in\Z} \Big|\scal{\psi(z)}{g^*(x) - g(x)}_\hh\Big| ~d\rhox(x) \\
    & \leq 2 \int_\X \sup_{z\in\Z} \|\psi(z)\|_\hh \|g^*(x) - g(x)\|_\hh ~d\rhox(x) \\
    & \leq 2 \closs \int_\X \|g^*(x) - g(x)\|_\hh ~d\rhox(x) \\
    & \leq 2 \closs \sqrt{\int_\X \|g^*(x) - g(x)\|_\hh^2 ~d\rhox(x)}, \\
}
where for the last inequality we have used the Jensen's inequality. The proof is concluded recalling that, by \cref{lemma:surrogate-problem-sol} 
\begin{equation}\label{eq:residual_error_equivalence}
    \mathcal{R}(g) - \mathcal{R}(g^*) = \int_{\X} \|g(x) - g^*(x)\|_\hh^2 ~ d\rho_\X(x)
\end{equation}
\end{proof}

\noindent We conclude proving the result in \cref{prop:derivation}, which is a direct consequence of the linearity induced by the \structshort{} definition.

\PLossTrick*

\begin{proof}
For any $z\in\Z$ and $x\in\X$ we have
\eqals{
    \scal{\psi(z)}{\gn(x)}_\hh & = \scal{\psi(z)}{\sum_{i=1}^n \alpha_i(x)\ymap(y_i)}_\hh \\
    & = \sum_{i=1}^n \alpha_i(x) ~ \scal{\psi(z)}{\ymap(y_i)}_\hh \\
    & = \sum_{i=1}^n \alpha_i(x) \loss(z,y_i).
}
Therefore, substituting the above equation in the definition of $\fn$ concludes the proof. 
\end{proof}

\section{Universal Consistency and Learning Bounds}\label{sec:app-theory}

\paragraph{Additional Notation}
Let $k:\X\times\X\to\R$ a positive semidefinite function on $\X$. We denote $\ff$ the Hilbert space obtained by the completion
\eqals{
  \ff = \overline{\Span \{ k(x,\cdot) \ | \ x\in\X\} }
}
according to the norm induced by the inner product $\scal{k(x,\cdot)}{k(x',\cdot)}_\HX = k(x,x')$. Spaces $\HX$ constructed in this way are known as {\it reproducing kernel Hilbert spaces} and there is a one-to-one relation between a kernel $k$ and its associated RKHS. For more details on RKHS we refer the reader to \cite{berlinet2011}. Given a kernel $k$, in the following we will denote with $\xmap:\X\to\HX$ the feature map $\xmap(x) = k(x,\cdot) \in\HX$ for all $x\in\X$. We say that a kernel is bounded if $\|\xmap(x)\|_\HX \leq \kappa$ with $\kappa>0$. Note that $k$ is bounded if and only if $k(x,x') = \scal{\xmap(x)}{\xmap(x')}_\HX \leq \|\xmap(x)\|_\HX \|\xmap(x')\|_\HX\leq \kappa^2$ for every $x,x'\in\X$.  In the following we will always assume $k$ to be continuous and bounded by $\kappa>0$. The continuity of $k$ with the fact that $\X$ is Polish implies $\HX$ to be separable \cite{berlinet2011}.

We introduce here the ideal and empirical operators that we will use in the following to prove the main results of this work. 

\begin{itemize}
\item $S:\HX\to L^2(\X,\rho_\X)$ s.t. $f\in\HX\mapsto\scal{f}{\xmap(\cdot)}_\HX\in L^2(\X,\rho_\X)$, with adjoint 
\item $S^*:\LX\to\HX$ s.t. $h\in\LX\mapsto \int_\X h(x)\xmap(x)d\rhox(x)\in\HX$,
\item $Z:\HY\to L^2(\X,\rho_\X)$ s.t. $h\in\HY\mapsto\scal{h}{g^*(\cdot)}_\HY\in L^2(\X,\rhox)$, with adjoint
\item $Z^*:\LX\to\HY$ s.t. $h\in\LX\mapsto \int_\X h(x)g^*(x)d\rhox(x)\in\HY$,
\item $C = S^*S:\HX\to\HX$ and $L = SS^*:L^2(\X,\rho_\X)\to L^2(\X,\rho_\X)$,
\end{itemize}
with $g^*(x) = \int_\Y \psi(y)d\rho(y|x)$ defined according to \cref{eq:g_average}, (see \cref{lemma:surrogate-problem-sol}). 

Given a set of input-output pairs $\{(x_i,y_i)\}_{i=1}^n$ with $(x_i,y_i)\in\X\times\Y$ independently sampled according to $\rho$ on $\X\times\Y$, we define the empirical counterparts of the operators just defined as
\begin{itemize}
\item $\hat{S}:\HX\to\R^n$ s.t. $f \in\HX \mapsto \frac{1}{\sqrt{n}}(\scal{\xmap(x_i)}{f}_\HX)_{i=1}^n \in \R^n$, with adjoint
\item $\hat{S}^*:\R^n\to\HX$ s.t. $v = (v_i)_{i=1}^n\in\R^n \mapsto \frac{1}{\sqrt{n}} \sum_{i=1}^n v_i \xmap(x_i)$,
\item $\hat{Z}:\HY\to\R^n$ s.t. $h \in\HY \mapsto \frac{1}{\sqrt{n}}(\scal{\psi(y_i)}{h}_\HY)_{i=1}^n \in \R^n$, with adjoint
\item $\hat{Z}^*:\R^n\to\HY$ s.t. $v = (v_i)_{i=1}^n\in\R^n \mapsto \frac{1}{\sqrt{n}} \sum_{i=1}^n v_i \psi(y_i)$,
\item $\hat{C} = \hat{S}^*\hat{S}:\HX\to\HX$ and $K = n \hat{S}\hat{S}^*\in\R^{n \times n}$ is the empirical kernel matrix. 
\end{itemize}
In the rest of this section we denote with $A + \la$, the operator $A + \la I$, for any symmetric linear operator $A$, $\la \in \R$ and $I$ the identity operator.

\subsection{Preliminary results}

We recall here a basic result characterizing the operators introduced above.

\begin{proposition}\label{prop:basic_operator_result}
With the notation introduced above, 
\eqal{
  C = \int_\X \xmap(x) \otimes \xmap(x) d\rhox(x) \mbox{ \ \ \ \ and \ \ \ \ } Z^*S = \int_{\X\times\Y} \psi(y) \otimes \xmap(x) d\rho(x,y)
}
where $\otimes$ denotes the tensor product. Moreover, when $\xmap$ and $\psi$ are bounded by respectively $\kappa$ and $\Q$, we have the following facts
\begin{enumerate}[(i)]
\item $\tr(L) = \tr(C) = \|S\|_\HS^2 = \int_\X \|\xmap(x)\|_\HX^2 d\rhox(x) \leq \kappa^2$
\item $\|Z\|_\HS^2 = \int_X \|g^*(x)\|^2 d\rhox(x) = \|g^*\|_\rhox^2 < +\infty$.
\end{enumerate}
\end{proposition}
\begin{proof}
By definition of $C = S^*S$, for each $h,h'\in\HX$ we have
\eqals{
  \scal{h}{Ch'}_\HX = \scal{Sh}{Sh'}_\rhox  & = \int_\X \scal{h}{\xmap(x)}_\HX \scal{\xmap(x)}{h'}_\HX d\rhox(x) \\ 
  & = \int_\X \scal{h}{\Big(\xmap(x)\scal{\xmap(x)}{h'}_\HX\Big)}_\HX d\rhox(x) \\
  & = \int_\X \scal{h}{\Big(\xmap(x)\otimes\xmap(x) \Big)h'} d\rhox(x) \\
  & = \scal{h}{\Big(\int_\X \xmap(x)\otimes\xmap(x)d\rhox(x)\Big)h'}_\HX
}
since $\xmap(x)\otimes\xmap(x):\HX\to\HX$ is the operator such that $h\in\HX\mapsto \xmap(x)\scal{\xmap(x)}{h}_\HX$. The characterization for $Z^*S$ is analogous.

Now, $(i)$. The relation $\tr(L) = \tr(C) = \tr(S^*S) = \|S\|_\HS^2$ holds by definition. Moreover
\eqals{
  \tr(C) = \int_\X \tr(\xmap(x) \otimes \xmap(x))~ d\rhox(x) = \int_\X \|\xmap(x)\|_\HX^2 ~d\rhox(x)
}
by linearity of the trace. $(ii)$ is analogous. Note that $\|g^*\|_\rhox^2 < +\infty$. by \cref{lemma:surrogate-problem-sol} since $\psi$ is bounded by hypothesis.

\end{proof}

\begin{lemma}\label{lm:dec-R}
Let $\gn(x) = \widehat{G}^* \xmap(x)$ with $\widehat{G}:\hh \to \HX$ a bounded linear operator, then
\eqal{
{\cal R} & (\gn) - {\cal R}(g^*)  = \|S\hat{G} - Z\|_\HS^2,
}
where $\|A\|_\HS^2 := \tr(A^*A)$, for a linear operator $A$, is the Hilbert-Schmidt norm.
\end{lemma}
\begin{proof}
By \cref{lemma:surrogate-problem-sol}, we know that $g^*(x) =  \int_\Y \psi(y) d\rho(y|x)$ almost everywhere on the support of $\rhox$, moreover by \cref{lm:ghatla_def} $\gn$. Therefore, a direct application of \cref{prop:basic_operator_result} leads to
\eqals{
  {\cal R} & (\gn) - {\cal R}(g^*) = \int \|\gn(x) - g^*(x)\|_\HY^2 d\rhox(x) = \\
  & = \int_\X \|\hat{G}^*\xmap(x)\|_\HY^2 - 2 \scal{\hat{G}^*\xmap(x)}{g^*(x)}_\HY + \|g^*(x)\|_\HY^2 d\rhox(x) \\
  & = \int_\X \tr\left(\hat{G}^* \Big( \xmap(x) \otimes \xmap(x) \Big) \hat{G}\right) - 2 \tr\left(\hat{G}^* \Big(\xmap(x)\otimes g^*(x)\Big)\right) + \tr(g^*(x) \otimes g^*(x)) d\rhox(x) \\
  & = \tr(\hat{G}^*S^*S\hat{G}) - 2\tr(\hat{G}^*S^*Z) + \tr(Z^*Z) \\
  & = \|S\hat{G} - Z\|_\HS^2
}  
\end{proof}

\subsection{Analytic Decomposition for Spectral Filters}

To study the various estimators considered in this paper, we need to introduce the notion of {\em spectral filter}.

\begin{definition}[spectral filters \citep{engl1996regularization}]\label{def:spectral-filter}
Let $\kappa>0$. Then $\eta_\la:(0,\kappa^2]\to\R$ is a spectral filter if there exist $q_1, q_2>0$ s.t. for $\sigma\in(0,\kappa^2]$ and $\la > 0$
\eqals{
(\sigma + \la)\eta_\la(\sigma) \leq q_1, \qquad\qquad  (1 - \sigma \eta_\la(\sigma))(\sigma + \la) \leq q_2 \la.
}
\end{definition}

\noindent In this work we have considered a simplified definition of spectral filters with respect to the standard notion. In particular, we do not make a distinction between filters with qualification {\itshape larger} than $1$ (see e.g. \citep{engl1996regularization,bauer2007regularization}). The following result gives three concrete examples of spectral filters that will be useful to characterize the estimators $\ghat$ studied in this work.

\begin{lemma}\label{lm:spectral-filter-examples}
The following functions are spectral filters:
\begin{enumerate}
    \item (Ridge Regression) \quad $\eta_\la(\sigma) = (\sigma + \la)^{-1}$, with $q_1 = q_2 =1$
    \item (L2-Boosting) \quad $\eta_\la(\sigma) = \nu\sum_{j=0}^t(1-\nu \sigma)^{j}$, with step-size $0<\nu<1/\kappa^2$ and $\lambda = 1/t$.\\
    \hphantom{(L2-Boosting)} \quad With constants $q_1 = 1 + 2\nu$ and $q_2 = e^{\nu-1}/\nu$.
    \item (PCR) \quad $\eta_\la(\sigma) = \frac{1}{\sigma}{\bf1}_{\sigma > \la}$, where ${\bf 1}_{\sigma > \la} = 1$ when $\sigma \geq \la$ and 0 otherwise.\\
    \phantom{(PCR)} \quad With constants $q_1 = q_2 = 2$.
\end{enumerate}
\end{lemma}

\begin{proof}
{\itshape (Ridge Regression)}. It is easy to show that 
\begin{itemize}
    \item \qquad $\displaystyle{(\sigma + \la)\eta_\la(\sigma) ~=~ (\sigma + \la)(\sigma + \la)^{-1} = 1 = q_1}$,
    \item \qquad $\displaystyle{\frac{1}{\la}(1 - \sigma\eta_\la)(\sigma + \la) ~=~ \frac{1}{\la}\left(1 - \frac{\sigma}{\sigma+\la}\right)(\sigma+\la) = 1 = q_2}$.
\end{itemize}

\noindent {\itshape (L2-Boosting)}. Let $\la = 1/t$. Recall that since $\nu<1/\kappa^2$ and $\sigma\in(0,\kappa^2]$, we have $\nu\sigma<1$. Therefore, $\sum_{j=0}^{+\infty} (1-\nu\sigma)^j = 1/(\nu\sigma)$ and we have
\eqals{
    \eta_\la(\sigma) = \nu\sum_{j=0}^{t} (1 - \nu\sigma)^{j} = \frac{1}{\sigma} \left(1 - (1-\nu\sigma)^{t+1}\right).
}
Now $(\sigma + \la)\eta_\la(\sigma) = \sigma\eta_\la(\sigma) + \la\eta_\la(\sigma)$. Then,
\eqals{
    \sigma\eta_\la(\sigma) = \sigma \frac{1}{\sigma} \left(1 - (1-\nu\sigma)^{t+1}\right) < 1,
}
since $\nu\sigma>0$. Moreover, since $\nu\sigma<1$,
\eqals{
    \la\eta_\la(\sigma) = \frac{1}{t}\nu\sum_{j=0}^{t} (1-\nu\sigma)^j \leq \frac{(t+1)\nu}{t} < 2\nu.
}
Hence we have 
\eqals{
    (\sigma + \la)\eta_\la(\sigma) \leq 1 + 2\nu = q_1.
}
Now, since $(1-z) \leq e^{-z}$ and defining $x = (t+1)\sigma$, we have 
\eqals{
    \frac{1}{\la}(1 - \sigma\eta_\la(\sigma))(\sigma+\la) & = t(1-\nu\sigma)^{t+1} (\sigma+1/t) \\ 
    & \leq t e^{-\nu \sigma(t+1)}(\sigma + 1/t) \\
    & = e^{-\nu\sigma(t+1)} + t\sigma e^{-\nu\sigma(t+1)}\\
    & \leq e^{-\nu\sigma(t+1)} + (t+1)\sigma e^{-\nu\sigma(t+1)}\\
    & = e^{-\nu x} + x e^{-\nu x}\\
    & \leq e^{\nu-1}/\nu = q_2.
}

\noindent {\itshape (PCR)}. Let $\sigma<\la$, then $\eta_\la(\sigma) = 0$ and
\begin{itemize}
    \item \qquad $\displaystyle{(\sigma + \la)\eta_\la(\sigma) ~=~ 0 < 2 = q_1}$,
    \item \qquad $\displaystyle{\frac{1}{\la}(1 - \sigma\eta_\la)(\sigma + \la) ~=~ \frac{1}{\la}(\sigma+\la) < \frac{2\la}{\la} = 2 = q_2}$.
\end{itemize}
If $\sigma\geq\la$, we have $\eta_\la(\sigma) = 1/\sigma$ and
\begin{itemize}
    \item \qquad $\displaystyle{(\sigma + \la)\eta_\la(\sigma) ~=~ \frac{\sigma+\la}{\sigma} < \frac{2\sigma}{\sigma} = 2 = q_1}$,
    \item \qquad $\displaystyle{\frac{1}{\la}(1 - \sigma\eta_\la)(\sigma + \la) ~=~0 < 2 = q_2}$.
\end{itemize}

\end{proof}

\noindent We will be applying filters $\eta_\la$ to the specturm of an operator as follows. Let $M:\hh_1\to\hh_2$ be a compact linear operator between two separable Hilbert spaces $\hh_1,\hh_2$. Let $M = \sum_{i=1}^{+\infty} ~\sigma_i ~ u_i \otimes v_i$ be the singular value decomposition of $M$, with $(u_i)_{i\in\N}$ and $(v_i)_{i\in\N}$ a suitable pair of orthonormal bases of $\hh_1$ and $\hh_2$ respectively and $\sigma_i\geq0$ for every $i\in\N$. We denote the application of $\eta_\la$ to $M$ as 
\eqals{
    \eta_\la(M) = \sum_{i=1}^{+\infty} ~ \eta_\la(\sigma_i)~ u_i \otimes v_i.
}
The following results shows that several estimators described in \cref{sec:alternative-alpha} can be formulated in terms of spectral filters.

\begin{lemma}\label{lm:ghatla_def}
The following algorithms can be represented as
\eqal{\label{eq:characterization-spectral-filtering}
    \ghat_\la(x) = \widehat{G}_\la^* \xmap(x), \qquad\qquad  \widehat{G}_\la = \filter_\la(\widehat{C})\widehat{S}^*\widehat{Z},
}
where $\eta_\la$ is a spectral filter, in particular
\begin{enumerate}
    \item (Kernel Ridge Regression) \quad $\eta_\la(\sigma) = (\sigma + \la)^{-1}$
    \item (Kernel L2-Boosting) \quad $\eta_\la(\sigma) = \nu\sum_{j=0}^t(1-\nu \sigma)^{j}$, with step-size $\nu$ and $\lambda = 1/t$,
    \item (Kernel PCR) \quad $\eta_\la(\sigma) = \frac{1}{\sigma}{\bf1}_{\sigma > \la}$, where ${\bf 1}_{\sigma > \la} = 1$ when $\sigma \geq \la$ and 0 otherwise.
\end{enumerate}
\end{lemma}
\begin{proof}
Recall that, according to \cref{eq:gn-as-weighted-sum} the estimator $\ghat$ is such that, for any $x\in\X$ 
\eqals{
    \ghat_\la(x) = \sum_{i=1}^n \alpha_i(x) \ymap(y_i).
}
It follows by the definition of the three methods considered to learn the vector-valued function $\alpha$ that, for any $x\in\X$ 
\eqals{
 \alpha(x) = \frac{1}{n} \eta_\la\left(\frac{K}{n}\right) \msf{v}(x)
}
where $\eta_\la$ is the corresponding spectral filter function given in the thesis of this Lemma. Recall that $\msf{v}(x) = (k(x_1,x),\dots, k(x_n,x))$. By definition of $\Shat$ and $\hat Z$, we have $\msf{v}(x) = \sqrt{n}\Shat \xmap(x)$ and $\sum_{i=1}^n \alpha_i(x) \ymap(y_i) = \frac{1}{\sqrt{n}} \widehat{Z}^*\alpha(x)$. Then
\eqals{
\ghat_\la(x) = \frac{1}{\sqrt{n}}\widehat{Z} \alpha(x) = \frac{1}{\sqrt{n}}\widehat{Z}^*\eta_\la\left(\frac{K}{n}\right) \msf{v}(x) = \widehat{Z}^*\eta_\la\left(\frac{K}{n}\right) \Shat \xmap(x).
}
Since $K = n \Shat \Shat^*$ , then
\eqals{
\eta_\la\left(\frac{K}{n}\right) \Shat =  \eta_\la(\Shat \Shat^*) \Shat = \Shat \eta_\la(\Shat^*\Shat) = \Shat \eta_\la(\Chat),
}
from which it follows
\eqals{
\ghat_\la(x) = \widehat{Z}^*\Shat \eta_\la(\Chat) \xmap(x),
}
as required.
\end{proof}

\noindent With the characterization provided by the result above, we now proceed in deriving an upper bound for the risk of estimators obtained via spectral filtering methods. 

\begin{theorem}\label{thm:decompose_risk}
Let $\ghat$ be characterized as in \cref{lm:ghatla_def} in terms of a spectral filter $\eta_\la$ with constants $q_1,q_2>0$. Let $\beta = \|\Cnl^{-1/2}\Cl^{1/2}\|^2$, with $G_\la = S^*L_\la^{-1} Z$. Then, 
\eqals{
|{\cal R}(\gn_\la) - {\cal R}(g^*)|^{1/2} & ~\leq~ q_1\beta \|\Cl^{-1/2}(\Shat^*\widehat{Z} - \widehat{C} G_\la)\|_\HS ~+~ 2(1+q_2)\beta ~\la \|L_\la^{-1} Z\|_\HS,
}
\end{theorem}
\begin{proof}
From \cref{lm:ghatla_def} we know that $\ghat_\la(x) = \widehat{G}_\la^* \xmap(x)$ with $\widehat{G}_\la = \eta_\la(\widehat{C}) \Shat^*\widehat{Z}$. From \cref{lm:dec-R} we know that ${\cal R}(\gn_\la) - {\cal R}(g^*) = \|S \widehat{G}_\la - Z\|^2_\HS$. We add and remove the term $S \eta_\la(\widehat{C}) \Shat^*\Shat G_\la$, with $G_\la = S^*L_\la^{-1}Z$ and $L_\la = L + \la I$, namely
\eqals{\label{eq:decomposition-sgla-z}
S \widehat{G}_\la - Z &= S \eta_\la(\widehat{C}) \Shat^*(\widehat{Z} - \Shat G_\la) ~+~ S \eta_\la(\widehat{C}) \Shat^*\Shat G_\la - Z.
}
We decompose the first term in the sum above as 
\eqals{
S \eta_\la(\widehat{C}) \Shat^*(\widehat{Z} - \Shat G_\la) = (S\Cnl^{-1/2}) (\Cnl^{1/2}\eta_\la(\widehat{C})\Cnl^{1/2})(\Cnl^{-1/2}\Cl^{1/2})[\Cl^{-1/2}\Shat^*(\widehat{Z} - \Shat G_\la)].
}
Hence, we have
\eqals{
\|S \eta_\la(\widehat{C}) \Shat^*(\widehat{Z} &  - \Shat G_\la)\|_\HS \\
& = \|S\Cnl^{-1/2}\| \|\Cnl^{1/2}\eta_\la(\widehat{C})\Cnl^{1/2}\| \|\Cnl^{-1/2}\Cl^{1/2}\| \|\Cl^{-1/2}\Shat^*(\widehat{Z} - \Shat G_\la)\|_\HS.
}
Note that since $\eta_\la$ is a filter and $\Chat$ and $\Cnl^{1/2}$ have same spectral decomposition,
\eqals{
\|\Cnl^{1/2}\eta_\la(\widehat{C})\Cnl^{1/2}\| \leq \sup_{\sigma \in (0,\kappa^2]} (\sigma + \la) \eta_\la(\sigma) \leq q_1.
}
Moreover, since $\widehat{C} = \widehat{S}^*\widehat{S}$, then
\eqals{
\Cl^{-1/2}\Shat^*(\widehat{Z} - \Shat G_\la) &= \Cl^{-1/2}(\Shat^*\widehat{Z} - \widehat{C} G_\la).
}
We now focus on the second term of the sum in \cref{eq:decomposition-sgla-z}. Let
$r_\la(\sigma) = 1 - \sigma \eta_\la(\sigma)$. Since $\widehat{C} = \Shat^*\Shat$, we have
\eqals{
S \eta_\la(\widehat{C}) \Shat^*\Shat G_\la - Z = (S G_\la - Z) - S r_\la(\widehat{C}) G_\la.
}
In particular, since $L = S S^*$ and by definition of $G_\la$ we have
\eqals{\label{eq:equivalent-sgla-z}
S G_\la - Z & = -(I - SS^*L_\la^{-1}) Z = -(I - L L_\la^{-1}) Z = -\la L_\la^{-1} Z,
}
since $(I - LL_\la^{-1}) = (L_\la - L)L_\la^{-1} = (L + \la - L) L_\la^{-1} = \la L_\la^{-1}$. Moreover 
\eqals{
S r_\la(\widehat{C}) G_\la  = (S \Cnl^{-1/2}) (\Cnl^{1/2} r_\la(\widehat{C}) \Cnl^{1/2}) (\Cnl^{-1/2} S^*) (L_\la^{-1} Z).
}
Now note that by definition of $r_\la$ we have
\eqals{
\|\Cnl^{1/2} r_\la(\widehat{C}) \Cnl^{1/2}\| \leq \sup_{\sigma \in (0,\kappa^2)}(1-\sigma \eta_\la(\sigma))(\sigma + \la) \leq q_2\la.
}
To conclude, since $\|S\Cnl^{-1/2}\| \leq \|\Cl^{1/2}\Cnl^{-1/2}\|$ \citep[see e.g.][]{rudi2015less}, we have
\eqals{
\|S \eta_\la(\widehat{C}) \Shat^*\Shat G_\la - Z\|_\HS \leq 2(1+q_2)\la \|\Cnl^{-1/2}\Cl^{1/2}\|^2 \|L_\la^{-1} Z\|_\HS.
}

\end{proof}

\subsection{Statistical Analysis}

In this section we use the decomposition in \cref{thm:decompose_risk} to derive statistical learning rates for the estimators $\ghat$. To this end, we recall the following result.

\begin{lemma}[\cite{carratino2018learning}, Lemma 3]\label{lm:beta}
Let $\delta \in (0,1)$. When $\la \geq \frac{9\kappa^2}{n} \log \frac{n}{\delta}$ then the following holds with probability at least $1-\delta$
\eqals{
\|\Cnl^{-1/2} \Cl^{1/2}\|^{2} \leq 2.
}
\end{lemma}
\begin{proof}
Apply Lemma 3 of \cite{carratino2018learning} with $R = n$ and $\zeta_i = \xmap(x_i)$.
\end{proof}

\noindent We now prove an intermediate result that will be instrumental in proving the excess risk bounds of the estimators $\ghat$. We recall the definition of effective dimension given in \cref{eq:eff-dim}, that will be useful in the following, namely
\eqals{
    \deff(\la) = \tr(C (C + \la I)^{-1}), \qquad \forall \la > 0.
}

\begin{proposition}\label{thm:probin}
Let $\delta \in (0,1)$ and $\la > 0$. The following holds with probability at least $1-\delta$,
\eqals{
\|\Cl^{-1/2}(\widehat{S}^*\widehat{Z} - \widehat{C}G_\la)\|_\HS ~~&\leq~~ \la \|L_\la^{-1} Z\|_\HS ~+~\frac{4\kappa \log \frac{2}{\delta}}{\sqrt{\la} n} (\kappa \|L_\la^{-1/2} Z\|) \\
{} &  \qquad + \quad \sqrt{\frac{16 (\deff(\la) + \kappa^2\la\|L_\la^{-1}Z\|_\HS^2)\log \frac{2}{\delta}}{n}}.
}
\end{proposition}
\begin{proof}
For any $i=1,\dots,n$ we consider the random linear operator 
\eqals{
\zeta_i = \Cl^{-1/2} \Big(\xmap(x_i) \otimes \ymap(y_i) -  \big(\xmap(x_i) \otimes \xmap(x_i)\big)G_\la\Big),
}
as a vector in the space of Hilbert-Schmidt operators. Hence, taking the expectation with respect to a random sample of training points $(x_i,y_i)_{i=1}^n$ from $\rho$, 
\eqals{
\mathbb{E} \zeta_i = C_\la(S^*Z - C G_\la),
}
since $\mathbb{E}[\ymap(y)|x_i] = g^*(x_i)$. Moreover, since $\|\Cl^{-1/2}S^*\| = \|\Cl^{-1/2}C^{1/2}\| \leq 1$, following the same reasoning in the proof of \cref{thm:decompose_risk} to obtain \cref{eq:equivalent-sgla-z}, we have
\eqals{
\|\mathbb{E} \zeta_i\|_\HS = \|\Cl^{-1/2}(S^*Z - C G_\la)\|_\HS \leq \|\Cl^{-1/2}S^*\|\|Z - SG_\la\|_\HS \leq \la\|L_\la^{-1} Z\|_\HS.
}
Now we need to study the moments of $Z_i$ to obtain the final result. First note that 
\eqals{
\|G_\la\| \leq \|S^* L_\la^{-1/2}\|\|L_\la^{-1/2} Z\| \leq \|L_\la^{-1/2} Z\|,
}
since $\|S^*L_\la^{-1/2}\| = \|L^{-1/2}L_\la^{-1/2}\| \leq 1$. Recall that $\sup_{y\in\Y}\|\ymap(y)\|\leq1$. Then, we have
\eqals{
\|\zeta_i\|_\HS ~\leq~ \|\Cl^{-1/2}\|\|\xmap(x_i)\|\Big(\|\ymap(y_i)\| + \|G_\la\|\|\xmap(x_i)\|\Big) ~\leq~ \la^{-1/2}\kappa \Big(1 + \kappa \|L_\la^{-1/2} Z\|\Big).
}
Hence, for any $p \geq 2$
\eqals{
\mathbb{E} \|\zeta_i - \mathbb{E} \zeta_i\|_\HS^p ~&\leq~ \mathbb{E} \|\zeta_i -  \zeta'_i\|_\HS^p ~\leq~ 2^p ~\mathbb{E} \|\zeta_i\|_\HS^p \\
& \leq~ 4 \Big(2\la^{-1/2}\kappa \big(1 + \kappa \|L_\la^{-1/2} Z\|\big)\Big)^{p-2} ~\mathbb{E} \|\zeta_i\|_\HS^2.
}
Moreover, denote by $\sigma^2(x)$ the conditional variance $\sigma^2(x) = \mathbb{E}_{y_i}[\|\ymap(y_i) - g^*(x_i)\|^2 | x_i]$. Since $\gstar(x) = \mathbb{E}[\varphi(y)|x]$ for any $x$ in the domain of $\rhox$, we have 
$\sigma(x) \leq 1$. Hence,
\eqals{
\mathbb{E} \|\zeta_i\|_\HS^2 &= \mathbb{E} \| \ymap(y_i) - g_\la(x_i)\|^2 \|\Cl^{-1/2}\xmap(x_i)\|^2\\
& = \mathbb{E}_{x_i} \|\Cl^{-1/2}\xmap(x_i)\|^2 \mathbb{E}_{y_i}[\|\ymap(y_i) - g_\la(x_i)\|^2 | x_i]  \\
& = \mathbb{E}_{x_i} \|\Cl^{-1/2}\xmap(x_i)\|^2 \big(\sigma(x)^2 + \|g^*(x) - g_\la(x)\|^2\big)\\
& \leq \deff(\la) + \kappa^2/\la ~\mathbb{E}\|g^*(x) - g_\la(x)\|^2,
}
where the last inequality follows by observing that
\eqals{
     \mathbb{E} \|\Cl^{-1/2}\xmap(x)\|^2 = \mathbb{E} ~\tr(\Cl^{-1} ~\xmap(x)\otimes\xmap(x)) = \tr(\Cl^{-1}C) = \deff(\la),
}
and also
\eqal{\label{eq:deff-less-lak}
  \deff(\la) = \mathbb{E} ~\|\Cl^{-1/2}\xmap(x)\|^2 \leq \mathbb{E} ~\|\Cl^{-1/2}\|^2~\|\xmap(x)\|^2 \leq \la^{-1} \kappa^2.
}
Recall from \cref{lemma:surrogate-problem-sol} that ${\cal R}(g) - {\cal R}(\gstar) = \mathbb{E}\|g(x) - \gstar(x)\|^2$. Then, by applying \cref{lm:dec-R}, we have
\eqals{
\mathbb{E}\|g^*(x) - g_\la(x)\|^2 &= \|g_\la - g^*\|_{L^2(X,\rho)}^2 \\
&= \|SG_\la - Z\|_\HS^2 \\ 
& = \|(SS^*L_\la^{-1} - I) Z\|^2_\HS\\
& = \la^2\|L_\la^{-1}Z\|_\HS^2.
}
Finally, we have for any $p \geq 2$
\eqals{
\mathbb{E} \|\zeta_i - \mathbb{E} \zeta_i\|_\HS^p & ~\leq~
\frac{1}{2} p! \Big(8 \deff(\la) + 8\kappa^2\la\|L_\la^{-1}Z\|_\HS^2\Big) \left(\frac{2\kappa}{\sqrt{\la}} (1 + \kappa \|L_\la^{-1/2} Z\|)\right)^{p-2}.
}

\noindent We can now apply Bernstein inequality as in (\cite{rudi2016rf} Proposition 2). We have 
\eqals{
\left\|\frac{1}{n} \sum_{i=1}^n \zeta_i - \mathbb{E} \zeta_i\right\|_\HS \leq \frac{4\kappa \log \frac{2}{\delta}}{\sqrt{\la} n} (1 + \kappa \|L_\la^{-1/2} Z\|) + \sqrt{\frac{16 (\deff(\la) + \kappa^2\la\|L_\la^{-1}Z\|_\HS^2)\log \frac{2}{\delta}}{n}}
}
holds with probability $1-\delta$. 

The proof is concluded by observing that
\eqals{
\|\Cl^{-1/2}(\widehat{S}^*\widehat{Z} - \widehat{C}G_\la)\|_\HS ~=~ \left\|\frac{1}{n} \sum_{i=1}^n \zeta_i\right\|_\HS ~\leq~ \left\|\frac{1}{n} \sum_{i=1}^n \zeta_i - \mathbb{E} \zeta_i\right\|_\HS +~ \left\|\mathbb{E} \zeta \right\|_\HS.
}
\end{proof}

\begin{theorem}\label{thm:surrogate-excess-risk-bound-with-lambda}
Under the assumptions of \cref{thm:decompose_risk}, let $\delta \in (0,1)$. Then, for $\la \geq \frac{9\kappa^2}{n} \log \frac{n}{\delta}$ the following holds with probability at least $1-\delta$
\eqals{
|{\cal R}(\gn_\la) - {\cal R}(g^*)|^{1/2} & \leq \frac{8q_1\kappa \log \frac{2}{\delta}}{\sqrt{\la} n} (1 + \kappa \|L_\la^{-1/2} Z\|) \\
{} &  \quad + \quad \sqrt{\frac{64 q_1^2 (\deff(\la) + \kappa^2\la\|L_\la^{-1}Z\|_\HS^2)\log \frac{4}{\delta}}{n}} \\
{} & \quad + \quad 2(2 + q_1 +2q_2) ~\la \|L_\la^{-1} Z\|_\HS
}
\end{theorem}
\begin{proof}
The result is obtained by decomposing the risk with \cref{thm:decompose_risk}, and controlling in high probability both terms $\beta = \|\Cnl^{-1/2}\Cl^{1/2}\|^2$ and
$\|\Cl^{-1/2}(\widehat{S}^*\widehat{Z} - \widehat{C}G_\la)\|_\HS$ and then taking the intersection bound of the two events.
\end{proof}

\subsection{Universal Consistency}
Now we are ready to give the universal consistency result.
\Tuniversal*
\begin{proof}
Recall that by \cref{thm:probin}, for any  $\la \geq \frac{9\kappa^2}{n} \log \frac{n}{\delta}$ the following holds with probability at least $1-\delta$
\eqals{
|{\cal R}(\gn_\la) - {\cal R}(g^*)|^{1/2} & \leq \frac{8q_1\kappa \log \frac{4}{\delta}}{\sqrt{\la} n} (1 + \kappa \|L_\la^{-1/2} Z\|) \\
{} &  \quad + \quad \sqrt{\frac{64 q_1^2 (\deff(\la) + \kappa^2\la\|L_\la^{-1}Z\|_\HS^2)\log \frac{4}{\delta}}{n}} \\
{} & \quad + \quad 2(2 + q_1 +2q_2) ~\la \|L_\la^{-1} Z\|_\HS.
}
In particular, let $n_0$ be sufficiently large, such that $n_0^{-1/2} \geq  \frac{9 \kappa^2}{n_0} \log \frac{n_0}{\delta}$. For any $n\in\N$ with $n\geq n_0$, let $\la_n = n^{-1/2}$. Recall that $\deff(\la) \leq \kappa^2/\la$ (see \cref{eq:deff-less-lak}) and note that  $\|L_\la^{-1/2} Z\| \leq \la^{-1/2} \|Z\|$, $\|L_\la^{-1}Z\|_\HS \leq \la^{-1} \|Z\|_\HS$ and $\|Z\| \leq \|Z\|_\HS$. Applying \cref{thm:probin} guarantees that the inequality
\begin{equation}\label{eq:event}
\begin{aligned}
|{\cal R}(\gn_\la) - {\cal R}(g^*)|^{1/2} &\leq [8q_1\kappa(1 + \kappa \|Z\|)]~ n^{-1/2}~\log(4/\delta) \\
{} &  \quad + \quad [64 q_1^2 \kappa^2(1 + \|Z\|_\HS^2)]^{1/2} ~ n^{-1/4}~(\log (4/\delta))^{1/2} \\
{} & \quad + \quad 2(2 + q_1 +2q_2) ~\la \|L_\la^{-1} Z\|_\HS,
\end{aligned}
\end{equation}
holds with probability at least $1-\delta$. We denote this event by $E_{n, \delta}$.

Recall that $L$ is a compact operator (actually Hilbert-Schmidt), hence it admits an eigendecomposition $L = \sum_{i \in \N} \sigma_i ~ u_i \otimes u_i$, with $\sigma_i \geq \sigma_j>0$  for $1 \leq i \leq j \in \N$ and $(u_i)_{i\in\N}$ is a set of orthonormal functions in $L^2(\X,\rhox)$. Note that $(u_i)_{i \in \N}$ is an orthonormal basis of $\LX$. To show this, consider $W \subseteq \X$ the support of $\rhox$. Note that $W$ is compact and Polish since it is a closed subset of the compact Polish space $\X$. Let ${\cal L}$ be the RKHS ${\cal L} = \overline{\Span\{k(x,\cdot)~|~x \in W\}}$, with same inner product of $\HX$. Note that ${\cal L}$ is separable since it is the image of a compact space via a continuous function. By definition of universality for the kernel $k$, the set ${\cal L}$ is dense in $C(W)$. Additionally, by Corollary $5$ in \citep{micchelli2006universal} we have $C(W) = \overline{\Span\{u_i~|~i \in \N\}}$. Thus, since $C(W)$ is dense in $\LX$, we can conclude that $(u_i)_{i \in \N}$ is a basis of $\LX$.

We now focus on $\la \|L_\la^{-1} Z\|_\HS$. In particular, we want to express  $\|L_\la^{-1} Z\|_\HS^2$ in terms of the basis $(u_i)_{i\in\N}$ associated to $L$. In particular, since $(u_i)_{i\in\N}$ is a basis for $\LX$ and $(L + \la I)^{-1} u_i = (\sigma_i + \la)^{-1} u_i$ for any $i\in\N$, we have
\eqals{
\la_n^2\|(L + \la_n)^{-1} Z\|_\HS^2 = \la_n^2\tr((L + \la_n)^{-1} ZZ^*(L + \la_n)^{-1}) = \la_n^2 \sum_{i \in \N} \frac{\scal{u_i}{ZZ^* u_i}_{\Ltwo}}{(\sigma_i + \la)^2}.
}
Now let $t_n = n^{-1/4}$, and $T_n = \{i \in \N \ | \ \sigma_i \geq t_n\} \subset \N$. Denote by $w_i^2 := \scal{u_i}{ZZ^* u_i}_{\Ltwo}$ and note that $(w_i)_{i \in \N}$ is square summable and $\sum_{i \in \N} w_i^2 = \|Z\|_\HS^2 < \infty$. For any $n\in\N$, we have
\eqals{
\la_n^2 \| (L + \la_n)^{-1} Z\|^2_\HS &= \sum_{i\in T_n} \frac{\la_n^2 w_i^2}{(\sigma_i + \la_n)^2} + \sum_{i \in \N\setminus T_n} \frac{\la_n^2 w_i^2}{(\sigma_i + \la_n)^2} \\
& \leq \frac{\la_n^2}{t_n^2}\sum_{i \in T_n} w_i^2 + \sum_{i \in \N\setminus T_n} w_i^2  \leq \|Z\|_\HS^2 ~ n^{-1/4} + \sum_{i\in\N\setminus T_n} w_i^2
}
since $\la_n/t_n = n^{-1/2}/n^{-1/4} =  n^{-1/4}$. Since the series $\sum_{i\in\N} w_i^2$ is convergent, we have $\sum_{i\in\N\setminus T_n} w_i^2 \to 0$ as $n \to +\infty$. We conclude that
\eqals{
0 \leq \lim_{n \to \infty} \la_n^2 \| (L + \la_n)^{-1} Z\|_\HS^2 \leq \lim_{n \to \infty}  \|Z\|_\HS^2 n^{-1/4} + \sum_{i \in \N\setminus T_n} w_i^2 = 0.
}
Now, let $\delta_n = n^{-2}$ and $A_n = E_{n, \delta_n}^c$ be the complementary event to $E_{n, \delta_n}$ characterized by \cref{eq:event}. For any $n \geq n_0 := 200(1+\kappa^2)$ we have $\la_n \geq \frac{9\kappa^2}{n} \log n^3$ and the event $E_{n, \delta_n}$ holds with probability at least $1-\delta_n$. Equivalently, the probability of $A_n$ is upper bounded by $\delta_n$.
Since $\sum_{n=n_0+1}^{+\infty} \delta_n < +\infty$, we can apply the Borel-Cantelli lemma (Theorem 8.3.4. pag 263 of \cite{dudley2002real}) on the sequence $(E_{n,\delta_n})_{n \in \N}$ and conclude that the statement
\eqals{
\lim_{n \to \infty} {\cal R}(\gn_{\la_n}) - {\cal R}(g^*) > 0,
}
holds with probability $0$. Thus, the converse statement
\eqals{
\lim_{n \to \infty} {\cal R}(\gn_{\la_n}) - {\cal R}(g^*) = 0.
}
holds with probability $1$.
The final result is obtained by applying the comparison inequality between the surrogate problem and the original excess risk from \cref{prop:comparison-inequality}.
\end{proof}

\subsection{Learning Rates}

In this section we study the generalization properties of the proposed estimators. We address this question by considering the special case where the solution $\gstar$ of the expected surrogate risk belongs to the same hypotheses space $\hh\otimes\ff$ where our estimator belongs to. We start this analysis by recalling that in this case, $\gstar$ admits a closed form solution. 

\begin{lemma}\label{lemma:expected_risk_solution}
Let $\loss:\Z\times\Y\to\R$ satisfy \cref{def:self} for suitable $\zmap,\ymap$ and $\hh$. Assume that the surrogate expected risk minimization of ${\cal R}$ at \cref{eq:surrogate-risk} attains a minimum on $\hh\otimes\ff$. Then the minimizer $g^* \in \hh\otimes\ff$ of $\mathcal{R}$ with minimal norm $\|\cdot\|_{\hh\otimes\ff}$ is of the form
\begin{equation}
    g^*(x) = G_*\xmap(x),\quad \forall x \in \X \quad \mbox{with} \quad G_* = C^\dagger S^*Z : \ff \to \hh.
\end{equation}
\end{lemma}

\begin{proof} 
Let $g\in\hh\otimes\ff$ such that $g(x) = G\xmap(x), \forall x \in \X$ for some linear operator $G\in\hh\otimes\ff$. We have 
\eqal{
    \mathcal{R}(g) & = \int_{\X\times\Y} \|G\xmap(x) - \psi(y)\|_\HY^2 d\rho(x,y) \\ 
  & = \int_{\X\times\Y} \tr(G(\xmap(x)\ot\xmap(x))G^*) - 2 \tr(G(\xmap(x)\ot\psi(y))) + \|\psi(y)\|_\HY^2 ~d\rhox(x,y)  \\
  & = \tr(GCG^*) - 2\tr(GS^*Z) + const,
}
where we have used \cref{prop:basic_operator_result} and the linearity of the trace. The derivation above implies that $\mathcal{R}$ is a convex quadratic functional since $C$ is positive semidefinite. Hence, $\mathcal{R}$ attains a minimum on $\hh\otimes\ff$ if and only if the range of $S^*Z$ is contained in the range of $C$, namely $\ran(S^*Z) \subseteq \ran(C) \subset \HX$ (see \cite{engl1996regularization} Chap. 2). In this case $G = C^\dagger S^*Z: \ff \to \hh$ exists and is the minimum norm minimizer for $\mathcal{R}$, as desired.
\end{proof}

\noindent We recall here the two main assumptions required in the following.

\ASource*

\ACapacity*

\noindent We are ready to prove the main result characterizing the learning rates of the proposed estimators. 


\begin{theorem}\label{thm:main-excess-risk-bound}
Let $\hh$ be a Hilbert space and let $\ymap:\Y\to\hh$ a continuous map from $\Y$ to $\hh$. Let $k:\X\times\X\to\R$ be a continuous reproducing kernel on $\X$ with associated RKHS $\ff$ such that $\kappa^2 := \sup_{x\in\X}k(x,x)<+\infty$. Let $\rho$ be a distribution on $\X\times\Y$ and let the corresponding $\gstar$ defined in \Cref{eq:fstar-in-terms-of-gstar-full} satisfy \cref{asm:source}. Let also \cref{asm:capacity} hold. Let $\delta\in(0,1]$ and $n_0$ sufficiently large such that $n_0^{-1/(1+2r+\gamma)} \geq \frac{9\kappa^2}{n_0} \log \frac{n_0}{\delta}$. Let $n\in\N$, $n\geq n_0$ and $\la\geq\frac{9\kappa^2}{n}\log \frac{n}{\delta}$. Let $\gn:\X\to\hh$ be a spectral filtering estimator of the form introduced in \cref{lm:ghatla_def} trained on $n$ points randomly sampled from $\rho$. Then, the following holds with probability at least $1-\delta$,
\eqals{
|{\cal R}(\gn) - {\cal R}(\gstar)|^{1/2} & \leq \msf M~n^{-\frac{r+1/2}{2r + 1 + \gamma}}
}
Where the constant $\msf M$ is 
\eqals{
\msf M = \msf M(Q,\nor{H},\delta,q_1,q_2) = & 8 q_1 \left[\kappa (1 + \kappa \|H\|_\HS) + \sqrt{(Q + \kappa^2 \|H\|_\HS^2)}\right] \log\frac{4}{\delta} \\
& + 2(2 + q_1 +2q_2) \|H\|_\HS.
}


\end{theorem}

\begin{proof}
According to the excess risk bound in \cref{thm:surrogate-excess-risk-bound-with-lambda}
we have that the following holds with probability at least $1-\delta$
\eqals{
|{\cal R}(\gn_\la) - {\cal R}(\gstar)|^{1/2}& \leq \frac{8 q_1\kappa \log \frac{4}{\delta}}{\sqrt{\la} n} (1 + \kappa \|L_\la^{-1/2} Z\|) \\
{} &  \quad + \quad \sqrt{\frac{64 ~ q_1^2 (\deff(\la) + \kappa^2\la\|L_\la^{-1}Z\|_\HS^2)\log \frac{4}{\delta}}{n}} \\
{} & \quad + \quad 2(2 + q_1 +2q_2) ~\la \|L_\la^{-1} Z\|_\HS.
}
From \cref{asm:source},  we have $g^*(x) = (C^r \otimes I) ~h ~\xmap(x) = H^* C^r \xmap(x)$ for any $x\in\X$, where $H:\ff\to\hh$ is the Hilbert-Schmidt operator corresponding to $h$ under the canonical isomorphism between $\HS(\ff,\hh)$ and $\hh\otimes\ff$. In particular, $\nor{H}_\HS = \nor{h}_{\hh\otimes\ff}$. By \cref{asm:source} $\nor{h}_{\hh\otimes\ff}= R$. Therefore, we can characterize $Z = SC^r H$. Indeed, recall that for any $w\in\hh$, by definition of $Z$, we have
\eqal{
(Zw)(\cdot) = \scal{w}{g^*(\cdot)}_\hh = \scal{w}{H^* C^r \xmap(\cdot)}_\hh = \scal{C^r H w}{\xmap(\cdot)}_\hh = (SC^r H w)(\cdot)
}
Then, denote by ${(1/2-r)_+} = \max(0,1/2-r)$, we have
\eqals{
\|L_\la^{-1}Z\|_\HS = \|L_\la^{-1}SC^r\|\|H\|_\HS \leq \la^{-(1/2-r)_+} R,
}
and, analogously, 
\eqals{
\|L_\la^{-1/2}Z\|_\HS \leq R.
}
Then, let $\la = n^{-1/(1+2r+\gamma)}$ and $n \geq n_0$ with $n_0$ such that $n_0^{-1/(1+2r+\gamma)} \geq \frac{9\kappa^2}{n_0} \log \frac{n_0}{\delta}$. From \cref{asm:capacity}, we have $\deff(\la) \leq Q \la^{-\gamma}$. Therefore,
\eqals{
|{\cal R}(\gn_\la) - {\cal R}(\gstar) & \leq 8 q_1 \left[\kappa (1 + \kappa R) + \sqrt{(Q + \kappa^2 R^2)}\right] n^{-\frac{r+1/2}{2r + 1 + \gamma}} \log\frac{4}{\delta} \\
{} & \quad + \quad 2(2 + q_1 +2q_2) R ~ n^{-\frac{r+1/2}{2r + 1 + \gamma}},
}
with probability at least $1-\delta$.
\end{proof}

\noindent \cref{thm:rates-refined} from the main paper is a direct consequence the result above when considering specific spectral filters. 

\TRatesRefined*

\begin{proof}
Let $\gn:\X\to\hh$ be an estimator satisfying the hypotheses of \cref{thm:main-excess-risk-bound} and let $\fn:\X\to\Z$ be such that for any $x\in\X$, $\fn(x) = \argmin_{z\in\Z} \scal{\zmap(z)}{\gn(x)}_\hh$. By applying the comparison inequality from \cref{prop:comparison-inequality}, we have that
\eqals{
    {\cal E}(\fn) - {\cal E}(\fstar) \leq 2 \closs |{\cal R}(\gn) - {\cal R}(\gstar)|^{1/2} \leq 2\closs \msf{M} ~n^{-\frac{r+1/2}{2r + 1 + \gamma}},
}
holds with probability at least $1-\delta$. Now, note that $\log(4/\delta)>1$ since $\delta\leq 1$, we have that the constant $\msf M$ is upper bounded by
\eqals{
\msf M \leq \left(8 q_1 \left[\kappa (1 + \kappa \|H\|_\HS) + \sqrt{(Q + \kappa^2 \|H\|_\HS^2)}\right]
+ 2(2 + q_1 +2q_2) \|H\|_\HS\right) \log\frac{4}{\delta}.
}
Replacing the quantities $q_1$ and $q_2$ from \cref{lm:spectral-filter-examples} associated to the corresponding estimators we obtain the required upper bounds for $\msf{m}$ stated in the thesis of the theorem.

\end{proof}

\noindent We note that \cref{thm:rates} is a corollary of \cref{thm:rates-refined} as shown below. 

\TRates*
\begin{proof}
The result is a corollary of the theorem above, by considering that \cref{asm:capacity} is always satisfied with $Q = \kappa^2$ and $\gamma = 1$ and that when $g^*$ is in ${\cal G}$, then \cref{asm:source} is satisfied with $r = 0$ and $h = g$.
\end{proof}

\section{Sufficient Conditions for \structshort{}}\label{sec:app-self-sufficient-conditions}

In this section we provide more details related to the \structshort{} definition introduced in \cref{def:self}. In particular we discuss the connection with the original framework considered in \citep{ciliberto2016} and prove the results reported in \cref{sec:self-sufficient-conditions} providing sufficient conditions to determine whether a function admits an \structshort{}. 

\subsection{Relations with the ``\structshort{}'' definition in \citep{ciliberto2016}} In \cite{ciliberto2016}, the \structshort{} definition was introduced as the following assumption in the case $\Z=\Y$ \cite[see Assumption $1$ in][]{ciliberto2016}.

\begin{assumption}\label{def:self-old}
There exists a separable Hilbert space $\hh$ with inner product $\scal{\cdot}{\cdot}_\hh$, a continuous embedding $\zeta:\Y\to\hh$ and a bounded linear operator such that 
\eqals{
    \loss(z,y) = \scal{\zeta(z)}{V\zeta(y)}_\hh \qquad \forall z,y\in\Y
}
\end{assumption}
It can be noticed that the two definitions are quite similar one to the other. Both require the existence of a separable Hilbert space $\hh$ where the function $\loss$ assumes a ``bilinear'' structure. However the definition above requires $\Z = \Y$ and the existence of a linear operator combining a single feature map $\zeta$.

Despite these differences, the following result shows that the above assumption is equivalent to the \structshort{} definition (\cref{def:self}) in the main paper. 

\begin{proposition}[Equivalence of \structshort{} Definitions]\label{prop:equivalence-self-old}
A loss $\loss:\Y\times\Y\to\R$ admits an \structshort{} if and only if it satisfies \cref{def:self-old}.
\end{proposition}

\begin{proof}
($\Rightarrow$). Let $\loss$ satisfy the \structshort{} definition with Hilbert space $\hh$ and feature maps $\zmap,\ymap:\Y\to\hh$. We define $\overline\hh = \hh\oplus\hh$ and consider the map $\zeta:\Y\to\overline\hh$ such that $\zeta(y) = (\zmap(y),\ymap(y))^\top\in\overline\hh$ for any $y\in\Y$. Moreover, we define $V:\overline\hh\to\overline\hh$ the linear operator such that $V(h_1,h_2) = (h_2,0)$ for any $h = (h_1,h_2) \in\overline\hh$. It is easy to see that $V$ is bounded, and actually it has operator norm $\|V\|=1$. Therefore, we have
\eqals{
    \scal{\zeta(z)}{V \zeta(y)}_{\overline\hh} = \scal{(\zmap(z),\ymap(z))}{(\ymap(y),0)}_{\overline\hh} = \scal{\zmap(z)}{\ymap(y)}_\hh + \scal{\ymap(z)}{0}_\hh = \loss(z,y)
}
for any $y,z\in\Y$ as desired. Hence $\loss$ satisfies \cref{def:self-old} with associated Hilbert space $\overline\hh$. Note that $\zeta$ is continuous since both $\zmap$ and $\ymap$ are continuous by the \structshort{} definiton and $V$ is linear and bounded by construction.\\

\noindent ($\Leftarrow$). Let $\loss$ satisfy \cref{def:self-old} with Hilbert space $\hh$, feature map $\zeta:\Y\to\hh$ and linear operator $V:\hh\to\hh$. Let $\Phi = \sup_{y\in\Y}\|\zeta(y)\|_\hh$. Then we take $\zmap,\ymap:\Y\to\hh$ the functions such that $\zmap(z) = \Phi~V\ymap(z)$ and $\ymap(y) = \zeta(y)/\Phi$ for any $z,y\in\Y$. Clearly, we have 
\eqals{
    \scal{\zmap(z)}{\ymap(y)}_\hh = \scal{\zeta(z)}{V\zeta(y)}_\hh = \loss(z,y).
}
By construction, we have $\sup_{y\in\Y}\|\ymap(y)\|\leq 1$ and $\closs = \sup_{z\in\Z} \|\zmap(z)\|_\hh \leq \Phi^2\|V\|$. 
\end{proof}
\subsection{\structshort{} definition without ``normalization'' of $\ymap$} 

We point out that the requirement for $\sup_{y\in\Y}\|\ymap(y)\|_\hh\leq1$ is not necessary but was introduced for the sake of exposition. We formalize this in the following.  

\begin{lemma}[``Unnormalized'' \structshort{}]\label{lem:self-equivalence-not-normalized}
Let $\loss:\Z\times\Y\to\R$ be such that there exists a separable Hilbert space $\hh$ and two continuous bounded maps $\bar\zmap:\Z\to\hh$ and $\bar\ymap:\Y\to\hh$, such that $\sup_{z\in\Z}\|\bar\zmap(z)\|_\hh\leq \boldsymbol\zmap_\loss$ and $\sup_{y\in\Y}\|\bar\ymap(y)\|_\hh \leq \boldsymbol\Phi_\loss$, with $\boldsymbol\zmap_\loss>0$, $\boldsymbol\Phi_\loss>0$ and 
\eqals{
    \loss(z,y) = \scal{\bar\zmap(z)}{\bar\ymap(y)}_\hh,
    }
for every $z\in\Z$ and $y\in\Y$. Then $\loss$ admits an \structshort{} with $\closs \leq \boldsymbol\zmap_\loss\boldsymbol\Phi_\loss$.
\end{lemma}

\begin{proof}
The result is easy to prove by taking $\zmap:\Z\to\hh$ and $\ymap:\Y\to\hh$ such that $\zmap(z) = \sup_{y'\in\Y}\boldsymbol\Phi_\loss~\bar\zmap(z)$ for any $z\in\Z$ and $\ymap(y) = \bar\ymap(y)/\boldsymbol\Phi_\loss$ for any $y\in\Y$. Indeed it is straightforward to see that the characterization of $\loss$ in terms of the inner product between $\zmap$ and $\ymap$ still holds and, by construction, $\sup_{y\in\Y}\|\ymap(y)\|_\hh\leq 1$ and $\closs = \sup_{z\in\Z}\|\zmap(z)\|_\hh \leq \boldsymbol\zmap_\loss\boldsymbol\Phi_\loss$. 
\end{proof}

\subsection{Finite $\Y$ or $\Z$}

We now focus on proving the sufficient conditions to guarantee $\loss$ to admit an \structshort{}. We begin from the case where either the label set $\Y$ or the output set $\Z$ are finite. 

\TFiniteYorZ*

\begin{proof}
(a). The proof of point (a) has been already given in the discussion after \cref{def:self}. We recall it here for completeness. By hypothesis we have $\Z = \{z_1,\dots,z_p\}$ and $\Y = \{y_1,\dots,y_q\}$ for some $p,q\in\N$. Let $V\in\R^{p \times q}$ be the matrix whose entries correspond to the values of $\loss$ on pairs of points in $\Z\times\Y$. More precisely 
\eqals{
	V_{ij} = \loss(z_i,y_j) \qquad \forall i =1,\dots,p, \quad j=1,\dots,q. 
}
It is easy to prove that the \structshort{} definition holds for $\hh = \R^q$ and feature maps $\zmap:\Z\to\hh$ and $\ymap:\Y\to\hh$ corresponding to 
\eqals{
	\zmap(z_i) = V^\top e_i^{(p)} \qquad \ymap(y_j) = e_j^{(q)}
}
for any $i=1,\dots,p$ and any $j=1,\dots,q$, with $e_i^{(p)}\in\R^p$ denoting the $i$-th element of the canonical basis of $\R^p$, namely the $p$-dimensional vector with $i$-th entry equal to $1$ and all others equal to $0$. Indeed, by construction we have 
\eqals{
	\scal{\zmap(z_i)}{\ymap(y_j)}_\hh = \scal{e_i^{(p)}}{V e_j^{(q)}} = V_{ij} = \loss(z_i,y_j).
}
Finally, we have 
\eqals{
	\closs = \sup_{z\in\Z} \|\zmap(z)\|_\hh = \sup_{i=1,\dots,p} \|V e_i^{(p)}\| \leq \|V\|,
}
as required.

\noindent (b). $\Z = \{z_1,\dots,z_p\}$ with $p\in\N$. We choose $\hh = \R^p$ and the feature maps $\bar\zmap:\Z\to\hh$ and $\bar\ymap:\Y\to\hh$ such that $\zmap(z_i) = e_i^{(p)}$ for every $i=1,\dots,p$ and 
\eqals{
	\ymap(y) = \left(\loss(z_1,y),\dots,\loss(z_p,y)\right)^\top \in \R^{p},
}
for any $y\in\Y$. Now, let
\eqals{
	\msf r = \sup_{y\in\Y} \|\ymap(y)\|_\hh = \sup_{y\in\Y} \sqrt{\sum_{z\in\Z} \loss(z,y)^2}
}
we can define $\zmap = \msf r \bar\zmap$ and $\ymap/\msf r$. We have that the \structshort{} definition is satisfied, since 
\eqals{
	\scal{\zmap(z_i)}{\ymap(y)}_\hh = \scal{e_i^{(p)}}{\ymap(y)} = \loss(z_i,y),
}
for every $i=1,\dots,p$. Moreover, since $\|\zmap(z)\| = 1$ for every $z\in\Z$, we conclude that $\closs = r$ as required.\\

\noindent (c). The proof of point (c) is analogous to (b) with the difference that for $\Y = \{y_1,\dots,y_q\}$ with $q\in\N$ we choose $\hh = \R^q$, and feature maps $\ymap(y_j) = e_{j}^{(q)}$ for any $j=1,\dots,q$ and 
\eqals{
	\zmap(z) = \left(\loss(z,y_1),\dots,\loss(z,y_q)\right)^\top \in \R^{q},
} 
for any $z\in\Z$.
The proof follows identically to (b).
\end{proof}

\subsection{\structshort{} and Reproducing Kernel Hilbert Spaces}
We now focus on the relation between \structshort{} and reproducing kernel Hilbert spaces.

\TRkhsAndStruct*

\begin{proof}
(a). Follows directly from point (c). Indeed $h(z,\cdot)\in\hh$ and $\sup_{z\in\Y} \|h(z,\cdot)\| \leq \eta$ by hypothesis.\\

\noindent (b). We note that point (b) follows directly from \cref{cor:sum-and-products-of-self}, which guarantees the finite sums and products of \structshort{} functions to be \structshort{} as well. Here we give a more direct proof for completeness.

Let $\overline{\hh} = \R \oplus \hh \oplus \R$ equipped with the canonical inner product of the direct sum and let $V:\overline{\hh}\to\overline\hh$ the linear operator such that
\eqals{
	V~(\alpha,h,\beta) = (\beta,-2h,\alpha)
}
for any $h\in\hh$ and $\alpha,\beta\in\R$. Let $\zeta:\Y\to\overline\hh$ be such that 
\eqals{
	\zeta(y) = (h(y,y),h(y,\cdot),1)^\top,
}
for any $y\in\Y$. Let 
\eqals{
r = \sup_{y\in\Y} \|\zeta(y)\|_{\overline\hh} = \sup_{y\in\Y} \sqrt{h(y,y)^2 + \|h(y,\cdot)\|^2 + 1} \leq \sqrt{2\eta^4 + 1}.
}
We choose $\zmap,\ymap:\Y\to\overline\hh$ as 
\eqals{
	\zmap(z) = r~ V\zeta(z)  \qquad \ymap(y) = \frac{1}{r}~\zeta(y),
}
for any $z,y\in\Y$. Then, by construction
\eqals{
	\scal{\zmap(z)}{\ymap(y)}_{\overline\hh} & = \scal{(1,-2h(z,\cdot),h(z,z))^\top}{(h(y,y),h(y,\cdot),1)^\top}_{\overline\hh} \\
	& = h(y,y) -2 h(z,y) + h(z,z) \\
	& = \loss(z,y),  
}
as desired. Moreover, 
\eqals{
	\closs = \sup_{z\in\Z} \zmap(z) = r~\sup_{z\in\Z} \|V\zeta(z)\|\leq r^2\|V\| \leq 2 r^2 = 2(2\eta^4 +1),
}
as desired, with $\|V\| = 2$ denoting the operator norm of $V$. \\

\noindent (c). We prove the statement for the case $\sup_{z\in\Z} \|\loss(z,\cdot)\| = \msf{D} < +\infty$. We consider the feature maps $\zmap,\ymap:\Y\to\hh$ such that 
\eqals{
	\zmap(z) = \eta~ \loss(z,\cdot) \qquad \textrm{and} \qquad \ymap(y) = \frac{1}{\eta} h(y,\cdot),
}
for any $z,y\in\Y$. Then, by construction we have
\eqals{
	\scal{\zmap(z)}{\ymap(y)}_\hh = \scal{\loss(z,\cdot)}{h(y,\cdot)}_\hh = \loss(z,y),
}
where the last inequality follows from the reproducing property of the kernel $h$ and the fact that $\loss(z,\cdot)\in\hh$. Moreover we have $\closs = \sup_{z\in\Z} \|\zmap(z)\|_\hh = \eta~\sup_{z\in\Z}\|\loss(z,\cdot)\| = \eta\msf{D}$. The case $\sup_{y\in\Y}\|\loss(\cdot,y)\| \leq \msf{D}$ follows from an analogous reasoning. 
\\

\noindent (d). Note that the kernel $\bar h$ has feature map $(z,y)\mapsto h(z,\cdot)\otimes h(y,\cdot)$ for any $z,y\in\Y$. Since by hypothesis $\loss\in\hh\otimes\hh$, the reproducing property for $\bar h$ implies
\eqals{
	\loss(z,y) = \scal{\loss}{h(z,\cdot)\otimes h(y,\cdot)}.
}
Since $\hh\otimes\hh$ is isometric to the space of Hilbert-Schmidt operators from $\hh$ to $\hh$, there exists an operator $V:\hh\to\hh$ such that $\|V\|_\HS = \|\loss\|_{\hh\otimes\hh}$ and 
\eqals{
	\scal{\loss}{h(z,\cdot)\otimes h(y,\cdot)} = \scal{V}{h(z,\cdot)\otimes h(y,\cdot)}_\HS = \scal{V h(z,\cdot)}{h(y,\cdot)}_\hh,
}
where the last inequality follows from the standard properties of tensor products. We can therefore choose $\zmap,\ymap:\Y\to\hh$ such that 
\eqals{
	\zmap(z) = \eta~V h(z,\cdot) \qquad \textrm{and} \qquad \ymap(y) = \frac{1}{\eta} h(y,\cdot),
}
for any $y\in\Y$ to guarantee the \structshort{} definition to hold. Moreover, by construction $\closs = \sup_{z\in\Z} \|\zmap(z)\|_\hh \leq \eta\|V\|_\HS~\sup_{z\in\Z}\|h(z,\cdot)\|_\hh \leq \eta^2\|V\|_\HS$ which concludes the proof since $\|V\|_\HS = \|\loss\|_{\hh\otimes\hh}$ by construction.  
\end{proof}
We now report the result relating general notions of regularity for the loss and the \structshort{} definiton. Before proving the main result in \cref{prop:self-and-smoothness} we need the following two Lemmas. 

\begin{lemma}[Multiple Fourier Series]\label{lemma:fseries}
Let $\Y = [-B,B]^d$ with $d\in\N$ and $B>0$. Let $(\widehat  f_h)_{h \in \mathbb Z^d} \in \CC$ and $f: \Y  \to \CC$ defined as
\eqals{
f(y) = \sum_{h \in \mathbb Z^d} \widehat f_h e^{2\pi i h^\top y}, \quad \forall y \in \Y, \quad \textrm{with}\quad \sum_{h \in \mathbb Z^d} |\widehat f_h| \leq M,
}
for $0 < M < \infty$ and $i = \sqrt{-1}$. Then the function $f$ is continuous and
\eqals{
\sup_{y \in \Y} |f(y)| \leq M.
}
\end{lemma}
\begin{proof}
The continuity of $f$ follows from \citep{kahane1995fourier}, pag. $129$ and Example $2$. To show that $f$ is uniformly bounded on $\Y$ it is sufficent to see that 
\eqals{
\sup_{y \in \Y} |f(y)| \leq \sup_{y \in \Y} \sum_{h \in \mathbb Z^d} |\widehat f_h| |e^{2\pi i h^\top y}| \leq \sum_{h \in \mathbb Z^d} |\widehat f_h| \leq M.
}
\end{proof}

\begin{lemma}\label{lemma:abs-cont-functions}
Let $\Y = [-B,B]^d$ with $d \in \N$ and $B>0$. Let $\loss: \Y \times \Y \to \R$ be such that
\eqals{
\loss(y, z) = \sum_{h,k \in \mathbb Z^d} \widehat\loss_{h,k} ~e_h(y)~e_k(z), \quad \forall y, z \in \Y,
}
with $e_h(y) = e^{2\pi i h^\top y}$ for any $y \in \Y$, $i = \sqrt{-1}$ and $\widehat\loss_{h,k} \in \CC$ for any $h,k \in \mathbb Z^d$. If
\eqals{
\closs = \sum_{h,k \in \mathbb Z^d} |\widehat\loss_{h,k}| < \infty,
}
then $\loss$ admits an \structshort{}. 
\end{lemma}
\begin{proof}
We start by applying \cref{lemma:fseries} for the input domain $\Y\times\Y$, which guarantees that the function $\loss$ is bounded continuous. We introduce the following sequences
\eqals{
\alpha_{h} = \sum_{k \in \mathbb Z^d} |\widehat\loss_{h,k}|,  \qquad f_h(z) = \frac{1}{\alpha_h} \sum_{h \in \mathbb Z^d} \widehat\loss_{hk} e_k(z) \qquad \forall h \in \mathbb Z^d, z \in \Y.
}
For any $p>0$, let $\ell_p(\mathbb Z^d)$ denote the set of sequences $(a_k)_{k\in\mathbb Z^d}$ such that $\sum_{h\in\mathbb Z^d} |a_k|^p < +\infty$. Note that by hypothesis $(\alpha_h)_{h \in \mathbb Z^d}\in\ell_1(\mathbb Z^d)$. Moreover, by applying again \cref{lemma:fseries}, we have that the functions $f_h$ are continuous and bounded by $1$ for any $h \in \mathbb Z^d$. 

Denote $A^2 = \|(\alpha_h)_{h \in \mathbb Z^d}\|_{\ell_1(\mathbb Z^d)} =  \sum_{h,k \in \mathbb Z^d} |\widehat\loss_{h,k}|$. Let $\hh = \ell_2(\mathbb Z^d)$ and let $\zmap,\ymap: \Y \to \H$ be such that 
\eqals{
    \zmap(z) = A~\Big(\sqrt{\alpha_h}f_h(z)\Big)_{h \in \mathbb Z^d} \qquad \textrm{and} \qquad \ymap(y) = \frac{1}{A} ~\Big(\sqrt{\alpha_h}e_h(y)\Big)_{h \in \mathbb Z^d},
}
for all $z,y\in\Y$. By construction, we have 
\eqal{
 \scal{\zmap(z)}{\ymap(y)}_\hh = \sum_{h \in \mathbb Z^d} \alpha_h ~e_h(y)~f_h(z) = \sum_{h,k \in \mathbb Z^d} \widehat\loss_{h,k} ~e_h(y)~e_h(z) = \loss(y, z).
}
Moreover, by construction we also have $\|\ymap(y)\|_\hh \leq 1$ and $\closs = \sup_{z\in\Z} \|\zmap(z)\| \leq A^2$. 

To conclude the proof, we need to show that the two feature maps are continuous. Define $\zeta_1(y,z) = \scal{\ymap(y)}{\ymap(z)}_{\H}$ and $\zeta_2(y,z) = \scal{\zmap(y)}{\zmap(z)}_{\H_0}$  for all $y,z \in \Y$. We have 
\eqals{
\zeta_1(y,z) &= \sum_{h \in \mathbb Z^d} \alpha_h \overline{e_h(y)}e_h(z),\\
\zeta_2(y,z) &= \sum_{h \in \mathbb Z^d} \alpha_h \overline{f_h(y)} f_h(z) = \sum_{k,l \in \mathbb Z^d} \beta_{k,l} \overline{e_k(y)} e_l(z)
}
with $\beta_{k,l} = \sum_{h \in \mathbb Z^d} \frac{\overline{\widehat\loss_{h,k}}\widehat\loss_{h,l}}{\alpha_h}$, for $k, l \in \mathbb Z^d$, therefore $\zeta_1$, $\zeta_2$ are bounded and continuous by \cref{lemma:fseries}, since $\sum_{h \in \mathbb Z^d} \alpha_h < \infty$ and $\sum_{k,l \in \mathbb Z^d} |\beta_{k,l}| < \infty$.
Note that $\zmap$ and $\ymap$ are bounded, since $\zeta_1$ and $\zeta_2$ are. Moreover for any $y, z \in \Y$, we have
\eqals{
\|\ymap(y) - \ymap(z)\|_\hh^2 &= \scal{\ymap(z)}{\ymap(z)}_{\H} + \scal{\ymap(y)}{\ymap(y)}_{\H} - 2 \scal{\ymap(z)}{\ymap(y)}_{\H} \\
& = \zeta_1(z,z) + \zeta_1(y,y) - 2 \zeta_1(z,y) \\
& \leq |\zeta_1(z,z) - \zeta_1(z,y)| + |\zeta_1(z,y) - \zeta_1(y,y)|,
}
and the same holds for $\zmap$ with respect to $\zeta_2$.
Thus the continuity of $\ymap$ is ensured by the continuity of $\zeta_1$ and the same for $\zmap$ with respect to $\zeta_2$.
\end{proof}
We are ready to prove the following result. 
\PSmoothness*
\begin{proof}
(a-b). Either hypotheses in $(a)$ or $(b)$ are sufficient to guarantee that the Fourier expansion of $\loss$ is absolutely summable \citep[see Theorem~$5'$ and Theorem~$6'$ pag.~$291$ of][]{moricz2007absolute}. By \cref{lemma:abs-cont-functions} we can conclude that $\loss$ admits an \structshort{}.

(c). Let $\gamma^2 = \int_{-\infty}^{+\infty} |\widehat v(\omega)| < +\infty$ we have that for any $z,y\in\Y$, the anti-Fouier transform of $\hat v(\omega)$ in $z-y$ is
\eqals{
    v(z-y) = \int_{-\infty}^{+\infty}\widehat v(\omega) e^{i \scal{\omega}{z-y} } ~d\omega 
    = \int_{-\infty}^{+\infty} \widehat v(\omega) e^{i \scal{\omega}{z}} e^{-i \scal{\omega}{y}} ~d\omega. 
}
Let now $\hh = \Ltwo(\R,\mathbb{C})$ the space of square integrable functions from $\R$ to $\mathbb{C}$ with respect to the Lebesgue measure and let $\zmap,\ymap:\Y\to\hh$ be such that 
\eqals{
    \zmap(z) = \gamma ~ \sqrt{\widehat{v}(\cdot)} e^{i\scal{\cdot}{z}} \qquad \textrm{and} \qquad \ymap(y) = \frac{1}{\gamma} ~ \sqrt{\widehat{v}(\cdot)} e^{-i\scal{\cdot}{y}}.
}
By the anti-Fourier transform we have
\eqals{
    \scal{\zmap(z)}{\ymap(y)}_\hh = v(z-y) = \loss(z,y)
}
for every $z,y\in\Y$. Moreover, by construction $\|\ymap(y)\|_\hh = 1$ and $\closs = \sup_{z\in\Z}\|\zmap(z)\|_\hh = \gamma^2$ as required. 
\end{proof}

\subsection{Composition Rules for \structshort{}}

We conclude with the result characterizing composition rules for the \structshort{} property.
\TCompositions*
\begin{proof}
(a). Let $\overline\loss$ admit an \structshort{} with $\hh$ separable Hilbert space and feature maps $\overline\zmap:\overline\Z\to\hh$ and $\overline\ymap:\overline\Y\to\hh$. Clearly, for any $\Y\subseteq\overline\Y$, $\Z\subseteq\overline\Z$, we have that the restriction of the feature maps $\zmap = \overline\zmap|_\Z$ and $\ymap = \overline\ymap|_\Y$ are such that 
\eqals{
    \scal{\zmap(z)}{\ymap(y)}_\hh = \scal{\overline\zmap(z)}{\overline\ymap(y)}_\hh = \overline\loss(z,y) = \overline\loss|_{\overline\Z\times\overline\Y} = \loss(z,y),
}
for any $z\in\Z$ and $y\in\Y$. The proof is concluded by observing that $\sup_{y\in\Y}\|\overline\ymap(y)\|_\hh\leq\sup_{y\in\overline\Y}\|\overline\ymap(y)\|_\hh \leq 1$ and $\closs = \sup_{z\in\Z}\|\overline\zmap(z)\|_\hh\leq\sup_{z\in\overline\Z}\|\overline\zmap(y)\|_\hh = \msf{c}_{\overline\loss}$.\\

\noindent (b). Let $\overline\loss$ admit an \structshort{} with $\hh$ separable Hilbert space and feature maps $\overline\zmap:\overline\Z\to\hh$ and $\overline\ymap:\overline\Y\to\hh$.
Let $\bar\beta = \sup_{y\in\Y} |\beta(y)|$. We consider the feature maps $\zmap:\Z\to\hh$ and $\ymap:\Y\to\hh$ such that 
\eqals{
    \zmap(z) = \bar\beta ~\alpha(z) ~\overline\zmap(A(z)) \qquad \textrm{and} \qquad  \ymap(y) = \frac{\beta(z)}{\bar\beta} ~ \overline\zmap(B(z)),
}
for all $z\in\Z$ and $y\in\Y$. By construction we have 
\eqals{
    \scal{\zmap(z)}{\ymap(y)} = \alpha(z)\scal{\bar\zmap(A(z))}{\bar\ymap(B(y))}\beta(y) = \alpha(z)\overline\loss(A(z),B(y))\beta(y).
}
Moreover, we have $\sup_{y\in\Y} \|\ymap(y)\|_\hh =  \frac{1}{\bar\beta}\sup_{y\in\Y}\beta(y)\|\bar\ymap(y)\|_\hh\leq 1$ and $\closs = \sup_{z\in\Z} \|\zmap(z)\|_\hh = \bar\beta \sup_{z\in\Z} \alpha(z)\|\bar\zmap(z)\|_\hh \leq \bar\alpha\bar\beta \msf{c}_{\bar\loss}$ as required.\\

\noindent (c). By definiton of analytic functions, we have that $\Gamma$ has form
\eqals{
    \Gamma(r_1,\dots,r_P) = \sum_{t\in\N^P} \alpha_t \prod_{p=1}^P r_p^{t_p} \qquad \forall r = (r_1,\dots,r_P)^\top \in \R^P,
} 
for some scalar weights $\alpha_t$ with $t\in\N^P$. Therefore, for any $z=(z_1,\dots,z_P)\in\Z$ and $y=(y_1,\dots,y_P)\in\Y$, we have  
\eqals{
    \loss(z,y) = \Gamma\big(\loss_1(z_1,y_1),\dots,\loss_P(z_P,y_P)\big) = \sum_{t\in\N^P} \alpha_t \prod_{p=1}^P \loss_p(z_p,y_p)^{t_p}.
}
Recall that for any to Hilbert spaces $\hh$ and $\hh'$, by definiton of direct sum $\hh\oplus\hh'$ and tensor product $\hh\otimes\hh'$, we have 
\eqals{
    \scal{h_1\oplus h_1'}{h_2 \oplus h_2'}_{\hh\oplus\hh'} & = \scal{h_1}{h_2}_\hh + \scal{h_1'}{h_2'}_{\hh'}, \\
    \scal{h_1\otimes h_1'}{h_2 \otimes h_2'}_{\hh\otimes\hh'} & = \scal{h_1}{h_2}_\hh \cdot \scal{h_1'}{h_2'}_{\hh'} 
}
for any $h_1,h_2\in\hh$ and $h_1',h_2'\in\hh'$. Moreover, for any $p\in\N$, we denote $\hh^{\otimes p}$ the tensor product of $\hh$ with itself $p$ times (with $\hh^{\otimes 0} = \R$) and denote $h^{\otimes p}\in\hh^{\otimes p}$ the tensor product of $h$ with itself $p$ times, for any $h\in\hh$ (with $h^{\otimes 0} = 1$). This implies in particular that for any $t\in\N^P$ we have 
\eqals{
    \prod_{p=1}^P \loss_p(z_p,y_p)^{t_p} = \scal{\bigotimes_{p=1}^P \zmap_p(z_p)^{\otimes t_p}}{\bigotimes_{p=1}^P \ymap_p(y_p)^{\otimes t_p}}_{\overline\hh_t},
}
$\hh_t = \bigotimes_{p=1}^P \hh_p^{\otimes t_p}$ and we have denoted with $\zmap_p(z_p)^{\otimes t_p}$ the tensor product of $\zmap_p(z_p)$ with itself $t_p$ times.

For any $t\in\N^P$, let $\beta_t = \sign(\alpha_t) \sqrt{|\alpha_t|}$ and $\gamma_t = \sqrt{|\alpha_t|}$. Then, we have 
\eqals{
    \sum_{t\in\N^P}\alpha_t\prod_{p=1}^P \loss_p(z_p,y_p)^{t_p} & = \sum_{t\in\N^P} \alpha_t \scal{\bigotimes_{p=1}^P \zmap_p(z_p)^{\otimes t_p}}{\bigotimes_{p=1}^P \ymap_p(y_p)^{\otimes t_p}}_{\overline\hh_t} \\ 
    & = \scal{\bigoplus_{t\in\N^P} \beta_t\left[ \bigotimes_{p=1}^P \zmap_p(z_p)^{\otimes t_p}\right]}{\bigoplus_{t\in\N^P} \gamma_t\left[ \bigotimes_{p=1}^P \ymap_p(y_p)^{\otimes t_p}\right]}_{\hh} \\
    & = \scal{\zmap(z)}{\ymap(y)}_{\hh}.
}
where $\hh = \bigoplus_{t\in\N^P}\overline\hh_t$ and we $\zmap:\Z\to\hh$ and $\ymap:\Y\to\hh$ are feature maps such that 
\eqals{
    \zmap(z) = \bigoplus_{t\in\N^P} \beta_t\left[ \bigotimes_{p=1}^P \zmap_p(z_p)^{\otimes t_p}\right] \qquad \textrm{and} \qquad \ymap(y) = \bigoplus_{t\in\N^P} \gamma_t\left[ \bigotimes_{p=1}^P \ymap_p(y_p)^{\otimes t_p}\right],
}
for any $z=(z_1,\dots,z_P)\in\Z$ and $y=(y_1,\dots,y_P)\in\Y$. First note that such feature maps are well defined, namely that they indeed take value in $\hh$. In particular, we have 
\eqals{
    \nor{\ymap(y)}_\hh & = \nor{\bigoplus_{t\in\N^P} \gamma_t\left[ \bigotimes_{p=1}^P \ymap_p(y_p)^{\otimes t_p}\right]}_{\overline\hh}^2\\
    & = \sum_{t\in\N^P} |\alpha_t| \nor{\bigotimes_{p=1}^P \ymap_p(y_p)^{\otimes t_p}}_{\overline\hh_t}^2 \\ 
    & = \sum_{t\in\N^P} |\alpha_t| \prod_{p=1}^P \left(\|\ymap_p(y_p)\|_{\hh_p}^{2}\right)^{t_p}.
}
To show that the above series is finite, recall that the power series defining $\Gamma$ is absolutely convergent for any $r = (r_1,\dots,r_P)\in\R^P$. Indeed, let $\bar r = (\bar r_1,\dots,\bar r_P)\in\R^P$ such that $|r_p| < \bar |r_p|$ for any $p=1,\dots,P$. Since $\Gamma$ is analytic also in $\bar r$, the associated power series is convergent and therefore, for $\|t\|\to+\infty$ we have $a_t\prod_{p=1}^P \bar r_p^{t_p} \to 0$. This implies that there exists $Q>0$ such that, for any $t\in\N^P$ with $\|t\|>Q$,
\eqals{
    |\alpha_t|\prod_{p=1}^P|\bar r_p|^{t_p} \leq 1.
}
By multiplying both sides of the inequality above by $\prod_{p=1}^P|r_p/\bar r_p|^{t_p}$, we have 
\eqals{
    |\alpha_t|\prod_{p=1}^P|r_p|^{t_p} < \prod_{p=1}^P \left|\frac{r_p}{\bar r_p}\right|^{t_p}.
}
Since $|r_p/\bar r_p|<1$ for $p=1,\dots,P$ by construction, we can conclude that 
\eqals{
    \sum_{t\in\N^P}\left|\alpha_t\prod_{p=1}^Pr_p^{t_p}\right| < +\infty.
}
In particular, we have that the domain of the function $\overline \Gamma:\R^P\to\R$, such that
\eqals{
    \overline\Gamma(r_1,\dots,r_P) = \sum_{t\in\N^P} \left|\alpha_t \prod_{p=1}^P r_p^{t_p}\right| 
}
corresponds to $\R^P$, namely $\overline\Gamma(r_1,\dots,r_P)<+\infty$ for any $r=(r_1,\dots,r_P)^\top\in\R^P$. 

Therefore, using the fact that $\sup_{y_p\in\Y_p}\|\ymap_p(y_p)\|\leq 1$ for any $p=1,\dots,P$, we have 
\eqals{
    \nor{\ymap(y)}_\hh \leq \sqrt{\overline\Gamma(1,\dots,1)}<+\infty,
}
for any $y=(y_1,\dots,y_P)\in\Y$. By following an analogous reasoning for $\zmap$, we have
\eqals{
    \nor{\zmap(z)}_\hh \leq \sqrt{\overline\Gamma(\msf{c}_{\loss_1}^2,\dots,\msf{c}_{\loss_P}^2)} < +\infty,
}
for any $z=(z_1,\dots,z_P)\in\Z$.

We need to show that the maps $\zmap$ and $\ymap$ are continuous. To see this it is sufficient to prove that they are the uniform limit of continuous functions. In particular for any $Q\in\R$, let $\ymap^{(Q)}:\Y\to\hh$ be such that 
\eqals{
    \ymap^{(Q)}(y) = \left(\bigoplus_{\substack{t\in\N^P, \|t\|\leq Q} } \gamma_t\left[ \bigotimes_{p=1}^P \ymap_p(y_p)^{\otimes t_p}\right]\right) \oplus \left(\bigoplus_{\substack{t\in\N^P, \|t\|> Q} } 0\right),
}
for any $y=(y_1,\dots,y_P)\in\Y$. Note that $\ymap^{(Q)}(y)\in\hh$, since by construction $\hh = \left(\bigoplus_{\|t\|\leq Q}\hh_t\right) \oplus \left(\bigoplus_{\|t\|>Q} \hh_t\right)$. Moreover, $\ymap^{(Q)}:\Y\to\hh$ is continuous for any $Q\in\R$ since it is the direct sum of a finite number of continuous functions. 

Therefore, for any $y\in\Y$, we have
\eqals{
    \nor{\ymap(y) - \ymap^{(Q)}(y)}_\hh^2 
    & = \nor{\bigoplus_{\substack{t\in\N^P, \|t\|> Q} } \gamma_t\left[ \bigotimes_{p=1}^P \ymap_p(y_p)^{\otimes t_p}\right]}^2 \\
    & = \sum_{t\in\N^P, \|t\|>Q} |\alpha_t| \prod_{p=1}^P \left(\|\ymap_p(y_p)\|_{\hh_p}^{2}\right)^{t_p} \\
    & \leq \sum_{t\in\N^P, \|t\|>Q} |\alpha_t|,
}
where we have made use of the fact that $\sup_{y_p\in\Y_p}\|\ymap_p(y_p)\|_{\hh_p}\leq 1$ for any $p=1,\dots,P$. Since $\sum_{t\in\N^P, \|t\|\in\N^P} |\alpha_t|<+\infty$, we have that for $Q\to+\infty$, the residual $\sum_{t\in\N^P, \|t\|>Q} |\alpha_t|$ will tend to zero. We conclude that
\eqals{
    \lim_{Q\to+\infty} ~ \nor{\ymap(y) - \ymap^{(Q)}(y)}_\hh \to 0,
}
showing that $\ymap$ is uniform limit of continuous functions and hence continuous itself, as desired. The exact same argument holds for $\zmap:\Z\to\hh$. 

Clearly, $\ymap$ is not necessarily taking values in the ball of radius $1$ in $\hh$. To this end we can invoke \cref{lem:self-equivalence-not-normalized} by replacing $\ymap(y)$ with $\overline\ymap(y) = \ymap(y)/\sqrt{\overline\Gamma(1,\dots,1)}$ and $\zmap(z)$ with $\overline\zmap(z) = \sqrt{\overline\Gamma(1,\dots,1)}\zmap(z)$. In this way the \structshort{} definition in \cref{def:self} is satisfied, with $\closs \leq \sqrt{\overline\Gamma(1,\dots,1)\overline\Gamma(\msf{c}_{\closs_1}^2,\dots,\msf{c}_{\closs_P}^2)}$.
\end{proof}

\end{document}

%% file: commands.tex
\providecommand{\scal}[2]{\left\langle{#1},{#2}\right\rangle}

\providecommand{\nor}[1]{\left\|{#1}\right\|}